%% file: main.tex
\documentclass{article}

\usepackage{microtype}
\usepackage{graphicx}
\usepackage{subfigure}
\usepackage{booktabs} 

\usepackage{afterpage}

\usepackage{hyperref}       
\hypersetup{colorlinks,linkcolor={blue},citecolor={blue},urlcolor={red}}  
\usepackage{url}            
\usepackage{booktabs}       
\usepackage{amsfonts}       
\usepackage{nicefrac}       
\usepackage{microtype}      
\usepackage{graphicx}

\usepackage{mathtools}
\usepackage{amsthm}
\usepackage{bm}
\usepackage{lipsum}
\usepackage{wrapfig}

\newtheorem*{claim}{Claim}
\usepackage{bbm}
\usepackage{multirow,tabularx}

\usepackage[utf8]{inputenc}
\usepackage{fourier} 
\usepackage{array}
\usepackage{makecell}

\usepackage{xcolor}

\newcommand{\red}[1]{{\color{red}#1}}
\newcommand{\orange}[1]{{\color{orange}#1}}
\definecolor{darkgreen}{rgb}{0.0, 0.2, 0.13}
\newcommand{\green}[1]{{\color{darkgreen}#1}}

\definecolor{lightgreen}{rgb}{0.0, 0.4, 0.26}
\newcommand{\lgreen}[1]{{\color{lightgreen}#1}}

\newcommand{\pflem}[1]{\begin{proof}[Proof of Lemma \ref{#1}] \label{#1p}}
\newcommand{\pfthm}[1]{\begin{proof}[Proof of Theorem \ref{#1}] \label{#1p}}
\newcommand{\epf}{\end{proof}}
\newcommand{\xmark}{\ding{55}}%

\newcommand{\cin}{n_{\text{in}}}
\newcommand{\cout}{n_{\text{out}}}

\newcommand{\win}{W_{\text{b}}}
\newcommand{\wout}{W_{\text{f}}}

\newcommand{\figscale}{0.3in}
\newcommand{\figspacescale}{-1.0in}

\newcommand{\cmark}{\ding{51}}%

\usepackage{amssymb}
\usepackage{pifont}

\usepackage{hyperref}

\usepackage[accepted]{icml2020}

\icmltitlerunning{Beyond Signal Propagation: Is Feature Diversity Necessary in Deep Neural Network Initialization?}

\begin{document}

\twocolumn[
\icmltitle{Beyond Signal Propagation:\\ Is Feature Diversity Necessary in Deep Neural Network Initialization?}



\icmlsetsymbol{equal}{*}

\begin{icmlauthorlist}
\icmlauthor{Yaniv Blumenfeld}{tec}
\icmlauthor{Dar Gilboa}{col}
\icmlauthor{Daniel Soudry}{tec}
\end{icmlauthorlist}

\icmlaffiliation{tec}{Technion, Israel}
\icmlaffiliation{col}{Columbia University}

\icmlcorrespondingauthor{Yaniv Blumenfeld}{yanivblm6@gmail.com}

\icmlkeywords{Machine Learning, ICML, Deel Learning, Signal Propagation, Random Features, Lottery Ticket, Mean field}

\vskip 0.3in
]



\printAffiliationsAndNotice{} 

\begin{abstract}
Deep neural networks are typically initialized with random weights, with variances chosen to facilitate signal propagation and stable gradients. It is also believed that diversity of features is an important property of these initializations. We construct a deep convolutional network with identical features by initializing almost all the weights to $0$. The architecture also enables perfect signal propagation and stable gradients, and achieves high accuracy on standard benchmarks. This indicates that random, diverse initializations are \textit{not} necessary for training neural networks. An essential element in training this network is a mechanism of symmetry breaking; we study this phenomenon and find that standard GPU operations, which are non-deterministic, can serve as a sufficient source of symmetry breaking to enable training.
\end{abstract}


\section{Introduction}







Random, independent initialization of weights in deep neural networks is a common practice across numerous architectures and machine learning tasks \cite{he2016deep, vaswani2017attention, szegedy2015going}. When backpropagation was first proposed \citep{Rumelhart1986}, neural networks were initialized randomly with the goal of breaking symmetry between the learned features. Empirically, it had been established that the method used to initialize the weights of each layer can have a significant effect on the accuracy of the trained model, and common initialization schemes \cite{glorot2010understanding, he2015delving} are generally motivated by ensuring that the variance of the neurons does not grow rapidly with depth at initialization.

Neural networks with random features \cite{rahimi2008random} have been thoroughly studied, and such models admit a detailed analysis of their dynamics and generalization properties. Unlike the features in a neural network, the randomly initialized features in these models are not trained. While random features models are known to be limited compared to ones with learned or hand-designed features \cite{yehudai2019power}, numerous studies have shown that classifiers can be trained to good accuracy relying solely on random, untrained features \cite{louart2018random,mei2019generalization}. 

A possible explanation for the role of random initialization was implied by the recently proposed "lottery ticket hypothesis" \cite{frankle2018lottery}. After showing that sparse networks could be trained as effectively as dense networks if certain subsets of the weights were initialized with the same values, the authors suggest that neural networks contain trainable sub-networks, characterized by their unique architecture and initialization. These sparse sub-network can be extracted by first training dense neural networks and selecting weights with large magnitudes. Additionally, \cite{ramanujan2019s} has shown that randomly initialized neural network contain sub-networks that achieve relatively high accuracy without any training. However, more recent studies of this topic \cite{frankle2019lottery,frankle2019linear} has put the significance of the initial sub-networks back into question. It was suggested that in the case of deeper models, the Lottery Ticket method is more effective when relying on the values of the weights in a dense network at advanced training epochs, and not the randomly initialization values.



While this may indicate that feature diversity is an important property of the network initialization, there is evidence to the contrary as well. Calculations of signal propagation in random neural networks \cite{poole2016exponential,schoenholz2016deep}, applied in \cite{xiao2018dynamical} to convolutional networks, led the authors to suggest the Delta Orthogonal initialization, which they have used to successfully train a 10,000 layer convolutional network without skip connections on the CIFAR-10 dataset. The resulting initialization is relatively sparse, with only a single non-zero entry per convolution filter. When studying residual neural networks with a similar approach, as was done in \cite{yang2017mean}, it is apparent that signal propagation is optimized when the entire signal passes through the residual connection at initialization, which can be achieved by simply initializing all weights that can be bypassed to zero. This approach is supported, to some extent by \cite{zhang2019fixup}, where the authors suggest the Fixup initialization in which the final layer of the residual block (a block which can be bypassed by a single skip connection) is initialized to zero. 

The inevitable side-effect of these "sparse" initializations is that the variety of the sub-networks at initialization is limited, as zero-initialized parameters do not contribute new sub-networks to the grand total.

The initialization schemes proposed above suggest that feature diversity may not be necessary. In this work, we address the fundamental question: \textbf{Does deviating from the standard of independent, random initialization of weights have negative effects on training?} We do so by taking the idea of feature diversity to the extreme, and design networks where all the initial features are identical, while the requirements of signal propagation are maintained. We present surprising evidence that some networks are capable of fully recovering from these naive initializations during training, given some trivial requirements. We characterize the process where the features' symmetry is broken, distinguish between different levels of symmetry, and suggest criteria for the overall feature diversity in the network, which is shown to be tied with the success of the model.

\section{Feature Diversity}

When considering the function implemented by a neural network, the hidden state at every layer can be seen as a collection of \textit{features}, extracted from the input by the preceding computational logic in the network. In a classification task, the subsequent logic in the network uses these features to classify the input to the target label. Features are a function of the inputs and therefore two features will be considered identical only if their respected neurons are equal for \textbf{all} possible inputs. As mentioned in the introduction, there are reasons to believe that identical features at initialization will be detrimental to training. Symmetries between parameters will, in general, induce symmetries on the features at initialization, though the manner in which the two are related will depend on the architecture. In a fully-connected layer for example, a sufficient condition for equality between two features is equality of the two corresponding rows of the weight matrix. 

\subsection{Shallow Networks with Identical Features} \label{sec:single_hidden_layer}

Before delving into deep models, it worthwhile to start by examining the effect of identical features in a simple toy model. For this task, we use a fully-connected single hidden layer with ReLU activations. Additional details are provided in Appendix \ref{sup:replicated features}.

The initial symmetry between features in this model can theoretically be broken by back-propagation alone. We test it by training the neural network over the MNIST dataset, while changing the width of the network $n$ and the number of copies (replicas) we initialized for each unique row/feature. 

As shown in figure \ref{fig:replicated_features}, the test accuracy of the network degrades as the number of unique features at initialization decreases. Additional results, shown in the Appendix in figure \ref{fig:replicated_features_errorbars}, even suggest that initializing features with replicas of existing features may result in worse accuracy than removing those features altogether. Nevertheless, the negative effect of feature replication can be ameliorated by the addition of a small random independent 'noise' to the initial values of each replica. This suggests that the diversity of features at initialization can have long-term implications on the success of training, yet this effect appears to vanish quickly when minor stochastic elements are introduced.



\begin{figure}
    \centering
    \begin{subfigure}
        \centering
        \includegraphics[width=3.0in]{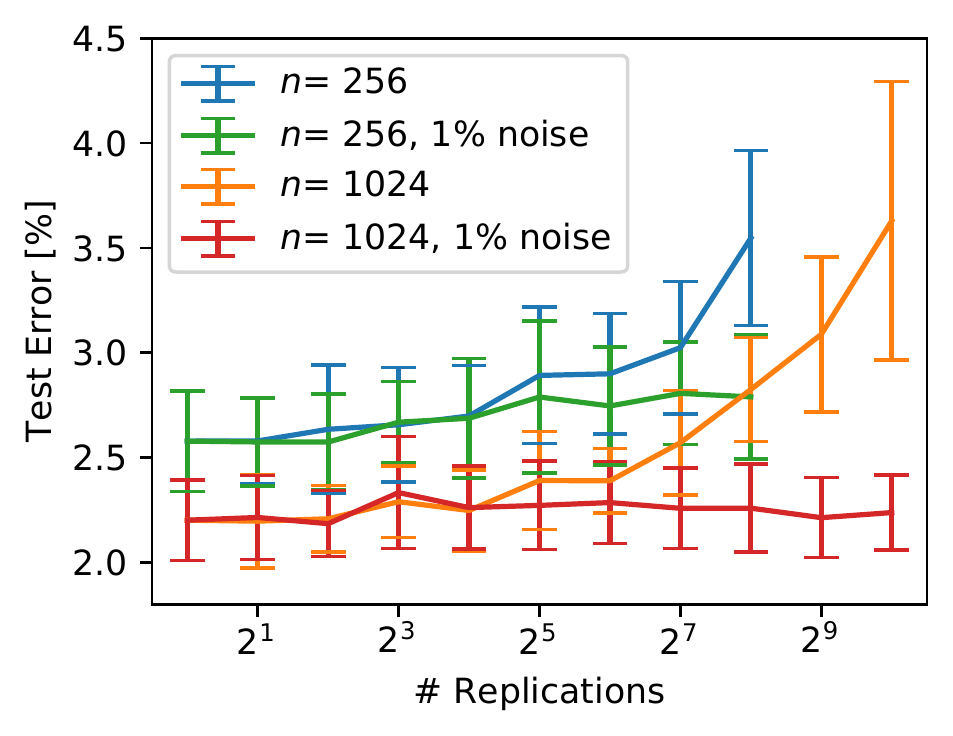}
    \end{subfigure}
    \vspace{-0.5cm}
    \caption{\small Test error of a 2-layers fully connected network, for networks with hidden layer of width $n$ initialized with replicated features. The network performance degrades as the number of \textbf{unique} features declines, unless a small amount of noise is introduced during initialization.} 
    \label{fig:replicated_features}
\end{figure}

\subsection{ConstNet - A Convolutional Network with Identical Features at Initialization}


In order to explore whether networks with identical features can be trained on more challenging tasks, we introduce ConstNet, an architecture based on the Wide-ResNet model \cite{zagoruyko2016wide}. Apart of standard computation operations, residual networks contain \textit{skip connections}, which can be described as an addition of identity operators to the operation performed by one or more layers. 
We provide full details of the design of ConstNet, in Appendix \ref{sup:constnet_detailed}. The ConstNet network, as used in the experiments, includes:

\begin{itemize}
    \item An initial convolutional layer, initialized to average the input channels to identical copies.
    \item A variable number of "skip-able" convolutions, with 3 layers where the number of features is increased (referred to as \textit{widening layers}). See figure \ref{fig:constnet_block} for illustration.
    \item All skip connections bypass a single ReLU + convolution block, \textit{with its weights initialized to $0$}.
    \item Values of the residual convolutions, used for network widening, are initialized to a constant.
    \item Skip section ends with a 2D pooling operation, followed by a fully connected layer initialized to zero.
    \item Optional batch-norm before each activation layer.
\end{itemize}

\begin{figure}[!ht]
    \centering
    \begin{subfigure}
        \centering
                \includegraphics[width=3.30in]{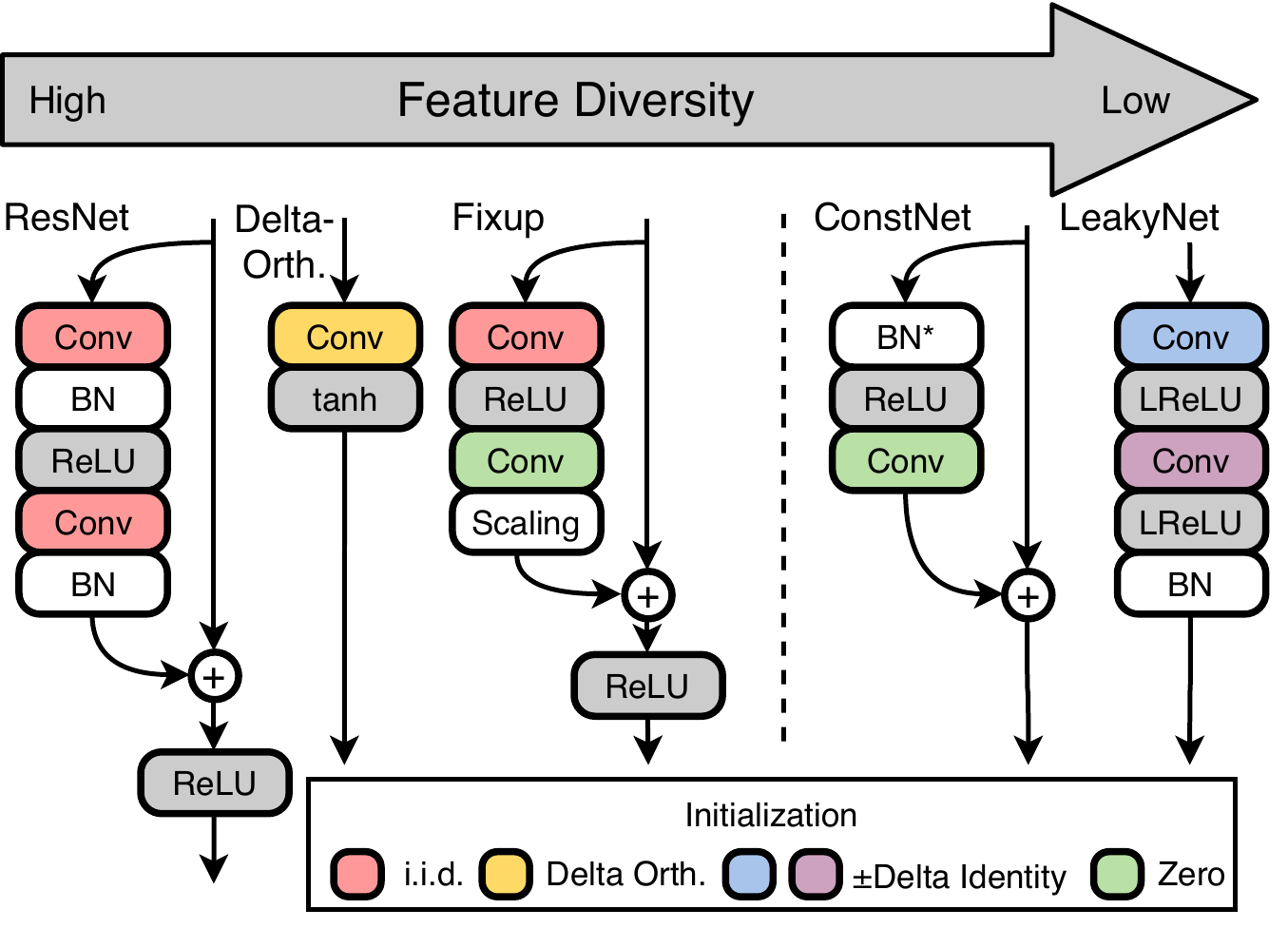}
    \end{subfigure}
    \vspace{-1cm}

    \caption{\small Convolutional architectures with reduced feature diversity. While a standard ResNet \cite{he2016deep} is initialized with i.i.d. random weights, the Fixup initialization \cite{zhang2019fixup} initializes one convolutional layer at $0$ while the Delta-Orthogonal initialization \cite{xiao2018dynamical} factorized the weight tensors, initializing the filter matrix deterministically and hidden state matrix with a random orthogonal initialization. ConstNet and LeakyNet by constrast are \textit{completely deterministic} at initialization with \textit{symmetrical features} as a result. In
    ConstNet all convolutional layers are initialized at $0$ while the LeakyNet initialization is identical to Delta-Orthogonal, with the random orthogonal matrix replaced by an identity that is multiplied by a negative factor at every odd layer. Both ConstNet and LeakyNet blocks implement an identity function at initialization. Batch-norm in ConstNet is optional, and we omit additional layers in Fixup that have no effect at initialization.}
    \label{fig:constnet_block}
\end{figure}

In order to motivate this architecture, we first show that due to the structure of the blocks in ConstNet, signals can propagate in a depth-independent manner from inputs to outputs, and gradients do not grow exponentially with depth. These are generally considered necessary  conditions for trainabiliy of a neural network. 


\section{Signal Propagation at Initialization} \label{sec:sig_prop}

Signal propagation in wide neural networks has been the subject of recent work for fully-connected \cite{poole2016exponential, schoenholz2016deep, pennington2017resurrecting, yang2017mean}, convolutional \cite{xiao2018dynamical} and recurrent architectures \cite{chen2018dynamical, gilboa2019dynamical}. These works study the evolution of covariances between the hidden states of the network and the stability of the gradients. At the wide-networks limit, the covariance evolution depends only on the leading moments of the weight distributions and the nonlinearities at the infinite width limit, greatly simplifying analysis. They identify critical initialization schemes that allow training of very deep networks (or recurrent networks on long time sequence tasks) without performing costly hyperparameter searches.

We briefly review standard approaches to the study of signal propagation in neural networks, before specializing to ConstNet. To begin, we consider a fully-connected feed-forward network $f(x)$ given by 
\[
\begin{array}{c}
\alpha^{(0)}(x)=x\\
\tilde{\alpha}^{(\ell)}(x)=W^{(\ell)}\alpha^{(\ell-1)}(x)+b^{(\ell)}\quad l=1,\dots,L-1\\
\alpha^{(\ell)}(x)=\phi(\tilde{\alpha}^{(\ell)}(x))\quad l=1,\dots,L-1\\
f(x)=W^{(L)}\alpha^{(L-1)}(x)+b^{(L-1)}.
\end{array}
\]
with $W^{(\ell)} \in \mathbb{R}^{n_{\ell } \times n_{\ell - 1}}, b^{(\ell)} \in \mathbb{R}^{n_l}$ initialized according to $W^{(\ell)}_{ij} \sim \mathcal{N}(0,\sigma_\ell^2)$ i.i.d., $b^{(\ell)}_i=0$, and $\phi$ is a nonlinearity that acts element-wise on vectors. Here $\alpha^{(\ell)}(x) \in \mathbb{R}^{n_\ell}$ is a vector of \textit{forward features} at layer $\ell$. We denote the scalar loss by $\mathcal{L }(f(x))$.

Common initialization schemes are motivated by guaranteeing stability of signals propagating from the inputs to the outputs by ensuring that the variance of the forward features does not change with depth \cite{he2016deep, vaswani2017attention, szegedy2015going}. For example, if we take $\phi$ to be the ReLU activation function, since  $\underset{j}{\sum}W_{ij}^{(\ell)}\alpha_{j}^{(\ell-1)}|\alpha^{(\ell-1)}$ is a Gaussian with variance $\sigma_{\ell}^2\left\Vert \alpha^{(\ell-1)}\right\Vert _{2}^{2}$ we obtain 
\[
\underset{W^{(\ell)}}{\mathbb{E}}\left(\alpha_{i}^{(\ell)}\right)^{2}|\alpha^{(\ell-1)}=\underset{W^{(\ell)}}{\mathbb{E}}\phi\left(\underset{j=1}{\overset{n_{\ell - 1}}{\sum}}W_{ij}^{(\ell)}\alpha_{j}^{(\ell-1)}\right)^{2}|\alpha^{(\ell-1)}
\]
\[
=\sigma_\ell^{2}\left\Vert \alpha^{(\ell-1)}\right\Vert _{2}^{2}\underset{0}{\overset{\infty}{\int}}g^{2}\mathcal{D}g=\frac{\sigma_\ell^{2}\left\Vert \alpha^{(\ell-1)}\right\Vert _{2}^{2}}{2}
\]
where $\mathcal{D}g$ is a standard Gaussian measure. The variance is preserved if $\mathbb{E}\left(\alpha_{i}^{(\ell)}\right)^{2}|\alpha^{(\ell-1)}=\frac{\left\Vert \alpha^{(\ell-1)}\right\Vert _{2}^{2}}{n_{\ell - 1}}$ which is guaranteed by choosing $\sigma_\ell^{2}=\frac{2}{n_{\ell - 1}}$.

Additionally, one can ensure that not only the variances of features are insensitive to depth but also the covariances between features given different inputs \cite{poole2016exponential,schoenholz2016deep, xiao2018dynamical}. For the forward features, these are given by
\begin{equation} \label{eq:covs}
    \Sigma_j^{(\ell)}(x,x')=\underset{\substack{\{ W^{(i)},\\b^{(i)} \} }}{\mathbb{E}}\left(\!\!\begin{array}{cc}
\left(\alpha_{j}^{(\ell)}(x)\right)^{2} & \alpha_{j}^{(\ell)}(x)\alpha_{j}^{(\ell)}(x')\\
\alpha_{j}^{(\ell)}(x)\alpha_{j}^{(\ell)}(x') & \left(\alpha_{j}^{(\ell)}\!\!\right)^{2}
\end{array}\!\!\right)
\end{equation}

where the expectations are taken over the weight distribution. At the wide network limit with random weights under weak moment assumptions, the above equation reduces to a deterministic dynamical system since the pre-activations are jointly Gaussian as a consequence of the Central Limit Theorem \cite{poole2016exponential, schoenholz2016deep, matthews2018gaussian}. As a result, at this limit the signals propagated through the network can be described completely in terms of these covariances, which are also independent of the neuron index $j$. Additionally, for a sufficiently wide network (typically once the width is few hundred neurons) the ensemble average above has proven to be predictive of the behavior of individual networks.

The covariances evolve according to 
\[
\Sigma^{(\ell+1)}(x,x')=\!\!\!\!\!\underset{\substack{
(u_{1},u_{2})\sim\\
\mathcal{N}(0,n_{\ell }\sigma_{\ell }^{2}\Sigma^{(\ell)}(x,x'))
}}{\mathbb{E}}\!\!\!\!\!\left(\!\!\!\begin{array}{cc}
\phi(u_{1})^{2} & \phi(u_{1})\phi(u_{2})\\
\phi(u_{1})\phi(u_{2}) & \phi(u_{2})^{2}
\end{array}\!\!\!\right).
\]
By studying this dynamical system one can obtain initialization schemes that allow signals to propagate stably, even in very deep networks, enabling them to be trained \cite{schoenholz2016deep, xiao2018dynamical}\footnote{These works in fact study closely related covariances defined for the pre-activations instead since these are jointly Gaussian at the infinite width limit and thus the evolution of the covariances obeys a simple closed form equation. See Appendix B of \cite{blumenfeld2019mean} for details of the relation between the two.}. 

In order to facilitate trainability with gradient descent, one can also study the variance of the \textit{backward features} 
\[
\beta_{j}^{(\ell)}(x)=\frac{\partial\mathcal{L }}{\partial\tilde{\alpha}_{j}^{(\ell)}(x)}
\]
which are of interest, since the gradients take the form 
\begin{equation} \label{eq:grad_fc}
\frac{\partial\mathcal{L }(x)}{\partial W_{ij}^{(\ell)}}=\beta_{i}^{(\ell)}(x)\alpha_{j}^{(\ell-1)}(x).
\end{equation}
In the case of convolutional networks, where $W^{(\ell)}\in\mathbb{R}^{K\times n_{\ell}\times n_{\ell-1}},b^{(\ell)}\in\mathbb{R}^{n_{\ell}}$, the backward and forwards features are tensors $\alpha^{(\ell)}(x)\in\mathbb{R}^{S^{(\ell)}\times n_{\ell}},\beta^{(\ell)}(x)\in\mathbb{R}^{S^{(\ell)}\times n_{\ell}}$ where $S^{(\ell)}$ is the space of spatial dimensions (pixels) at layer $\ell$. We will denote by $\gamma$ (or other Greek letters) a vector denoting the spatial location and $K$ denotes the dimensions of the kernel. We denote the convolution with respect to the kernel and a summation over the feature index by $\hat{*}$. The features are updated according to 
\[
\begin{array}{c}
\alpha^{(0)}(x)=x,\\
\begin{array}{c}
\tilde{\alpha}_{\gamma j}^{(\ell)}(x)=\left[W^{(\ell)}\hat{*}\alpha^{(\ell-1)}(x)\right]_{\gamma j}+b_{j}^{(\ell)}\\
=\underset{i=1}{\overset{n_{\ell-1}}{\sum}}\underset{\kappa\in K}{\sum}W_{\kappa ij}^{(\ell)}\alpha_{\gamma+\kappa, i}^{(\ell-1)}(x)+b_{j}^{(\ell)},
\end{array}\\
\alpha^{(\ell)}(x)=\phi(\tilde{\alpha}^{(\ell)}(x)),\\
f(x)=P(\alpha^{(L-1)}(x)).
\end{array}
\]
where $P$ is a function independent of depth (typically a composition of a pooling operation and an affine map). For simplicity of exposition we also assume periodic boundary conditions in the spatial dimensions. The covariances above can then be generalized to a tensor $\Sigma_{\gamma\gamma'jj'}^{(\ell)}(x,x')$ that can be analyzed in a similar manner to the fully-connected case \cite{xiao2018dynamical}. Analogously to eq. \ref{eq:grad_fc}, the gradients are related to the features by 
\begin{equation} \label{eq:grad_conv}
    \frac{\partial\mathcal{L}(x)}{\partial W_{\kappa ij}^{(\ell)}}=\underset{\gamma}{\sum}\beta_{\gamma i}^{(\ell)}(x)\alpha_{\gamma+\kappa,j}^{(\ell-1)}(x).
\end{equation}
At the infinite width it was shown that using smooth activations, stable gradients at initialization can be obtained with careful hyper-parameter tuning, inducing a dependence between the weight variance and the depth \cite{pennington2017resurrecting}. In \cite{burkholz2019initialization} it was recently shown that weight sharing at initialization enables signal propagation in feed-forward networks with ReLU activations. In Appendix \ref{app:rand_sigprop} we show that this approach can also be extended to convolutional networks by constructing a random initialization that guarantees signal propagation surely (and not just in expectation over the weights). As a consequence, the result is applicable to networks of arbitrary width.

 We also note that the FixUp initialization \cite{zhang2019fixup} does not exhibit stable backwards signal propagation since the gradients to certain layers within each block are zero at initialization (see figure \ref{fig:constnet_block} for details). 

\subsection{Depth-independent signal propagation with constant weights and skip connections}\label{sec:signal_prop_skips}
We now consider a class of neural networks that exhibit perfect forward signal propagation to arbitrary depth at initialization, and depth-independent backwards signal propagation by dispensing with randomness at initialization and utilizing skip connections. ConstNet is a member of this class.

Given a hidden state tensor $\alpha\in\mathbb{R}^{h\times w\times n}$ where $h,w$ are spatial dimensions and $n$ is the number of filters, we define for some integer $s$ such that $h\mod s=w\mod s=0$ a \textit{widening block} by 
\[
\text{WB}_{W,b,s,n}:\mathbb{R}^{h\times w\times n}\rightarrow\mathbb{R}^{h/s\times w/s\times sn},
\]
\[
\left[\text{WB}_{W,b,s,n}(\tilde{\alpha})\right]_{\gamma i}=\frac{1}{n}\underset{j=1}{\overset{n}{\sum}}\tilde{\alpha}_{s\gamma,j}+\left[W\hat{*}g(\tilde{\alpha})\right]_{\gamma i}+b_{i}.
\]
where $g:\mathbb{R}^{h\times w\times n}\rightarrow\mathbb{R}^{h\times w\times n}$ is a differentiable function such as a composition of a non-linearity and batch-normalization, and $W \in \mathbb{R}^{K \times sn\times n}, b \in \mathbb{R}^{sn}$\textit{are initialized as $0$}. 
This is equivalent at initialization to a convolution with a $1 \times 1$ identity filter, stride $s$ and a constant matrix acting on the channels. 

We define a \textit{ConstNet block function} as a map 
\[
\text{CB}_{W,b,n}:\mathbb{R}^{h\times w\times n}\rightarrow\mathbb{R}^{h\times w\times n},
\]
\[
\left[\text{CB}_{W,b,n}(\tilde{\alpha})\right]_{\gamma i}=\tilde{\alpha}_{\gamma i}+\left[W\hat{*}g(\tilde{\alpha})\right]_{\gamma i}+b_{i}
\]
where $W \in \mathbb{R}^{K \times n\times n}, b \in \mathbb{R}^{n}$ \textit{are initialized as $0$}. 

Considering inputs $x\in\mathbb{R}^{S^{(0)}\times n_{d}}$, we define a depth $L$ \textit{ConstNet function} by
\begin{equation} \label{eq:constnetfunc}
    f(x,B)=P( O_{L-1}( O_{L-2}(\dots O_{1}(\tilde{\alpha}^{(0)}(x)) \dots )))
\end{equation}
where $O_i$ is either a ConstNet block or a widening block, $P$ is a differentiable operation (typically a composition of pooling and an affine map) and $B$ is a batch of datapoints containing $x$. We define $\tilde{\alpha}^{(0)}(x,B)\in\mathbb{R}^{S^{(0)}\times n_{0}}$ to be a function of $B$ (but dropping the $B$ dependence to lighten notation), that obeys $\tilde{\alpha}_{\gamma i}^{(0)}(x)=\tilde{\alpha}_{\gamma j}^{(0)}(x)$ and normalized so that $\underset{\gamma\in S^{(0)}}{\sum}\underset{j\in B}{\sum}\tilde{\alpha}_{\gamma i}^{(0)}(x_{j})=0$ and $\underset{\gamma\in S^{(0)}}{\sum}\underset{j\in B}{\sum}\left(\tilde{\alpha}_{\gamma i}^{(0)}(x_{j})\right)^{2}=1$ (which can be achieved by applying a convolution and batch normalization operation to $x$ for instance). We assume the stride parameter $s$ and $S^{(0)}$ are chosen such that the spatial dimension remains larger than $0$ throughout. 

We consider translation invariant inputs, meaning for any $x,x'\in\mathbb{R}^{S^{(0)}\times n_d}$ drawn from the data distribution we have 
\[
x_{\gamma k}\overset{d}{=}x'_{\gamma ' k}
\]
where $\overset{d}{=}$ denotes equality in distribution and $\gamma,\gamma'$ are vectors denoting spatial location.
A local invariance to translation is a well-studied property of natural images and is believed to be key to the widespread use of convolutional networks in image classification tasks. 


\begin{claim}
Let $f$ be an $L$-layer ConstNet function as in eq. \ref{eq:constnetfunc} and denote the scalar loss by $\mathcal{L}$. Then for any $0 \leq \ell \leq L-1$ we have 
\[
\frac{\left\langle \beta^{(\ell)}(x),\beta^{(\ell)}(x')\right\rangle }{n_{\ell}}=\frac{\left\langle \beta^{(L-1)}(x),\beta^{(L-1)}(x')\right\rangle }{n_{L-1}}
\]
\[
\begin{array}{c}
\frac{\partial\mathcal{L}(x)}{\partial W_{\kappa ij}^{(\ell)}}=C_{ij}\cos(\theta_{\kappa,\ell}) \\
\frac{\partial\mathcal{L}(x)}{\partial b_{i}^{(\ell)}}=C_{i}'
\end{array}
\]
where $C_{ij},C_{i}'>0$ are constants that are independent of $L$ (but depend on the functions $P,g$ in the definition of the ConstNet function and on $\mathcal{L}$). $\theta_{\kappa,\ell}$ are constants that can depend on $L$.

Additionally, for translation invariant inputs, if we denote the spatial dimensions at layer $\ell$ by $S^{(\ell)}$ we have for any $\gamma,\gamma' \in S^{(\ell)}$ 
\[
\tilde{\alpha}_{\gamma i}^{(\ell)}(x)\overset{d}{=}\tilde{\alpha}_{\gamma'1}^{(0)}(x)
\]
\[
\frac{\left\langle \tilde{\alpha}^{(\ell)}(x),\tilde{\alpha}^{(\ell)}(x')\right\rangle }{\left\vert S^{(\ell)}\right\vert n_{\ell}}\overset{d}{=}\frac{\left\langle \tilde{\alpha}^{(0)}(x),\tilde{\alpha}^{(0)}(x')\right\rangle }{\left\vert S^{(0)}\right\vert n_{0}}
\]
\end{claim}

\textit{Proof:} See Appendix \ref{app:constnet_sigprop}

The claim shows that angles between forward and backward features are preserved by ConstNet (the former only in distribution), and that the gradients cannot grow exponentially with depth. The stability of angles between inputs is known to be predictive of trainability and generalization in many architectures \cite{schoenholz2016deep,pennington2017resurrecting,xiao2018dynamical}. In networks where angles are not preserved, training tends to fail.
Note that the stability conditions in the above claim hold for arbitrary width, unlike similar results that only apply to wide networks. 

One can ask whether the depth-independent signal propagation in ConstNet is a sufficient condition for it to be trainable. Indeed, this property is often considered a necessary but insufficient condition for trainability. The answer turns out to be negative due to the symmetry between features at all layers (which is also not surprising, given the results in Section \ref{sec:single_hidden_layer}), yet there happens to be a simple solution to this issue.


\section{Feature Symmetry and Symmetry Breaking}\label{sec:featureSym}





Two different features of the same hidden layer, which were initialized to present an identical function of the input, are still expected to diverge during back-propagation if their connection to the output is weighted differently. However, in the case when the their respective connection to the following layer is also symmetrical, the two features become interchangeable, and are expected to get the same updates in a deterministic process. This is the case for all the initial features in ConstNet as well. It stands to reason that some form of stochasticity must be introduced into the training process, in order for the initial symmetry to be broken. This could be achieved by a deliberate injection of noise to the parameters or gradients, or by relying on existing training mechanisms, such as Dropout.

One additional source of stochasticity that is often overlooked is the computation process itself. For example, the order of execution when parallelizing GEMM (General Matrix Multiplication) operations is often non-deterministic, resulting a stochastic output for non-commutative operations (including the addition of floating point values). It is therefore one of the mechanisms that may enable training with gradient descent to break the symmetry between the features at initialization. In our experiments, we used Nvidia GTX-1080 GPUs and the cudNN library, which is non-deterministic by default\footnote{See \url{https://github.com/NVIDIA/tensorflow-determinism} for additional details.}. 




\begin{table}[t]
\begin{center}
\begin{tabular}{| c  | c | c  | c |}
\hline
Model           & Init            & Variant         &  Accuracy [\%]   \\ [1.0ex]
\hline
Wide-ResNet     & He              &  --              &       $95.77 \pm 0.05$       \\ [1ex] 
\hline
ConstNet        & He             &  --    &         $95.40 \pm 0.07$         \\ [1ex] 
\hline
ConstNet        & $0$             &  1\% Dropout    &         $95.46 \pm 0.13$         \\ [1ex] 
\hline
ConstNet        & $0$             &  --              &       $95.37 \pm 0.06$ \\ [1ex] 
\hline
ConstNet        & $0$             & Deterministic            &     \red{24.79}  $ \pm 0.58$             \\ [1ex] 
\hline

\end{tabular}
\end{center}
\caption{Stochastic GPU computations enables training with identical features. Test accuracy on CIFAR-10, with 300 epochs. A minimal mechanism of symmetry breaking (e.g small dropout or non-deterministic GPU operations) is sufficient for ConstNet with almost all weights initialized at $0$ to train, matching the test accuracy achieved with standard random initialization. Training fails without such a mechanism (the "Deterministic" variant).} 
\vspace{-0.2cm}

\label{tbl:constnet_results}
\end{table}

Results of training ConstNet with different symmetry breaking mechanisms, as well as comparison with the baseline of He random initialization, can be seen in figure \ref{tbl:constnet_results}. Surprisingly, the effect of identical features at initialization appears to be minor whenever any form of symmetry breaking was allowed, and even negligible when non-deterministic computation noise is allowed. Not only was the training process capable of breaking features symmetry, but the overall test accuracy was not affected when compared to runs where the initial features were diverse (when 'He' initialization was used). Our immediate conclusion is that the fixed point where networks features are identical is extremely unstable. When no mechanism of symmetry breaking is present (the deterministic run with no Dropout) training fails.

It should be noted that '$0$' initialization does have its flaws, which were not captured in this experiment. The loss/accuracy curves of training when '$0$'-init was used, were typically few epochs behind the curves of its random initialization counterparts. On some occasions, depending on the hardware specifications, we have identified cases where the '$0$'-init runs relying only on hardware noise were stuck on the initial stages of training (10\% accuracy) for a random number of epochs. This phenomenon can be avoided by adding dropout as a complementary symmetry breaking mechanism, and did not significantly affect the final accuracy, due to the high number of epochs in this experiment. More results regarding the effect of different symmetry breaking mechanisms are detailed in appendix \ref{sup:comparing_symbreak} and appendix \ref{sup:lim}. Furthermore, $0$-initialized network were less robust to learning rate changes, and would fail completely for high learning rates, which could be handled by 'He' initialized networks.

\subsection{Evolution of Weight Correlations}\label{sec:symmetryBreak}


To analyze the phenomena of symmetry breaking, we will first define macro parameters which can represent the degree of symmetry in our network. Given a convolution with the weights $W\in \mathbb{R}^{k_x \times k_y  \times \cout \times \cin }$, we can split the weight tensor $W$ to either $\cout$ tensors that defines each of our output channels as $\wout^i \in \mathbb{R}^{k_x \times k_y \times \cin  }, 0 \le i < \cout $, or split it to $\cin$ tensors, where each tensor, $\win^i \in \mathbb{R}^{k_x \times k_y \times \cout}, 0 \le i < \cin $, contain the weights that operate on a single input channel. For our macro-parameter, we will look at the mean \textit{correlation} between the different tensors ($\win$ or $\wout$), and denote them as the \textit{forward correlations} $C_f$ and \textit{backward correlations} $C_b$ as:
\begin{align}\label{eq:correlationDef}
    C_f &\equiv \frac{1}{\cout\left(\cout-1\right)}\sum_{i}^{\cout}\sum_{i\ne j}^{\cout} \frac{\wout^{i}\cdot\wout^{j}}{\left\Vert \wout^{i}\right\Vert _{F}\left\Vert \wout^{j}\right\Vert _{F}},\\
    C_b &\equiv \frac{1}{\cin\left(\cin-1\right)}\sum_{i}^{\cin}\sum_{i\ne j}^{\cin}\frac{\win^{i}\cdot\win^{j}}{\left\Vert \win^{i}\right\Vert _{F}\left\Vert \win^{j}\right\Vert _{F}}
\end{align}
where $A \cdot B$ is an element-wise dot product operations, and $\left\Vert \cdot \right\Vert _{F}$ is Frobenius norm. For the case of a fully-connected layer, the operation can be viewed as a convolution where each of the input's and output's channels is of size $1\times1$, with a kernel of the same size. 





The definitions of backward and forward correlations allow us to go back and examine the different initializations presented in section \ref{sec:featureSym}:
\begin{itemize}
    \item For a weight tensor initialized with random, unbiased initialization, $C_f=0, C_b=0$ in the limit where the number of input/output channels or kernel dimensions goes to infinity.
    \item For an orthogonal initialized weight tensor in the channel dimensions, $C_f=0, C_b=0$.
    \item When the identity operation is expanded to also replicate $\cin$ channels to $\cout=P\times\cin$,  $C_f=\frac{P-1}{\cout-1}, C_b=0$.
    \item When all the filters are initialized to a constant $C_f=1, C_b=1$.
    \item Replicating features affects the forward correlation only. For $d$ sized output and $R$ replication per feature, $C_f=\frac{R-1}{d-1} , C_b=0$.
    \item For zero initialization, the fraction is undefined, but will be considered as fully correlated, $C_f=1, C_b=1$.
\end{itemize}

In the case of ConstNet with '$0$' init, all layers had been initialized with $C_f=1, C_b=1$. In contrast, standard random initialization will result in near-zero correlations. Correlations at a specific layer will be denoted by $C_f(\ell), C_b(\ell)$. 

A characteristic example for the evolution of the forward and backward correlations of a single tensor in ConstNet, during training, can be seen in figure \ref{fig:Constnet_correlations_full_run}, and the full results can be see in figure \ref{fig:full_run_constnet}. An interesting finding when comparing random and constant initialization is that in almost all layers, the forward correlations appear to converge to similar values for both initializations (He or ConstNet). Due to the highly nonconvex optimization landscape, one might expect the properties of the solutions that gradient descent finds to depend on the initialization in some way. At least with respect to feature correlations this appears to not be the case. 

Since perfectly correlated features during inference are equivalent to a single neuron, it is tempting to interpret the dissimilarity $1+(n_\ell-1)\sqrt{1-C_f(\ell)^2}$ as an "effective width" of a layer. The convergence of the correlations in figure \ref{fig:Constnet_correlations_full_run} despite different initializations indicate that gradient descent is biased towards solutions with a certain effective width regardless of the initial degree of feature diversity, as long as there is a symmetry breaking mechanism present. 

\begin{figure}
    \centering
    \begin{subfigure}
        \centering
        \includegraphics[width=3.0in]{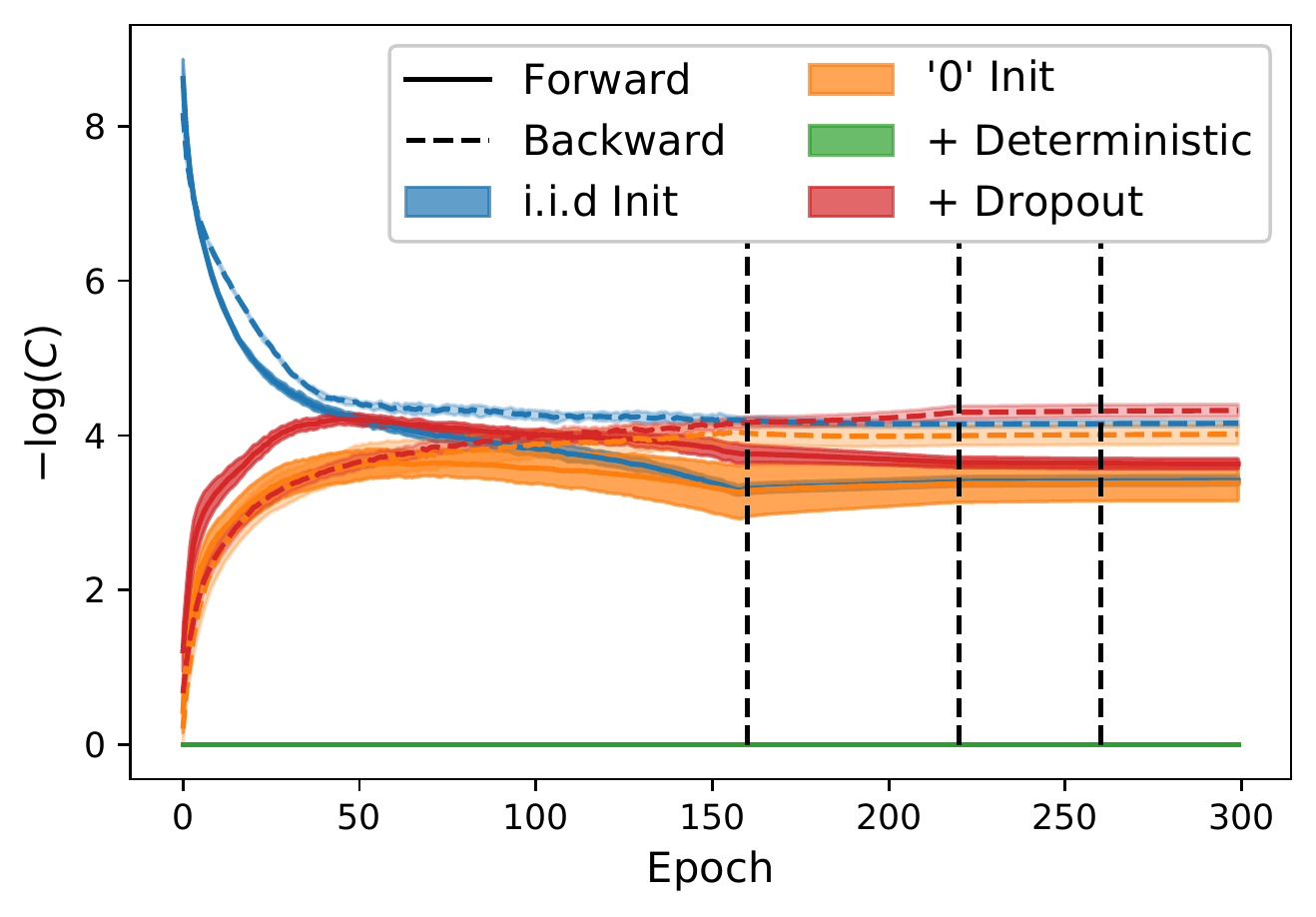}
    \end{subfigure}
    \caption{Weight correlations converge to similar values despite different initializations. 
    The forward correlation (solid) of a model with zero initialization and full feature symmetry eventually matches that of a model initialized without feature symmetry (He). Backward correlations (dashed) also converge to relatively similar values. When no mechanism for breaking symmetry was introduced (Green), the tensor remained maximally correlated, and training fails. Similar behavior is observed in almost all convolutional layers.} 
    \label{fig:Constnet_correlations_full_run}
\end{figure}

\subsection{Group behaviour of symmetry breaking} \label{subsec:group}

When observing the the behaviour of the correlations in different layers in ConstNet at the initial stages of training, we identify a group behaviour. The forward and backward correlations of the members of each group appears to behave in a similar manner. We identify several distinct groups in ConstNet, corresponding to the different widths of the hidden layers. We hypothesize that the layers between the the different widths, where we have a highly correlated averaging operation, are responsible for this phenomena. An example of this can be seen in figure \ref{fig:groups}. While the number of those groups in ConstNet is limited and depth independent (such operations are done only when the network width is increased), we suspect that an intensive use of this initialization, and thus a greater amount of individual groups required to break symmetry, could result a failure to train, despite the naive conservation of signal in depth.


\begin{figure}
    \centering
    \begin{subfigure}
        \centering
        \includegraphics[width=3.0in]{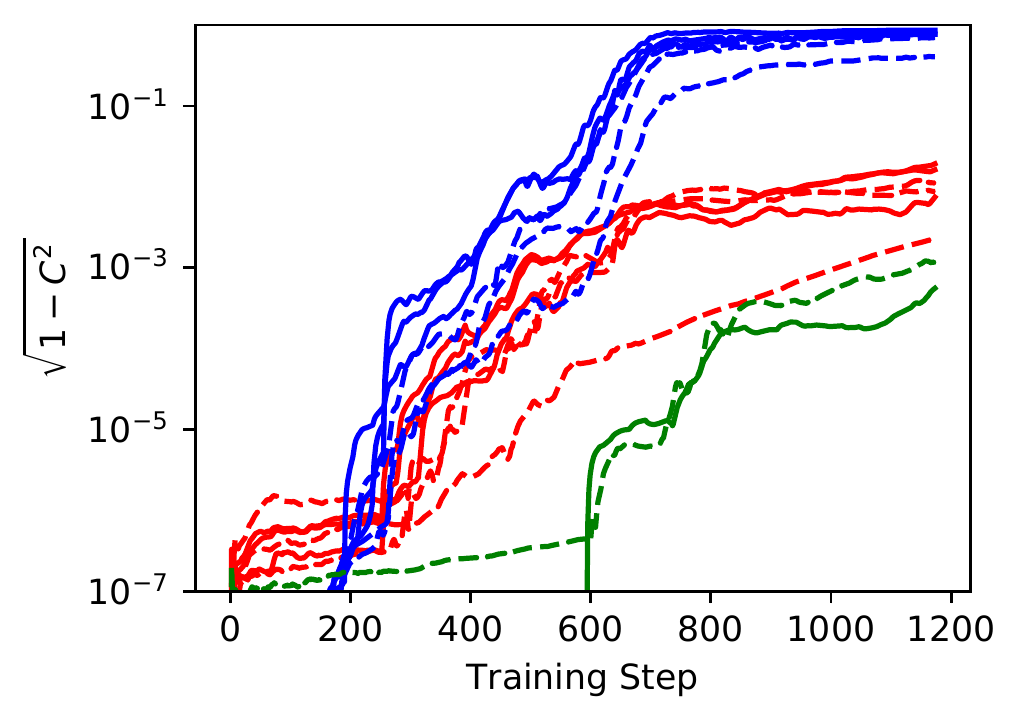}
    \end{subfigure}
    
    \caption{The early evolution of backward and forward correlations is group dependent (larger values are less correlated). This figure describe the forward correlations (solid) and backward correlations (dashed) members of different groups in ConstNet with '$0$' init, at the first 3 epochs of training (1 Epoch = 391 Steps). Different colors indicate different groups, as explained in section \ref{subsec:group}. A more detailed view of the individual correlations is available on figure \ref{fig:fullSymBreak_AtInit} in the Appendix.}
    \label{fig:groups}
\end{figure}


 

\section{Propagation of Signals that Break Symmetry}\label{sec:leakyNet}

In the previous sections, we have shown that networks with identical features can be trained if they i) enable signal propagation and ii) possess a mechanism for symmetry breaking. It appears that even when these conditions are satisfied, \textit{symmetry breaking can still be hindered if the signals that break symmetry cannot propagate} through the network. This is a novel type of signal that is not considered in standard analyses of signal propagation. We illustrate this phenomenon by designing a convolutional network that satisfies the above conditions (as shown in Appendix \ref{app:leaky_signal_prop}), yet still cannot be trained without additionally ensuring symmetry breaking signal propagation. 

\subsection{LeakyNet}

Rather than using the skip connection as a means of signal propagation, we can suggest alternative methods to initialize networks to represent identity at initialization, for the signal to be preserved. The main obstacle is the activation functions, being non-linear. We can overcome this limitation by replacing the \textit{ReLU} activation with the \textit{Leaky ReLU}: $\sigma^{\rho}(x)\equiv \text{LReLU}(x,\rho) = \rho x + (1-\rho)\text{ReLU}(x)$. In the limit $\rho \to 0$, both activations are identical, and we can use invertibility of leaky ReLU to form an identity. More specifically, we will use the equality: $\sigma^{\rho}(-\frac1{\rho}I \sigma^{\rho}(Ix))=-x$. Simply put, we initialize all layers so that every neuron swaps sign after each activation, making the composition of two nonlinear computation blocks linear (see figure \ref{fig:constnet_block}). 

To enforce identical features, it is sufficient to ensure that the first layer's features are identical, which we do by averaging the incoming channels, as in ConstNet. On instances when the network is widened, we map each feature to several identical copies. Apart of the computational block, all other components of network are similar to these presented in ConstNet. We train LeakyNet with the same configuration used for training ConstNet, and a default parameter $\rho=0.01$ for the leaky ReLUs. In all runs, the final layer was initialized to zero. 

The benefit of this architecture is that after factorizing the weight tensors into kernel matrices and matrices that act on the channels, we can choose between initializing these as an identity (denoted by '$I$'), and a matrix $\frac{1}{n}\bm{1}\bm{1}^T$ which we denote by '$\mathbb{1}$'. As seen in section \ref{sec:symmetryBreak}, the two different initializations will have a dramatically different effect on feature correlations, even though they both enable perfect signal propagation at initialization. Since the '$\mathbb{1}$' initialization averages over all channels, its output will be multiple copies of a single feature, hence they will be perfectly correlated. During training, as the symmetry between features is gradually broken due to some source of non-determinism, these symmetry breaking signals will decay when passing through a layer with a '$\mathbb{1}$' initialization, but not when passing through a layer with '$I$' initialization. We discuss this further in Appendix \ref{app:leaky_pert}.

\subsection{Controlling the Effectiveness of Symmetry Breaking}

Consequently, if we hold the total depth fixed and begin with a network that uses the '$I$' initialization at all layers, \textit{we can progressively hinder symmetry breaking by initializing more of layers with the '$\mathbb{1}$' initialization} instead. We perform a series of experiments of this nature. For reference, we also train this network with random (He) initialization, adjusted for the gain of the leaky ReLU initialization.




When random initialization was used, the correlations have frequently converged to a \textbf{higher value}, indicating some degree of features co-adaptation. The relative success of random initialization, compared with the orthogonal initialization, implies that feature diversity (in the sense of maximal forward correlations) have a negative effect as well, when examining the network in advanced stages of training. It is possible that it was, in fact, the orthogonal initializations that failed in the task of features co-adaptation. The full results regarding forward correlations are presented in figure \ref{fig:leakyNet_Correlations} in the appendix.


\begin{table}[h!]
\begin{center}\small
 \begin{tabular}{ | c | c | c | c  | c| c | c|}
 \hline

\multicolumn{4}{|c|}{\vspace{-0.3cm}  \makecell{Initialization\\ ($\#$ Conv Layers)}} & \makecell{Test\\Acc. [\%]\\ (4 seeds)}    &  \makecell{Max \\ $C_f$ \vspace{-0.3cm}} &  \makecell{Mean \\ $\sqrt{1-C_f^2}$ \vspace{-0.6cm}}       \\ 
\cline{1-4}
   Total  &     He   &    '$I$'    &   '$\mathbb{1}$'     &      &     &          \\ [1.0ex]
\hline
      13  &   13      &      -        &    -       &     $95.57\pm0.12$  &    0.21  &   0.97   \\ [1ex] 
\hline
      13  &   -      &      13        &    0       &    $94.05\pm0.23$    &   0.27  &   0.97       \\   [1ex] 
\hline
      13  &   -      &      12        &    1       &     $94.34\pm0.12$   &   0.25  &   0.97      \\ [1ex] 
\hline
      13  &   -      &      10        &    3       &    $93.47\pm0.32$  &  0.63  &   0.94 \\ [1ex] 
\hline
      13  &   -      &      8        &    5       &     $87.85\pm0.52$  &     \red{1.00}  &   0.79  \\  [1ex] 
\hline
      13  &   -      &      6        &    7       &   $44.89\pm2.52$   &   \red{1.00}    &   0.52\\ [1ex] 
\hline
      13  &   -      &      0        &    13       &    $29.77\pm2.47$   &   \red{1.00}  &   0.21 \\ [1ex] 
\hline

\end{tabular}
\end{center}
\caption{Test Accuracy of LeakyNet is correlated with dissimilarity between features. Results are on the CIFAR-10 dataset. Each network is initialized with a different number of '$\mathbb{1}$' and '$I$' initialization. The final performance degrades as more '$\mathbb{1}$' layers are present. The mean dissimilarity between features $\frac{1}{L}\underset{\ell}{\sum}\sqrt{1-C_{f}^{2}(\ell)}$ is indicative of final test accuracy, and networks where the maximal feature correlation $\underset{\ell}{\max}C_{f}(\ell)$ was close to $1$ performed poorly.} 
\label{tbl:leakynet_results}
\vspace{-0.5cm}

\end{table}

\section{Discussion}

In this work, we have shown that random initialization is not a necessary condition for deep convolutional networks to be trainable. Depth-independent propagation of signals through the network can be achieved without recourse to random initialization through the use of skip connections, or through artificial initialization methods. In one case, we show that the initial symmetry forced on a model is so fragile during training, that non-deterministic computation is sufficient to bring an otherwise untrainable model to a final accuracy of ~95\% on the CIFAR-10 benchmark.

By experimenting with radically uniform, naive initializations, we identify cases where the lack of feature diversity and overall symmetry in the network lead to failure of deep neural networks at classification tasks. Nonetheless, our ultimate conclusion is that feature diversity should \textbf{not} be a major factor, when initializing a deep neural network. The extent of symmetry we had to enforce over the network at initialization to disrupt the training process to a distinguishable degree, indicates that there is no significant penalty from subtle adjustments that sacrifice random initialization to improve the network's dynamics. Any negative effect we did encounter, could be easily negated by a small addition of independent, random values to break the initial symmetry.

While our symmetrically initialized networks are not intended for practical training, their simplicity can prove useful for the purpose of theoretical analysis. ConstNet, for example, implements a trivial mapping at initialization (as shown in Appendix \ref{app:ntk}). On the other hand, 
there is a growing theoretical literature analyzing the behavior of deep networks by linearizing near initialization \cite{Jacot2018-dv, arora2019exact, lee2019wide}. Our work suggests that understanding the behaviour of typical networks will likely require analysis of the nonlinear process of feature learning from data that happens in later stages of training. This of course is far from a novel conclusion, yet is made more palpable when considering that the features in the ConstNet model at initialization are identical to features of a linear model, as shown in Appendix \ref{app:ntk}. A relatively complex task like CIFAR-10 classification cannot be solved using kernel regression with such features. This is a clear indication that after training, the network is far from the linear regime around initialization. 


\subsection*{Acknowledgements}

The research of DS was supported by the Israel Science Foundation (grant No. 31/1031), by the Israel Innovation Authority (the Avatar Consortium), and by the Taub Foundation. A Titan Xp used for this research was donated by the NVIDIA Corporation. The work of YB is supported by Leonard and Diane Sherman Interdisciplinary Graduate School Fellowship. This work was partially done while the DS and DG were visiting the Simons Institute for the Theory of Computing.  The work of DG was supported by the NSF NeuroNex Award DBI-1707398, and the Gatsby Charitable Foundation.

\bibliography{bibliography}
\bibliographystyle{icml2020}

\onecolumn
\newpage
\appendix


{\Large \textbf{Supplementary Materials}}

\section{Detailed ConstNet Design and Considerations}\label{sup:constnet_detailed}
In this section, we provide technical details for ConstNet architectures, and explain the motivations behind each architecture decision. The code used for this experiment is available in: \url{https://github.com/yanivbl6/BeyondSigProp}.


\textbf{Zero initialization:} To ensure that only the skip connections are taken into account for signal propagation, we initialize all weights in computation blocks bypassed by skip connection to zero. In gradient descent, two sequential blocks initialized to zero will not allow gradient propagation: the update of each weights will remain zero as long as the other operations are $0$. To counter this, we replace the commonly used ResNet block which contains two chained operations, and use a single operation block instead, as seen in figure \ref{fig:constnet_block}.

We note that the forward signal propagation can also be maintained by only initializing the last chained block to zero, as was done in fixup \cite{zhang2019fixup}, but this method does not align with our goal of examining initializations with identical features. Additionally, we adjust the position of the nonlinearities to allow a linear path for the backward signal, moving from the 'output' to the input. 


\textbf{Forcing identical features:} For our goal of identical features to be reached, we initialize the first convolutional layer's filters $F=\mathrm{Mat}(W_{ij})\in\mathbb{R}^{3\times3}$ to $F_{kl} = \frac1{n_{\text{in}}}\delta_{k1}\delta_{l1}, k,l\in{0,1,2}$, where $n_{\text{in}}$ is the number of input channels, so all there channels are being averaged. As we advance in a residual network from input to output, it is commonplace to increase the number of channels (a.k.a widening) while trimming their spatial dimensions (using strides of size larger than 1). The widening of a neural network does not allow simply using skips, as the expected output of the computational block has additional channels, and 1x1 convolution operations are used instead, so each output channel is a linear combination of the input channels. We initialize these filters to the constant $\frac{1}{n_{\text{in}}}$, so they perform as an 'averaging operations' as well. Generally speaking, repeated averaging operations do not conserve signal propagation, and we take advantage of the fact that the number of those operations in the network is limited, and that only the first averaging operation causes a loss of information. One more possible issue is the usage of stride, as it may disturb the signal propagation. However, we show, in section \ref{sec:signal_prop_skips}, that under the assumption of spatial invariance, stride convolutions do not change the angles between features in expectation. Finally, we initialize the final fully connected layer to zero.

\textbf{Depth and width selection:} As mentioned, our channel sizes in the networks are based on Wide-Resnet, which has a similar number of parameters and operations. Since the number of channels is large, we expect the effects of feature diversity to be at a peak. Specifically, a randomly initialized Wide Res-Net has high number of features: the number of channels in the layers of the network increases from $16$ to $640$ at the final residual layers. A noteworthy design detail in ConstNet is that the activation in each skip is located before the convolution operation. This was done in order to have a non-linearity between the initial convolution and the following skippable convolution.

\textbf{Batch Normalization:} While the fixup initialization allow training neural networks without batch normalization, our default setup utilize batch-norm operations, positioned before the nonlinearities, as a tool for regularization. To avoid using batchnorm on the linear path, a scalar scale operator is added \cite{zhang2019fixup} and a Mixup \cite{zhang2017mixup} regularization has to be applied to prevent generalization loss. This method therefore requires an hyper-parameter search over the values of the Mixup parameter $\alpha$ and the learning rate. Notwithstanding the results \cite{yang2019mean}, who has shown that batch-normalization is incompatible with our goal of maintaining signal propagation, batch-normalization of the spatial dimensions (per channel) will not have an effect over the channels in an initialized network, since the first Batch-norm operation is the only one that will have an effect, at the exact point of initialization.

\textbf{Hyper-parameters:} All runs with ConstNet, and the baseline runs with Wide-ResNet, were done for a model with 12 skip-able layers. We used SGD with momentum ($0.9$), batch size $128\times 2$ GPUs, and with cross-entropy loss over the CIFAR-10 dataset. Unless specified otherwise, we trained the network for 200 epochs, with a learning rate decay at epochs $60,120,160$. Unless mentioned otherwise, the default learning rate was $0.03$, dropout and mixup were disabled, Batchnorm was used and cudNN was non-deterministic. All parameters where chosen based on the default Wide-ResNet setup, with the exception of the learning rate, which was picked in consideration of network pruning results \cite{frankle2018lottery}  (Lottery ticket method was shown to only work for residual networks with low learning rates). Results were averaged over $6$ seeds. The number of parameters in the model was $\sim17$M, and data was represented with $32$-bit floating point.

\section{Depth-Independent Signal Propagation in Convolutional Networks with Random Initialiation} \label{app:rand_sigprop}

Stability of backward features depends on the conditioning of the Jacobian between network states $J^{(\ell)}= \frac{\partial\tilde{\alpha}^{(\ell)}}{\partial\tilde{\alpha}^{(\ell-1)}}$, since the backward features take the form $\beta^{(\ell)}=\frac{\partial\mathcal{L }}{\partial\tilde{\alpha}^{(L-1)}}\frac{\partial\tilde{\alpha}^{(L-1)}}{\partial\tilde{\alpha}^{(\ell)}}$. In essence - stability of the backward features is ensured if $||J^{(\ell)}||=1$ and the variance of the singular values of $J^{(\ell)}$ is zero, where $||\cdot||$ denotes operator norm. Works that analyze signal propagation at the infinite width limit define \textit{dynamical isometry} conditions that ensure the singular values of $J^{(\ell)}$ depend weakly on depth and all have absolute value close to $1$.

In this section we briefly describe how to construct a convolutional network with random initialization and depth-independent forward and backward signal propagation (dynamical isometry). We combine elements of the approaches in \cite{xiao2018dynamical, burkholz2019initialization}. For simplicity consider a network that is given by a composition of convolutions, and we assume periodic boundary conditions and that the number of channels $n$ is constant throughout. Given an input tensor $x\in\mathbb{R}^{S\times n/2}$ with $S$ denoting a set of spatial dimensions, we define an augmented input 
\[
\overline{x}=\left(\begin{array}{cc}
x & 0\end{array}\right)\in\mathbb{R}^{S\times n}
\]
If we assume the kernel is of size $K=(2k+1)\times(2k+1)$ and index these coordinates from $-k$ to $k$, we define a kernel by 
\begin{equation} \label{eq:delta_kernel}
   \Delta\in\mathbb{R}^{2k+1\times2k+1},\Delta_{ij}=\delta_{i0}\delta_{j0}. 
\end{equation}
This definition can be generalized trivially to settings where $K$ has order different from $2$. Initializing the biases at $0$, the parameters of the $\ell$-th  convolutional layer are given by $W^{(\ell)}\in\mathbb{R}^{K\times n \times n}$, which we factorize as 
\begin{equation} \label{eq:delta_orth_weight}
    W^{(\ell)}=\Delta\otimes\left(\begin{array}{cc}
U^{(\ell)} & -U^{(\ell)}\\
-U^{(\ell)} & U^{(\ell)}
\end{array}\right)
\end{equation}
where $U^{(\ell)}\in\mathbb{R}^{n/2\times n/2}$ is an orthogonal matrix drawn from a uniform distribution over $O(n/2)$ and $\Delta^{(\ell)}\in\mathbb{R}^{K}$ is the kernel defined above (we assume $n$ is even). As before, we denote $\tilde{\alpha}^{(1)}_{\gamma,i}=\left[W^{(1)}\hat{*}\overline{x}\right]_{\gamma i}=\underset{\kappa\in K}{\sum}\underset{j=1}{\overset{n}{\sum}}W_{\kappa ij}^{(1)}x_{\gamma+\kappa,j}$. The pre-activations at a spatial location $\gamma$, arranged in a column vector $\left[W^{(1)}\hat{*}\overline{x}\right]_{\gamma}^{T}$, are thus given by
\[
\left[W^{(1)}\hat{*}\overline{x}\right]_{\gamma}^{T}=\underset{\kappa\in K}{\sum}\Delta_{\kappa}\left(\begin{array}{cc}
U^{(1)} & -U^{(1)}\\
-U^{(1)} & U^{(1)}
\end{array}\right)\overline{x}_{\kappa+\gamma}^{T}
\]
\[
=\left(\begin{array}{cc}
U^{(1)} & -U^{(1)}\\
-U^{(1)} & U^{(1)}
\end{array}\right)\overline{x}_{\gamma}^{T}=\left(\begin{array}{cc}
U^{(1)} & -U^{(1)}\\
-U^{(1)} & U^{(1)}
\end{array}\right)\left(\begin{array}{c}
x_{\gamma}^{T}\\
0
\end{array}\right)
\]
\[
=\left(\begin{array}{c}
U^{(1)}x_{\gamma}^{T}\\
-U^{(1)}x_{\gamma}^{T}
\end{array}\right).
\]

Applying a ReLU nonlinearity and another convolutional layer gives 
\[
\left[W^{(2)}\hat{*}\phi\left(W^{(1)}\hat{*}\overline{x}\right)\right]_{\gamma}^{T}
\]
\[
=\underset{\kappa\in K}{\sum}\Delta_{\kappa}\left(\begin{array}{cc}
U^{(2)} & -U^{(2)}\\
-U^{(2)} & U^{(2)}
\end{array}\right)\phi\left(W^{(1)}\hat{*}\overline{x}\right)_{\kappa+\gamma}^{T}
\]
\[
=\left(\begin{array}{cc}
U^{(2)} & -U^{(2)}\\
-U^{(2)} & U^{(2)}
\end{array}\right)\phi\left(\left(\begin{array}{c}
U^{(1)}x_{\gamma}^{T}\\
-U^{(1)}x_{\gamma}^{T}
\end{array}\right)\right)
\]
\[
=\left(\begin{array}{cc}
U^{(2)} & -U^{(2)}\\
-U^{(2)} & U^{(2)}
\end{array}\right)\left(\begin{array}{c}
U^{(1)}x_{\gamma}^{T}\circ\left[U^{(1)}x_{\gamma}^{T}>0\right]\\
-U^{(1)}x_{\gamma}^{T}\circ\left[U^{(1)}x_{\gamma}^{T}<0\right]
\end{array}\right)
\]
\[
=\left(\begin{array}{c}
U^{(2)}\left(U^{(1)}x_{\gamma}^{T}\circ\left[U^{(1)}x_{\gamma}^{T}>0\right]+U^{(1)}x_{\gamma}^{T}\circ\left[U^{(1)}x_{\gamma}^{T}<0\right]\right)\\
-U^{(2)}\left(U^{(1)}x_{\gamma}^{T}\circ\left[U^{(1)}x_{\gamma}^{T}>0\right]+U^{(1)}x_{\gamma}^{T}\circ\left[U^{(1)}x_{\gamma}^{T}<0\right]\right)
\end{array}\right)
\]
\[
=\left(\begin{array}{c}
U^{(2)}U^{(1)}x_{\gamma}^{T}\\
-U^{(2)}U^{(1)}x_{\gamma}^{T}
\end{array}\right).
\]
We thus obtain that for any two layers the pre-activations are simply related by a rotation:
\begin{equation} \label{eq:preac_orth}
\left(\tilde{\alpha}_{\gamma}^{(\ell+1)}\right)^{T}=\left(\begin{array}{cc}
U^{(\ell+1)} & 0\\
0 & U^{(\ell+1)}
\end{array}\right)\left(\tilde{\alpha}_{\gamma}^{(\ell)}\right)^{T}.
\end{equation}
This preserves both norms and angles, and thus the covariance between pre-activations is invariant of depth, meaning 
\[
\left\langle \tilde{\alpha}_{\gamma}^{(\ell)}(x),\tilde{\alpha}_{\gamma'}^{(\ell)}(x')\right\rangle =\left\langle \tilde{\alpha}_{\gamma}^{(1)}(x),\tilde{\alpha}_{\gamma'}^{(1)}(x')\right\rangle.
\]
Note that this holds surely, and not only in expectation over the weights as in the main text. As a result, the covariance between the hidden states themselves is also independent of depth. 

We now consider backwards signal propagation. Since 
at a spatial location $\gamma$ the pre-activation $\left(\tilde{\alpha}_{\gamma}^{(\ell)}\right)^{T}$ is a concatenation of two identical vectors with opposite sign, there are only $n/2$ independent degrees of freedom. The backward features are thus given by 
\[
\beta_{\eta j}^{(\ell)}(x)=\frac{\partial\mathcal{L}}{\partial\tilde{\alpha}_{\eta j}^{(\ell)}(x)}=\underset{\gamma\in S}{\sum}\underset{i=1}{\overset{n/2}{\sum}}\frac{\partial\mathcal{L}}{\partial\tilde{\alpha}_{\gamma i}^{(\ell+1)}(x)}\frac{\partial\tilde{\alpha}_{\gamma i}^{(\ell+1)}(x)}{\partial\tilde{\alpha}_{\eta j}^{(\ell)}(x)}
\]
\[
=\underset{\gamma\in S}{\sum}\underset{i=1}{\overset{n/2}{\sum}}\beta_{\gamma i}^{(\ell+1)}(x)\frac{\partial\tilde{\alpha}_{\gamma i}^{(\ell+1)}(x)}{\partial\tilde{\alpha}_{\eta j}^{(\ell)}(x)}
\]
\[
=\beta^{(L-1)}(x)J_{n/2}^{(L-1)}J_{n/2}^{(L-2)}\dots\frac{\partial\tilde{\alpha}^{(\ell+1)}(x)}{\partial\tilde{\alpha}_{\eta j}^{(\ell)}(x)}
\]
where 
\[J_{n/2}^{(\ell)}\in\mathbb{R}^{S\times S\times n/2\times n/2},\left[J_{n/2}^{(\ell)}\right]_{\gamma\eta ij}=\frac{\partial\tilde{\alpha}_{\gamma i}^{(\ell)}(x)}{\partial\tilde{\alpha}_{\eta j}^{(\ell-1)}(x)}
\]
(without this structure in the pre-activations the Jacobians would be defined as matrices in $\mathbb{R}^{S\times S\times n\times n}$). Since $\frac{\partial\tilde{\alpha}_{\gamma i}^{(\ell)}}{\partial\tilde{\alpha}_{\eta j}^{(\ell-1)}}=W_{\eta-\gamma,ij}^{(\ell)}\dot{\phi}(\tilde{\alpha}^{(\ell-1)})_{\eta,j}\mathbb{1}_{\eta-\gamma\in K}$, plugging in the form of the weight tensor from equation \ref{eq:delta_orth_weight} gives 
\[
\frac{\partial\tilde{\alpha}_{\gamma i}^{(\ell)}}{\partial\tilde{\alpha}_{\eta j}^{(\ell-1)}}=\delta_{\gamma\eta}U_{ij}^{(\ell)}\dot{\phi}(\tilde{\alpha}^{(\ell-1)})_{\eta,j}
\]
hence over the spatial dimensions $S \times S$, the tensor $J_{n/2}^{(\ell)}$ is simply a delta function, and since at a given spatial location the pre-activations are related according to eq. \ref{eq:preac_orth} we obtain 
\[
J_{n/2}^{(\ell)}=I_{S\times S}\otimes U^{(\ell)}
\]
where $I_{S\times S}$ is shorthand for a product of delta functions over every spatial dimension. Since the $U^{(\ell)}$ are orthogonal, $J_{n/2}^{(\ell)}$ obeys the dynamical isometry conditions (all its singular values over the non-trivial dimensions have magnitude 1). Thus norms and angles between backward features are also independent of depth with this initialization. 

Note that both forwards and backward features are stable surely, and not just in expectation over the weights which is the usual form of dynamical isometry results that study networks at the infinite width limit (which is predictive of the behavior of reasonably wide networks)\cite{schoenholz2016deep,pennington2017resurrecting}. Therefore, a network initialized in this way will exhibit stable signal propagation at any width. 

\section{Depth-Independent Signal Propagation in ConstNet} \label{app:constnet_sigprop}

\begin{claim}
Let $f$ be an $L$-layer ConstNet function as in eq. \ref{eq:constnetfunc} and denote the scalar loss by $\mathcal{L}$. Then for any $0 \leq \ell \leq L-1$ we have 
\[
\frac{\left\langle \beta^{(\ell)}(x),\beta^{(\ell)}(x')\right\rangle }{n_{\ell}}=\frac{\left\langle \beta^{(L-1)}(x),\beta^{(L-1)}(x')\right\rangle }{n_{L-1}}
\]
\[
\begin{array}{c}
\frac{\partial\mathcal{L}(x)}{\partial W_{\kappa ij}^{(\ell)}}=C_{ij}\cos(\theta_{\kappa,\ell}) \\
\frac{\partial\mathcal{L}(x)}{\partial b_{i}^{(\ell)}}=C_{i}'
\end{array}
\]
where $C_{ij},C_{i}'>0$ are constants that are independent of $L$ (but depend on the functions $P,g$ in the definition of the ConstNet function and on $\mathcal{L}$). $\theta_{\kappa,\ell}$ are constants that can depend on $L$. 

Additionally, for translation invariant inputs, if we denote the spatial dimensions at layer $\ell$ by $S^{(\ell)}$ we have for any $\gamma,\gamma' \in S^{(\ell)}$ 
\[
\tilde{\alpha}_{\gamma i}^{(\ell)}(x)\overset{d}{=}\tilde{\alpha}_{\gamma'1}^{(0)}(x)
\]
\[
\frac{\left\langle \tilde{\alpha}^{(\ell)}(x),\tilde{\alpha}^{(\ell)}(x')\right\rangle }{\left\vert S^{(\ell)}\right\vert n_{\ell}}\overset{d}{=}\frac{\left\langle \tilde{\alpha}^{(0)}(x),\tilde{\alpha}^{(0)}(x')\right\rangle }{\left\vert S^{(0)}\right\vert n_{0}}
\]
\end{claim}

\begin{proof}
We assume that at the first layer we have 
\[
\tilde{\alpha}_{\gamma i}^{(0)}=\tilde{\alpha}_{\gamma j}^{(0)}
\]
where $\tilde{\alpha}_{\gamma}^{(0)}$ is obtained by some affine transformation of the data. 
At any layer $\ell$, at initialization either $\alpha^{(\ell)}=\text{CB}_{0,0}(\alpha^{(\ell-1)})$ or $\alpha^{(\ell)}=\text{WB}_{0,0,s}(\alpha^{(\ell-1)})$. In the former case we have 
\[
\tilde{\alpha}_{\gamma j}^{(\ell)}(x)=\tilde{\alpha}_{\gamma j}^{(\ell-1)}(x)=\tilde{\alpha}_{\gamma 1}^{(\ell-1)}(x)
\]
for all $\gamma,j$, while in the latter case we have 
\[
\tilde{\alpha}_{\gamma j}^{(\ell)}(x)=\frac{1}{n}\underset{j=1}{\overset{n}{\sum}}\tilde{\alpha}_{\gamma s,j}^{(\ell-1)}(x)=\tilde{\alpha}_{\gamma s,1}^{(\ell-1)}(x).
\]
 If we assume that there are $p$ narrowing layers between $1$ and $\ell$, repeated application of the above equations gives  
 \[
\tilde{\alpha}_{\gamma i}^{(\ell)}(x)=\tilde{\alpha}_{\gamma s^{p},1}^{(0)}(x).
 \]
 If we denote the spatial dimensions at layer $\ell$ by $S^{(\ell)}$, since tranlation invariance of the inputs implies the same invariance of the pre-activations, we have
 \[
\tilde{\alpha}_{\gamma i}^{(\ell)}(x)=\tilde{\alpha}_{\gamma s^{p},i}^{(0)}(x)=\tilde{\alpha}_{\gamma s^{p},1}^{(0)}(x)\overset{d}{=}\tilde{\alpha}_{\gamma'1}^{(0)}(x)
 \]
 for any $\gamma,\gamma'$ and 
 \[
 \frac{\left\langle \tilde{\alpha}^{(\ell)}(x),\tilde{\alpha}^{(\ell)}(x')\right\rangle }{\left\vert S^{(\ell)}\right\vert n_\ell}=\frac{1}{\left\vert S^{(\ell)}\right\vert n_\ell}\underset{\gamma\in S^{(\ell)}}{\sum}\underset{i=1}{\overset{n_\ell}{\sum}}\tilde{\alpha}_{\gamma i}^{(\ell)}(x)\tilde{\alpha}_{\gamma i}^{(\ell)}(x')
 \]
 \[
 =\frac{1}{\left\vert S^{(\ell)}\right\vert }\underset{\gamma\in S^{(\ell)}}{\sum}\tilde{\alpha}_{\gamma s,1}^{(0)}(x)\tilde{\alpha}_{\gamma s,1}^{(0)}(x')
 \]
 \[
\overset{d}{=}\frac{1}{\left\vert S^{(0)}\right\vert }\underset{\gamma\in S^{(0)}}{\sum}\tilde{\alpha}_{\gamma,1}^{(0)}(x)\tilde{\alpha}_{\gamma,1}^{(0)}(x')=\frac{\left\langle \tilde{\alpha}^{(0)}(x),\tilde{\alpha}^{(0)}(x')\right\rangle }{\left\vert S^{(0)}\right\vert n_0}.
 \]

We now consider a layer $\ell$ such that there are $p$ narrowing layers between $1$ and $\ell$ and $q$ narrowing layers between $\ell$ and $L-1$. If we choose $\gamma$ such that $\gamma/s^q$ is a vector of integers (which we denote by $\gamma/s^q\in\mathbb{Z}^{d}$), we have 
\[
\beta_{\gamma j}^{(\ell)}(x)=\frac{\partial\mathcal{L}}{\partial\tilde{\alpha}_{\gamma j}^{(\ell)}(x)}=\frac{\partial\mathcal{L}}{\partial\tilde{\alpha}_{\gamma/s^{q},1}^{(L-1)}(x)}=\beta_{\gamma/s^{q},1}^{(L-1)}(x).
\]
from which it follows that 
\[
\frac{\left\langle \beta^{(\ell)}(x),\beta^{(\ell)}(x')\right\rangle }{n_{\ell}}=\frac{1}{n_{\ell}}\underset{\gamma\in S^{(\ell)}}{\sum}\underset{i=1}{\overset{n_{\ell}}{\sum}}\beta_{\gamma i}^{(\ell)}(x)\beta_{\gamma i}^{(\ell)}(x')
\]
\[
=\underset{\gamma\in S^{(\ell)},\gamma/s^{q}\in\mathbb{Z}^{d}}{\sum}\beta_{\gamma/s^{q}1}^{(L-1)}(x)\beta_{\gamma/s^{q}1}^{(L-1)}(x')=\underset{\gamma\in S^{(L-1)}}{\sum}\beta_{\gamma1}^{(L-1)}(x)\beta_{\gamma1}^{(L-1)}(x')
\]
\[
=\frac{\left\langle \beta^{(L-1)}(x),\beta^{(L-1)}(x')\right\rangle }{n_{L-1}}.
\]
Additionally, 
\[
\frac{\partial\mathcal{L}}{\partial\tilde{\alpha}_{\gamma/s^{q},1}^{(L-1)}(x)}=\frac{\partial\mathcal{L}}{\partial f(\tilde{\alpha}_{\gamma/s^{q},1}^{(L-1)}(x))}\frac{\partial f}{\partial\tilde{\alpha}_{\gamma/s^{q},1}^{(L-1)}(x)}
\]
\[
=\frac{\partial\mathcal{L}}{\partial P(\tilde{\alpha}_{\gamma/s^{q},1}^{(L-1)}(x))}\frac{\partial P(\tilde{\alpha}_{\gamma/s^{q},1}^{(L-1)}(x))}{\partial\tilde{\alpha}_{\gamma/s^{q},1}^{(L-1)}(x)}
\]
\[
=\dot{\mathcal{L}}\left(P(\tilde{\alpha}_{\gamma s^{p},1}^{(0)}(x))\right)\dot{P}(\tilde{\alpha}_{\gamma s^{p},1}^{(0)}(x)).
\]
Note that if $\gamma/s^q$ is not a vector of integers, the location $\gamma$ will not contribute to the loss and hence $\beta_{\gamma j}^{(\ell)}(x)=0$. We have thus shown that the non-zero elements of $\alpha_{i}^{(\ell)}(x),\beta_{i}^{(\ell)}(x)$ can be written in terms of quantities that are independent of depth.  

The gradients are given by 
\[
\frac{\partial\mathcal{L}(x)}{\partial W_{\kappa ij}^{(\ell)}}=\underset{\gamma\in S^{(\ell)}}{\sum}\beta_{\gamma i}^{(\ell)}(x)\alpha_{\gamma+\kappa,j}^{(\ell-1)}(x)
\]
\[
=\underset{\gamma\in S^{(\ell)},\gamma/s^{q}\in\mathbb{Z}^{d}}{\sum}\begin{array}{c}
\dot{\mathcal{L}}\left(P(\tilde{\alpha}_{\gamma s^{p},1}^{(0)}(x))\right)\dot{P}(\tilde{\alpha}_{\gamma s^{p},1}^{(0)}(x))*\\
g\left(\tilde{\alpha}_{(\gamma+\kappa)s^{p},1}^{(0)}(x)\right)
\end{array}
\]
\[
\frac{\partial\mathcal{L}(x)}{\partial b_{i}^{(\ell)}}=\underset{\gamma\in S^{(\ell)},\gamma/s^{q}\in S^{(L-1)}}{\sum}\beta_{\gamma i}^{(\ell)}(x).
\]
Note that the number of terms in this summation is $\left\vert S^{(L-1)}\right\vert $ and is thus independent of $\ell$. The gradients depend on $\ell$ only through the relative shift between $\alpha_{j}^{(\ell-1)}(x)$ and $\beta_{i}^{(\ell)}(x)$, which is $\kappa s^p$, and is thus bounded by product of the norms of these vectors which are independent of $\ell$. It follows that we can write 
\[
\frac{\partial\mathcal{L}(x)}{\partial W_{\kappa ij}^{(\ell)}}=C_{ij}\cos(\theta_{\kappa,\ell}),\frac{\partial\mathcal{L}(x)}{\partial b_{i}^{(\ell)}}=C_{i}'.
\]

Where $C_{ij},C_{i}'$ are independent of depth. 
\end{proof}

Note that if $C_{ij},C_{i}'$ are equal to $0$, which can be ensured with an appropriate choice of $P$, the gradient at initialization for all layers aside from the last will be $0$ as well. This result ensures that during early stages of training, once gradients take non-zero values, they will behave in a stable manner even if the network is very deep. There is also no reason to expect $\cos(\theta_{\kappa,\ell})$ to decay with depth and lead to vanishing gradients. 

If batch-norm parameters are also trained, or the map from the input to $\tilde{\alpha}^{(0)}$ is parametrized by trainable parameters, their gradients will also be insensitive to depth due to the insensitivity of the forward and backward features. Batch-norm will have no effect at initialization due to the symmetries of $\tilde{\alpha}^{(0)}$. Note also that for a reasonable batch size, if $\tilde{\alpha}^{(0)}(x)$ applies a convolution to $x$ with a $W=\frac{1}{n_{d}}\mathbb{1}_{n_{0}}\mathbb{1}_{n_{d}}^{T}\times\Delta$ where $\Delta$ is the delta kernel defined in Appendix \ref{app:constnet_sigprop}, we expect the angles between the inputs averaged over the input channels to be preserved in the sense $\angle(\ensuremath{\tilde{\alpha}^{(0)}}(x),\ensuremath{\tilde{\alpha}^{(0)}}(x'))\approx\angle(\underset{\gamma}{\sum}x_{\gamma},\underset{\gamma'}{\sum}x'_{\gamma})$. 

We also note that if the number of features was held constant at all layers (as in \cite{xiao2018dynamical} for instance) the norm of the backward features would be preserved as well, and the magnitude of the gradient elements would be completely independent of depth. The dependence is only a result of the widening layers and not of the ConstNet blocks themselves.

\begin{figure*}[t!]
    \centering
    \begin{subfigure}
        \centering
        \includegraphics[width=5.5in]{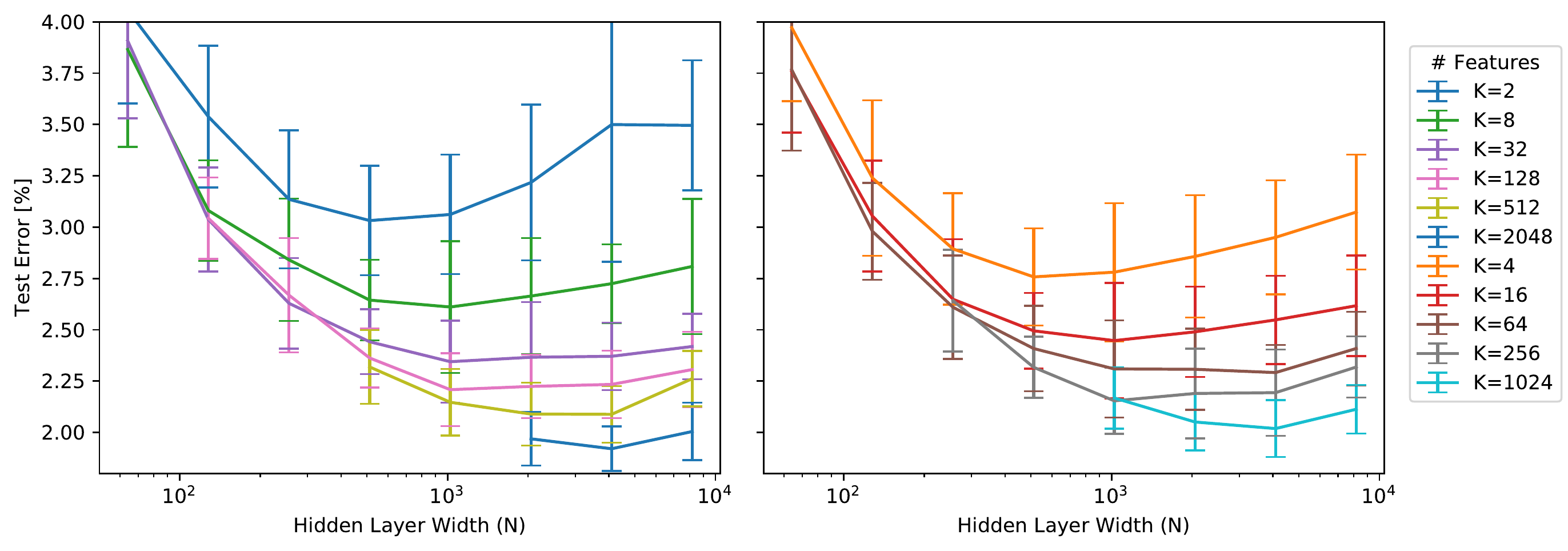}
    \end{subfigure}
    \caption{Full Test error results for a 2-layers fully connected network. 
    The hidden layer features were initialized so the same random feature is represented $N/K$ times.}
    \label{fig:replicated_features_errorbars}
\end{figure*}


\section{ConstNet Features at Initialization are Equivalent to a Shallow Network} \label{app:ntk}
We consider the evolution of the network function $f$ during full-batch gradient flow over a dataset of size $d$. Assuming $f(x) \in \mathbb{R}^{n_c}$, individual elements will evolve according to
\[
\frac{\partial f_{k}(x)}{\partial t}=\underset{i}{\sum}\frac{\partial f_{k}(x)}{\partial\theta_{i}}\frac{\partial\theta_{i}}{\partial t}=-\underset{ij}{\sum}\frac{\partial f_{k}(x)}{\partial\theta_{i}}\frac{\partial\mathcal{L}(x_{j})}{\partial\theta_{i}}
\]
\[
=-\underset{jl}{\sum}\left[\underset{i}{\sum}\frac{\partial f_{k}(x)}{\partial\theta_{i}}\frac{\partial f_{l}(x_{j})}{\partial\theta_{i}}\right]\frac{\partial\mathcal{L}(x_{j})}{\partial f_{l}(x_{j})}\equiv\underset{jl}{\sum}\Theta_{kl}(x,x_{j})\frac{\partial\mathcal{L}(x_{j})}{\partial f_{l}(x_{j})}.
\]
where the tensor $\Theta\in\mathbb{R}^{n_{c}\times n_{c}\times d\times d}$ is known as the Neural Tangent Kernel (NTK). It is of interest because a sufficiently overparametrized neural network has the capacity to train with the NTK remaining essentially constant during training, enabling a detailed analysis of the dynamics of learning in this regime \cite{Jacot2018-dv,lee2019wide,arora2019exact}. Training is thus essentially equivalent to kernel regression with respect to the kernel $\Theta\in\mathbb{R}^{n_{c}\times n_{c}\times d\times d}$, and if a standard random initialization is used this is a random feature kernel. While this is interesting from a theoretical perspective, it is also apparent that the real strength of deep neural networks is in learning features from data, and the standard operational regime of modern neural networks is one where $\Theta$ changes considerably during training \cite{chizat2018note,ghorbani2019limitations}. 

The ConstNet architecture is interesting in this regard because the features it implements at initialization are particularly weak. They correspond to the features of a linear classifier. To show this, we consider a ConstNet function where the final layers of the network implement average pooling and an affine map, namely 
\[
f(x)=P(\alpha^{(L-1)}(x))=W^{(L)}\frac{1}{\left\vert S^{(L-1)}\right\vert }\underset{\gamma\in S^{(L-1)}}{\sum}\alpha^{(L-1)}(x)+b^{(L)}
\]
and all the previous layers are either convolutions or narrowing layers as described in section \ref{sec:signal_prop_skips}. The NTK takes the form 
\[
\Theta_{ij}(x,x')=\underset{k}{\sum}\frac{\partial f_{i}(x)}{\partial\theta_{k}}\frac{\partial f_{j}(x')}{\partial\theta_{k}}
\]
\[
=\underset{\ell=1}{\overset{L-1}{\sum}}\underset{i_{\ell}j_{\ell}\gamma_{\ell}\gamma'_{\ell}\kappa}{\sum}\beta_{\gamma_{\ell}i_{\ell}}^{(\ell)}(x)\alpha_{\gamma_{\ell}+\kappa,j_{\ell}}^{(\ell-1)}(x)\beta_{\gamma'_{\ell}i_{\ell}}^{(\ell)}(x')\alpha_{\gamma'_{\ell}+\kappa,j_{\ell}}^{(\ell-1)}(x')
\]
\[
+\underset{\gamma\gamma'}{\sum}\beta_{\gamma_{\ell}i_{\ell}}^{(\ell)}(x)\beta_{\gamma'_{\ell}i_{\ell}}^{(\ell)}(x')
\]
\[
+\delta_{ij}\left(\left\langle \frac{\underset{\gamma\in S^{(L-1)}}{\sum}\alpha_{\gamma}^{(L-1)}(x)}{\left\vert S^{(L-1)}\right\vert },\frac{\underset{\gamma'\in S^{(L-1)}}{\sum}\alpha_{\gamma'}^{(L-1)}(x')}{\left\vert S^{(L-1)}\right\vert }\right\rangle +1\right)
\]
where the last term is the contribution from the final layer. Note that since we initialize the final layer weights to zero, we have $\beta^{(\ell)}(x)=0$ for $\ell < L$. In Appendix \ref{app:constnet_sigprop} we noted that if there are $p$ narrowing layers in the network $ \alpha_{\gamma i}^{(L-1)}(x)=\alpha_{\gamma s^{p},i}^{(0)}(x)$ hence 
\[
\Theta_{ij}(x,x')=\delta_{ij}\left(\left\langle \frac{\underset{\gamma\in S^{(L-1)}}{\sum}\alpha_{\gamma s^{p}}^{(0)}(x)}{\left\vert S^{(L-1)}\right\vert },\frac{\underset{\gamma'\in S^{(L-1)}}{\sum}\alpha_{\gamma's^{p}}^{(0)}(x')}{\left\vert S^{(L-1)}\right\vert }\right\rangle +1\right).
\]
If we consider ConstNet without batch-normalization, which achieves a test accuracy of $95\%$ on CIFAR-10 classification, we have $\tilde{\alpha}_{\gamma i}^{(0)}(x)=\underset{j=1}{\overset{3}{\sum}}x_{\gamma j}$. Therefore the kernel above is that of a linear model (since if the model was $f(x)=\theta^{T}x$, we would obtain $\Theta(x,x')=\left\langle x,x'\right\rangle $).  
The fact that ConstNet reaches this level of performance implies that there will be a massive performance gap between training ConstNet in the linear regime as in \cite{lee2019wide,arora2019exact} and full nonlinear training. If we choose a more complex form for $P$ the resulting NTK will still be identical to that of a shallow model (though not necessarily a linear one). 

\input{sections/replicated_features.tex}

\input{sections/subnetworks}

\input{sections/fullSymBreak}

\input{sections/constNetRuns}

\input{sections/LeakyNetsRunsFinal}

\section{Perturbation Propagation} \label{app:leaky_pert}

We illustrate in an idealized setting the effect of the $I$ and '$\mathbb{1}$' initializations on propagation of a symmetry breaking signal. Recall that the tensors for these two initializations are given respectively by $\Delta \otimes I$ or $\Delta \otimes \frac{1}{n}\bm{1}\bm{1}^T$ where $\Delta$ is the delta kernel defined in eq. \ref{eq:delta_kernel}. As shown in appendix \ref{app:leaky_signal_prop}, both of these choices of initialization are indistinguishable with respect to standard signal propagation as it applies to the features at initialization when these are symmetric. 

Consider a feature tensor $\alpha\in\mathbb{R}^{S\times n}$ and a zero mean additive perturbation $\varepsilon\in\mathbb{R}^{S\times n}$ that breaks the symmetry between features. The source of such a perturbation can be non-deterministic GPU operations during parameter updates which will break symmetry after the first iteration of gradient descent. Applying a convolution gives
\[
\left[\Delta\otimes I\widehat{*}\left(\alpha+\varepsilon\right)\right]_{\gamma i}=\left[\alpha+\varepsilon\right]_{\gamma i}
\]
\[
\left[\Delta\otimes\frac{1}{n}\bm{1}\bm{1}^{T}\widehat{*}\left(\alpha+\varepsilon\right)\right]_{\gamma i}=\left[\alpha+\underset{j}{\sum}\frac{1}{n}\varepsilon_{\gamma j}\right]_{\gamma i}.
\]
We see that in the first instance the perturbation is propagated perfectly. In the second instance, if the perturbation components are independent, its norm will be reduced by about a factor of $\frac{1}{\sqrt{n}}$.

Propagation of the perturbation $\varepsilon$ facilitates the breaking of symmetry in subsequent layers, and hence its attenuation hampers symmetry breaking. Thus even though from the perspective of propagation of the symmetric features the $I$ and '$\mathbb{1}$' initializations are equivalent, the $I$ initialization facilitates symmetry breaking signal propagation.
Empirical evidence of this phenomenon is shown in figures \ref{fig:pertubation_prop} and \ref{fig:pertubation_grad_prop}. We measure the size of the symmetry breaking perturbation, given by the norm of the forward and backward features projected onto the complement of the all ones direction, divided by the feature norm itself for normalization. This is repeated for networks initialized with different combinations of the $I$ and '$\mathbb{1}$' initializations. Generally, it can be seen that the relative norm of the symmetry breaking component increases during training. We also find that the relative size of the perturbation decreases sharply when passing through a '$\mathbb{1}$' initialized layer, yet is relatively unaffected by an $I$ layer.

By measuring feature correlations as a function of layer, one can see the detrimental effect of the '$\mathbb{1}$' initialization on symmetry breaking even after training. These results are presented in figures \ref{fig:leakyNet_Correlations}. Networks with reduced symmetry breaking also achieve lower test accuracy as shown in table \ref{tbl:leakynet_results}. 

One can also conjecture based on these results that the combination of a zero initialization and skip connection as in ConstNet may be unique in the sense that the trainable parameters are initialized in a completely symmetric manner yet symmetry is easily broken. The '$\mathbb{1}$' initialization is symmetric but hampers symmetry breaking, while the $I$ initialization is not symmetric (and it only preserves symmetry from previous layers since in LeakyNet the initial layer of the network induces a symmetry between features). 

\section{Leaky ReLU Signal Propagation} \label{app:leaky_signal_prop}

For simplicity we consider an $L$ layer network of constant width at every layer and no batch normalization. The more general case with widening layers included can be handled in a similar manner to appendix \ref{app:constnet_sigprop}. Since batch normalization layers are applied after every block (which implements an identity map), their effect will simply be that after the first layer the normalized features are propagated.

We denote by $\alpha^{(0)}(x)\in\mathbb{R}^{S\times n},\alpha_{\gamma i}^{(0)}(x)=\alpha_{\gamma j}^{(0)}(x)$ the input to the first LeakyNet block. As in the case of ConstNet, the input features are completely symmetric. Recall that the weight tensors of the network take the form $\Delta \otimes I$ or $\Delta \otimes \frac{1}{n}\bm{1}\bm{1}^T$ where $\Delta$ is the delta kernel defined in eq. \ref{eq:delta_kernel}, and the second factor acts on the channel indices. Due to the symmetry of the incoming features,   
\[
\Delta\otimes I\widehat{*}\alpha^{(0)}(x)=\Delta\otimes\frac{1}{n}\bm{1}\bm{1}^{T}\widehat{*}\alpha^{(0)}(x)=\alpha^{(0)}(x).
\]
Defining by $\alpha^{(k)}(x)$ the features after $k$ layers (where each LeakyNet block is composed of two layers), and recalling that $\sigma^{\rho}(-\frac1{\rho} \sigma^{\rho}(x))=-x$, we find for every integer $k \leq L/4$
\begin{align*}
 \alpha^{(4k)}(x)&=\alpha^{(0)}(x)\\
 \alpha^{(4k+1)}(x)&=-\frac{1}{\rho}\sigma^{\rho}(\alpha^{(0)}(x)) \\
  \alpha^{(4k+2)}(x)&=-\alpha^{(0)}(x) \\
   \alpha^{(4k+3)}(x)&=-\frac{1}{\rho}\sigma^{\rho}(-\alpha^{(0)}(x)). \\
\end{align*}
Since we can express all features as simple functions of the initial features, norms of features and angles between them are trivially preserved (up to the application of single nonlinearity, which does not induce any dependence on the depth of the network).

Similarly, the backwards features cannot incur a dependence on depth, and their norms and angles between them are preserved up to the effect of a single non-linearity.

\input{sections/LeakyNetPretubationPropagation}

\section{Comparison between symmetry breaking mechanisms}\label{sup:comparing_symbreak}

As seen in table \ref{tbl:constnet_results}, the hardware noise is a sufficient symmetry breaking mechanism for a ConstNet model initialized with identical features to perform as well as a similar model initialized with random initialization. It is, nevertheless, also apparent that not all breaking mechanisms perform equally well: When using $1\%$ dropout as the sole symmetry breaking mechanism, there remains a gap of more than $1\%$ between the model final accuracy to the accuracy of a randomly initialized model. The reason for dropout's failure is simple: On a standard setup, dropout is not being applied on the first convolutional layer, and it is therefore unable to break symmetry in this layer. Applying dropout on all layers does close this gap.

Another concern, is that due to our reliance hardware noise, it remains unclear how the same models will perform when used over different hardware. In table \ref{tbl:symbreaking_comparison}, we compare the accuracy of ConstNet models initialized with identical features and varying learning rates, using different GPUs. Our results show that the choice of hardware is indeed significant. Furthermore, they provide an insight concerning the effects of network quantization: Unlike the 32-bit floating point models tested before, models quantized to 16-bit, do not train successfully without the addition of dropout, as a complementary symmetry breaking mechanism. This result can be useful for the study of network-quantization, as it highlights a property of quantization that had not been studied --- 32bit and 16bit training are typically expected to perform equally well for standard image-classification tasks. In this experiment, we used 200 epochs for training, which is sub-optimal ($\sim\%0.5$ drop baseline accuracy). Consequently, we can expect lower accuracy in cases where the symmetry breaking process is slow.

\begin{table}[b]
\begin{center}
 \begin{tabular}{| c | c | c | c | c| c|}
 \hline
Hardware           & Data-type & \makecell{Computation \\ Noise}  & \multicolumn{3}{c|}{\makecell{Drop Rate:}\vspace{-0.3cm}}  \\ [1.0ex]
 \cline{4-6}
                   &           &                 &   $0\%$        &   $0.1\%$ &   $1\%$   \\ [1.0ex]
                   
 \hline
GeForce GTX-1080  & FP32       &        \cmark   &   \green{*94.09\% }      &   \lgreen{94.83\%} &  \lgreen{*94.8\%}  \\ 
 \hline
GeForce GTX-1080  & FP32       &        \xmark   &   \red{24.96\%}      &   \orange{92.72\%} &  \orange{92.99\%}  \\ 
 \hline
GeForce RTX-2080  & FP32       &        \cmark   &   \red{*77.81\%}      &   \lgreen{94.78\%} &  \lgreen{94.82\%}  \\
\hline
GeForce RTX-2080  & FP32       &        \xmark   &   \red{24.80\%}      &   \orange{92.69\%} &  \orange{92.58\%}  \\ 
\hline
GeForce RTX-2080  & FP16       &        \cmark   &   \red{38.61\%}      &   \lgreen{94.58\%} &  \lgreen{94.62\%}  \\ 
 \hline
GeForce RTX-2080  & FP16       &        \xmark   &   \red{22.55\%}      &   \orange{91.99\%} &  \orange{92.55\%}  \\ 
 \hline

\end{tabular}
\end{center}
\caption{Experiment results (Test Accuracy) of training a ConstNet over the CIFAR10 dataset, with 200 epochs, $0$ initialization ensuring identical features, and varying symmetry-breaking mechanisms. Dropout/ Computation noise by themselves are insufficient for successful training: the symmetry breaking in this case is too slow, and it come in the expense of the training process. Additionally, we can see a big gap between the effects of computation noise of different GPUs. When utilizing a 16-bit floating-point hardware architecture for a FP16 model, the hardware noise is insufficient and the training fails. *These results are discussed in more detail in appendix \ref{sup:lim}.}
\label{tbl:symbreaking_comparison}
\end{table}

\afterpage{

\subsection{Limitations of symmetry breaking during training}\label{sup:lim}
As we can see in table \ref{tbl:symbreaking_comparison}, the negative effects of constant initialization are further highlighted in this experiment. In particular, we can see that the accuracy of $0$ initialized ConstNet without dropout drops by more than $\%1.5$ in comparison to its result for 300 epochs, as presented in table \ref{tbl:constnet_results}. A more thorough examination of the results shows the reason for this performance drop: On several cases ($1/3$ of the seeds tested), the hardware computation mechanism managed to break symmetry eventually, but spent too many epochs doing it. The distribution of the results for several cases can be seen in figure \ref{fig:histogram}. When running the same model with the exact same configuration on a different GPU (Nvidia RTX), the degradation in performance is much more significant, with more than $1/2$ of the experiments ending in failure ($<80\%$).

\begin{figure*}[t]
    \centering
    \begin{subfigure}
        \centering
        \includegraphics[width=2.8in]{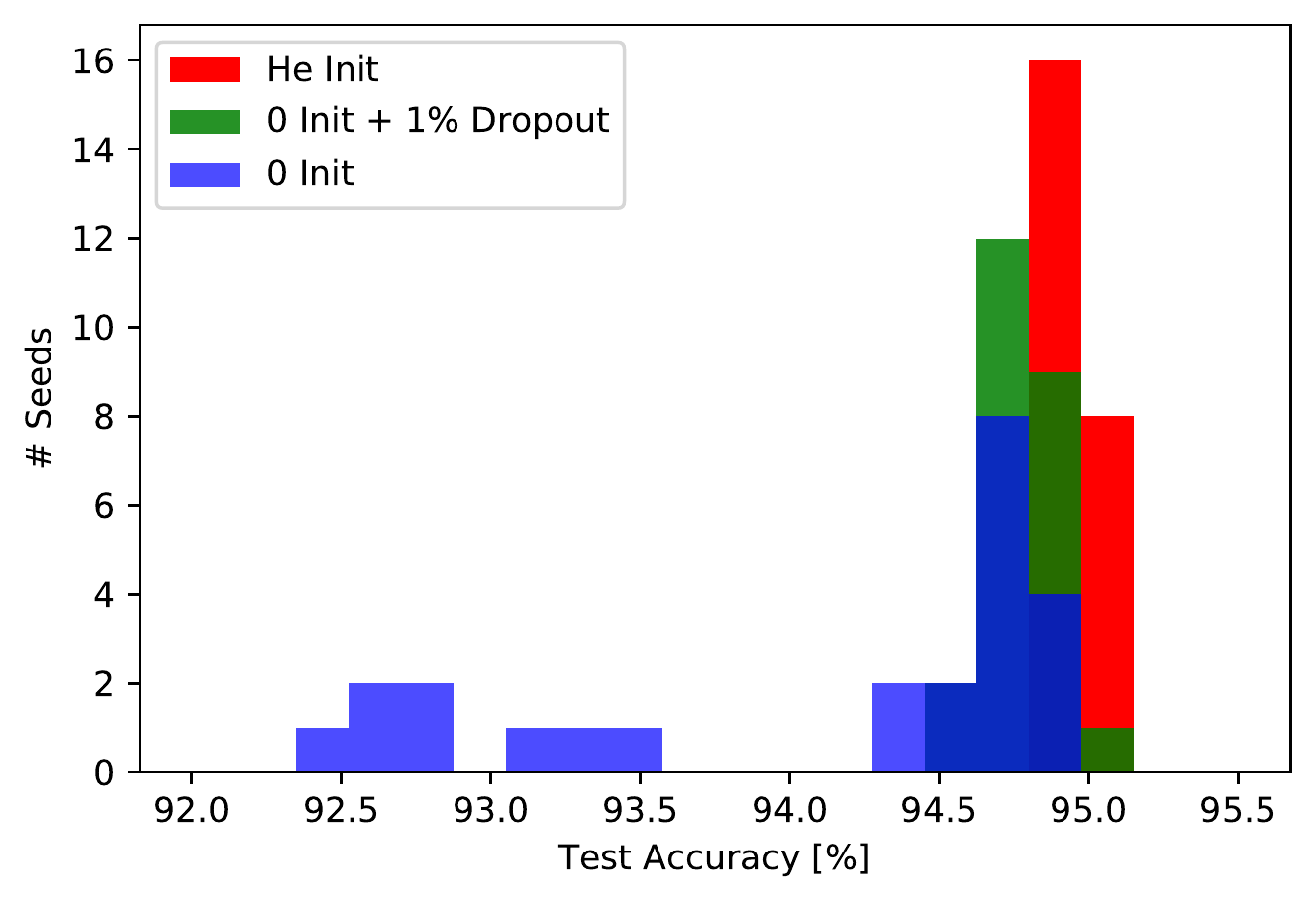}
    \end{subfigure}
    \begin{subfigure}{}
        \centering
        \includegraphics[width=2.8in]{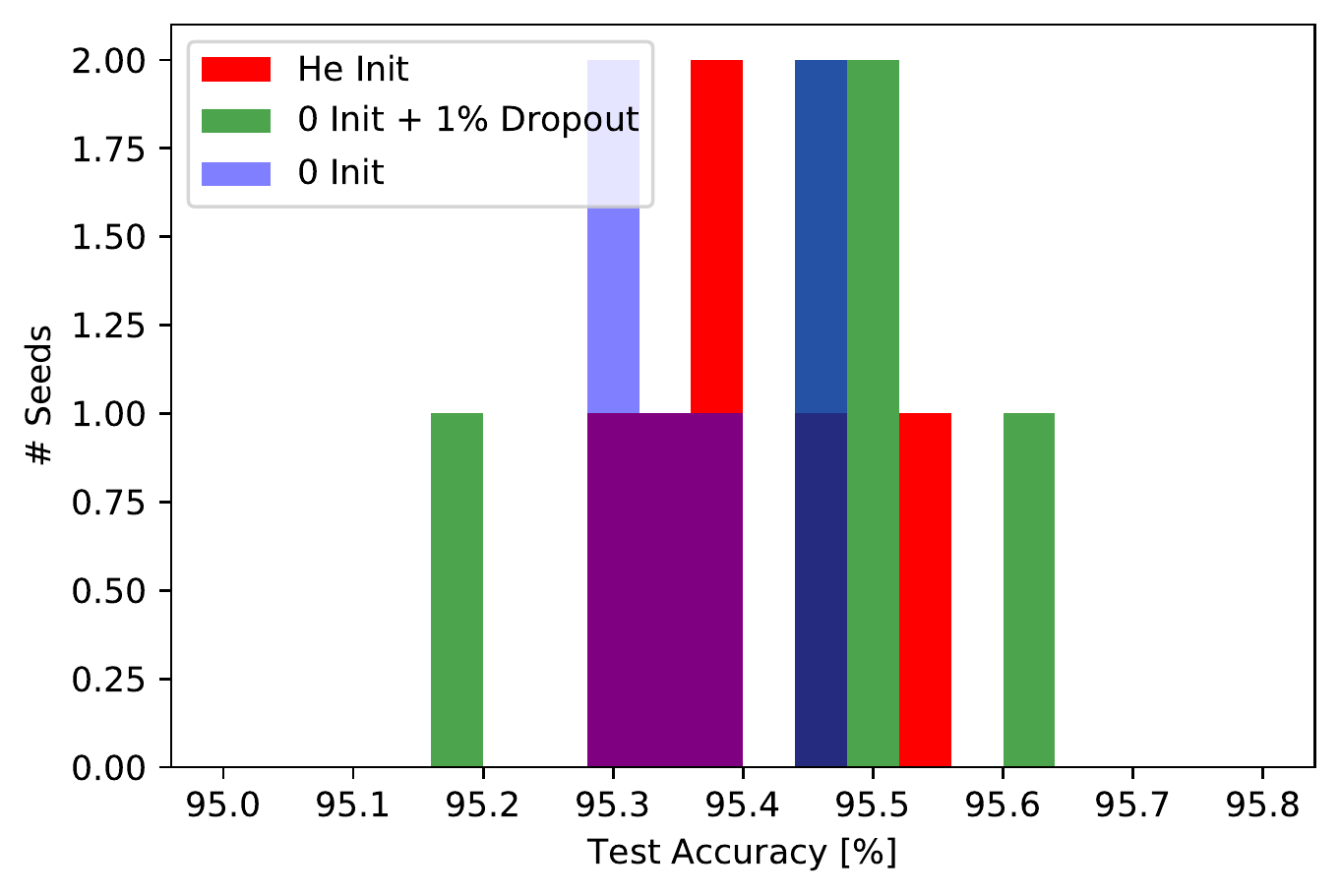}
    \end{subfigure}
    
    \caption{Distribution of results (Test accuracy) for training ConstNet with $200$ epochs (left panel) and $300$ epochs (right panel). Training ConstNet with hardware noise as the sole symmetry breaking mechanism is unreliable: On several instances, the symmetry breaking process was too slow for the model to fully recover during a $200$ epochs training.}
    \label{fig:histogram}
\end{figure*}

\begin{table}[b]
\begin{center}
 \begin{tabular}{| c | c | c | c | c| c|}
  \hline

 Data-set &   \multicolumn{4}{c|}{Test accuracy [\%]}   \\ [1.0ex]
\cline{2-5}
          &   \multicolumn{2}{c|}{Vanilla}    & \multicolumn{2}{c|}{Mixup  $(\alpha = 1.0)$}  \\ [1.0ex]
\cline{2-5}
           & '$0$'-Init          & i.i.d-Init      & '$0$'-Init          &    i.i.d-Init   \\ [1.0ex]
 \hline
Cifar10    & $94.82 \pm 0.14$    &$94.90 \pm 0.06$  & $96.22 \pm 0.08$   & $96.14 \pm 0.01$   \\ 
 \hline
Cifar100   &  $77.02 \pm 0.14$   & $76.42 \pm 0.20$ & $79.08 \pm 0.12$   &  $79.17 \pm 0.10$  \\
 \hline
CINIC   & $87.37 \pm 0.02$  &  $87.50 \pm 0.01$ &  $89.09 \pm 0.05$ &  $89.21 \pm 0.01$ \\
 \hline

\end{tabular}
\end{center}
\caption{\textbf{Generalization of the results}: Comparison of '$0$' and i.i.d initializations of ConstNet, for varying data-sets. In all runs, we used the training parameters as detailed in appendix \ref{sup:constnet_detailed}: specifically, we ran 2 seeds per configuration, with non- deterministic computation, and dropout with a drop rate of $1\%$. For all the data-sets and for all configurations, ConstNet achieved high accuracies, independent on the method of initialization (And thus, independent on the number of unique features at initialization).}
\label{tbl:generality}
\end{table}

\section{Additional results for LeakyNet}

In section \ref{sec:leakyNet}, we presented LeakyNet, and examined the effect of mixing averaging ('$\mathbb{1}$') and identity ('$I$') initializations, by initializing arbitrary layers using different initializations. As complementary results, we were also interested to see how the results may be affected by using a linear mixture of the different initializations --- so each of the $12$ layers is initialized to  $\gamma I + (1-\gamma) \mathbb{1}$ instead. Our results indicate that for any value of $\gamma$ larger than $0$, the training is equally successful, but is still significantly worse than training with random initialization. In this case, we can see that the forward correlations for each layer at the end of training are alternately similar/dissimilar, as can be seen in figure \ref{fig:gammaRun}.

\begin{figure}[h!]
    \centering
    \begin{subfigure}
        \centering
        \includegraphics[width=3.0in,height=1.85in]{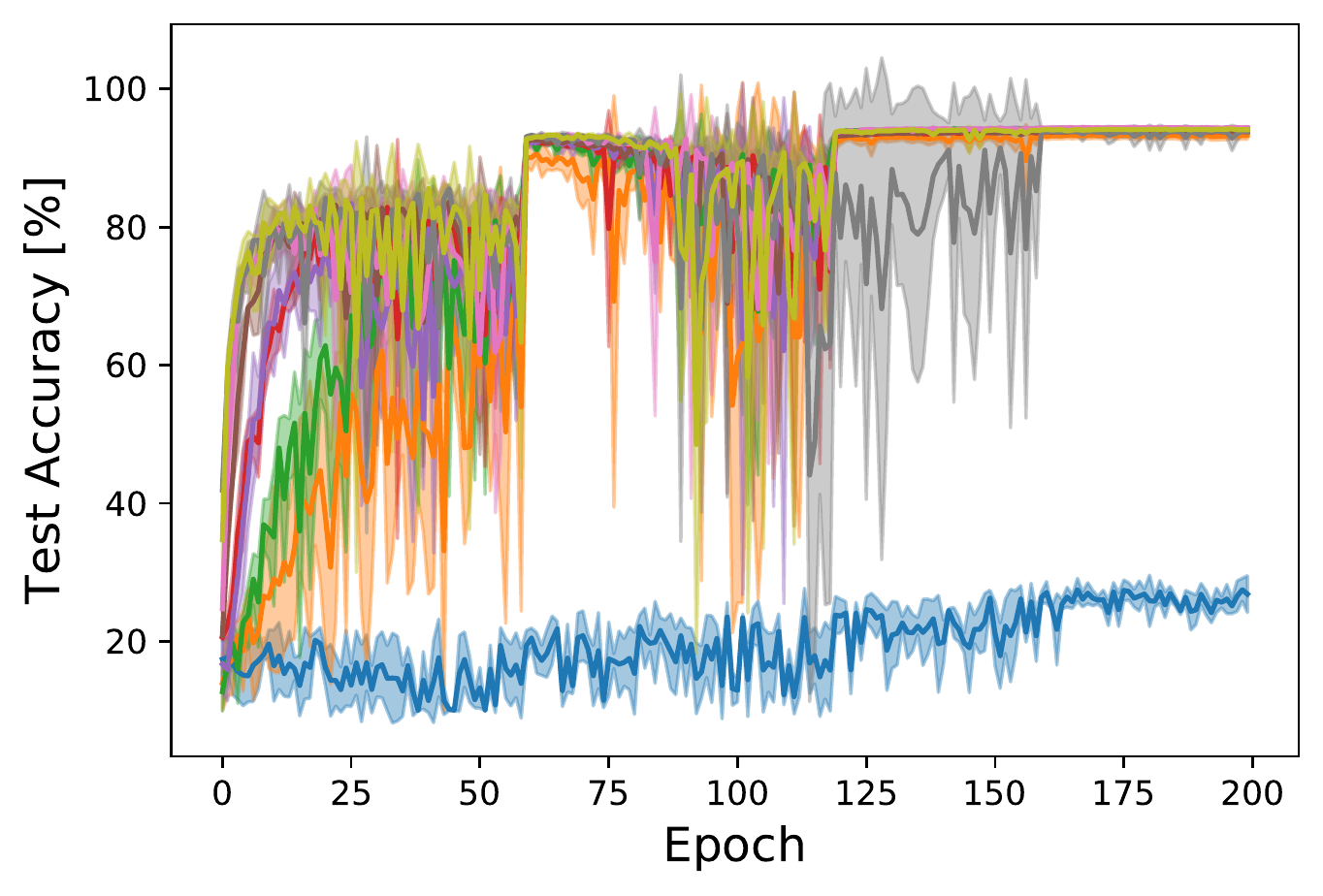}
    \end{subfigure}
    \begin{subfigure}{}
        \centering
    	\includegraphics[width=3.4in, height=1.89in]{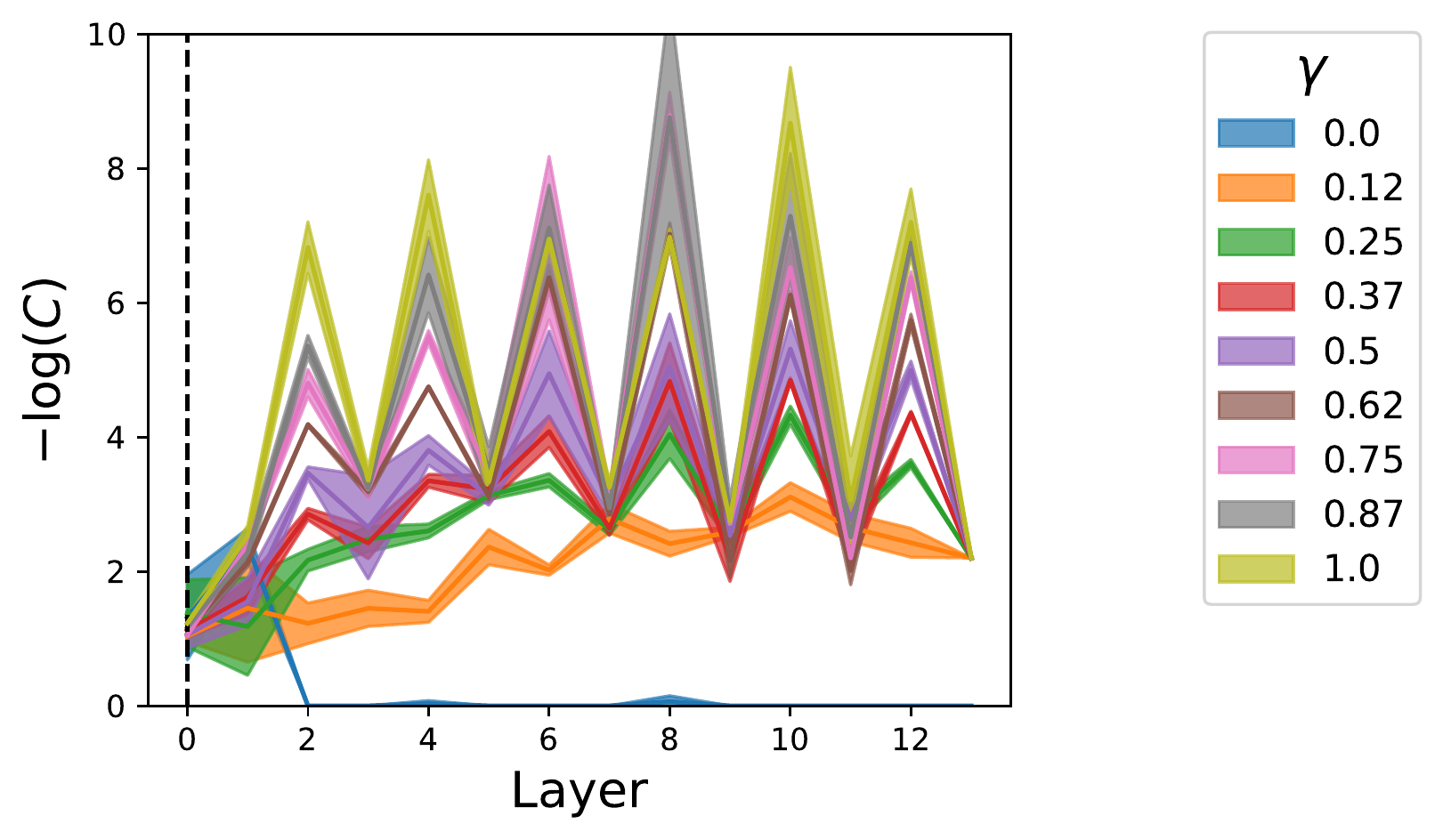}
    \end{subfigure}
    
    \caption{LeakyNets with layers initialized to a linear combination of $\mathbb{1}$ and '$I$' initializations (3 seeds per configuration). The final accuracy was similar for all runs with $\gamma > 0$. The right panel describes the final forward correlation for all layers. With layer 0 being the initial convolution, only the odd layers have similar forward correlation values. The forward correlations of the even layers at the end of the run are highly dependent on their values at initialization.}
    \label{fig:gammaRun}
\end{figure}

\begin{figure}[b!]
    \centering
    \begin{subfigure}
        \centering
        \includegraphics[width=3.0in]{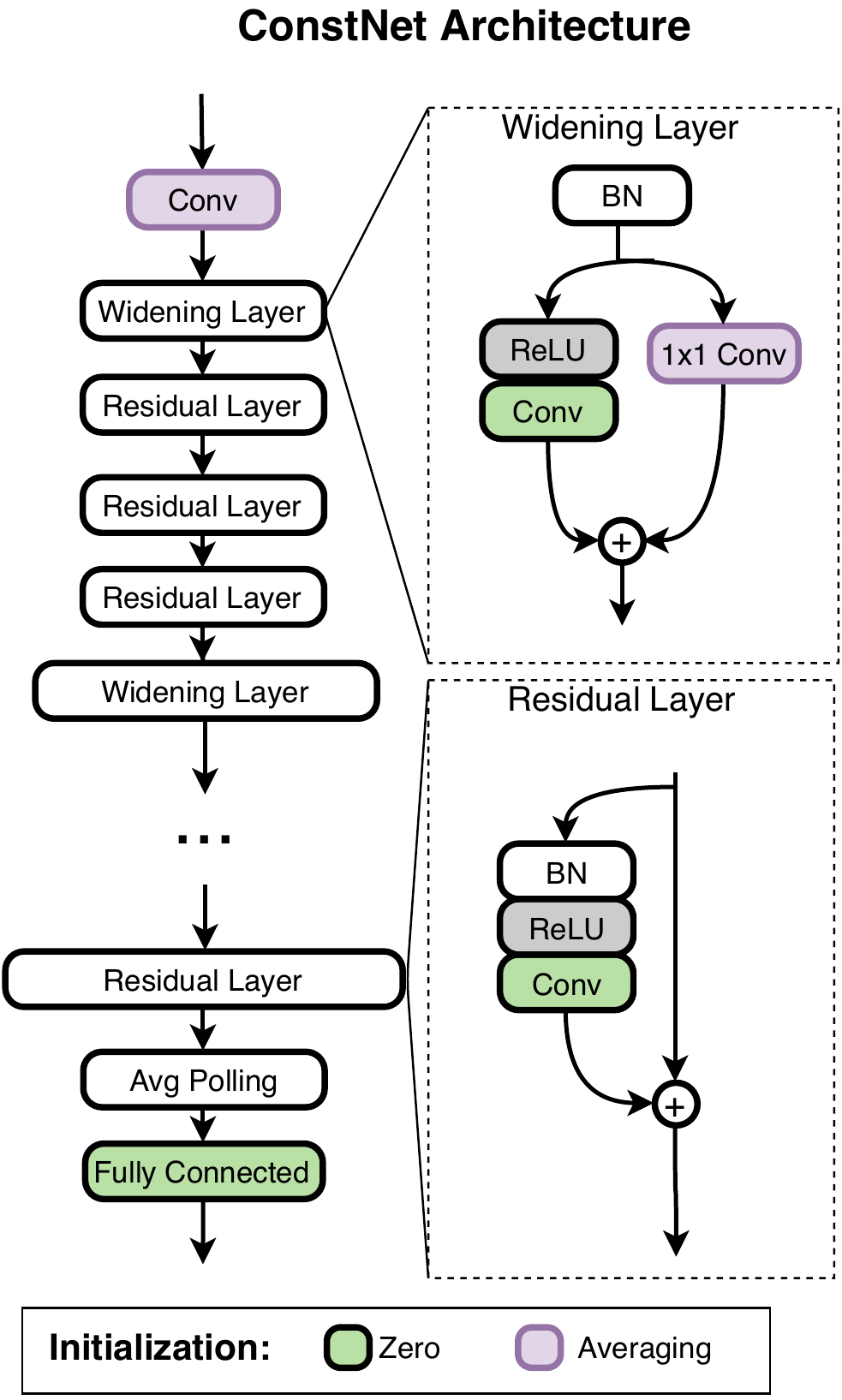}
    \end{subfigure}
    \caption{Detailed architecture of ConstNet.}
    \label{fig:constnetArch}
\end{figure}

}

\end{document}

%% file: sections/replicated_features.tex
\section{Replicating Features in a Single Hidden Layer}\label{sup:replicated features}

It is clear that having more than a single neuron representing the same feature in our final trained model is redundant --- a group of neurons that are identical for all inputs will not contribute to a successful classification. We therefore ask the question: given that our network was initialized so the same neurons in a layer represent the same feature, will they diverge and contribute to the model accuracy during the training process?

To answer this question, we consider a 2-layers fully connected neural network:
\begin{equation}\label{eq:mnist_model}
\begin{array}{c}
X\in\mathbb{R}^{d}, W_{1}\in\mathbb{R}^{ N \times d}, W_{2}\in\mathbb{R}^{ \text{\#classes}  \times N} \\
h(X,W_1) = \text{Relu}\left(W_{1}X\right) \\
f(X,W_1,W_2) = \text{SoftMax}\left(W_{2}h(x,W_1)\right)
\end{array}
\end{equation}
where we initialize the neural network as follow: the weight matrix $W_2$ is initialized using the standard He initialization \cite{he2015delving}, and for each number of features $K$, we initialize the temporary matrices $\tilde{W}_1\in\mathbb{R}^{ K \times d}, \widehat{W}_1\in\mathbb{R}^{ N \times d}$ using the He initialization, and initialize $W_1$ using $r=\frac{N}{K}$ replications of $\tilde{W}_1$, so 
\begin{equation}
\forall i, W_1[i,:]=\tilde{W}_1[ i\text{ mod }k,:](1-\lambda) + \lambda \widehat{W}_1[i,:] .
\end{equation}
where $\lambda$ is a parameter simulating noise. For $\lambda=0, K>1$, this would result in each row in $W_1$ being repeated $r$ times at initialization (and consequently, the same also applies to each hidden layer neuron). We identify two main possible causes for initially identical neurons to diverge during training: First, stochastic operations can result similar rows in $W_1$ receiving different gradient updates. The most commonly used operation that would achieve this is dropout, which will randomly mask neurons, so some neurons may freeze while their "replicas" change. The second cause is back-propagation itself: Even when not using dropout, the gradient of $W_1$ will be: $\frac{dL}{dW_{i,j}}=\frac{dL}{dh_i}\frac{dh_i}{dW_{i,j}}=\frac{dL}{dh_i}X_j$ and since $\frac{dL}{dh_i}$ depends on the corresponding column $(W^{T}_{2})_i$ (which is random at initialization), the update gradient may be different for each row.

The training was done with standard SGD, learning rate of 0.1, with test accuracy measured after 7500 training steps and 30 seeds per sample.

%% file: sections/subnetworks.tex
\section{Features Symmetry and Sub-networks}\label{sup:subnetworks}

Network with identical features has inherently less unique sub-networks at initialization, but the there could be different levels of symmetries. In the case where all features have all-to-all connection with neighbouring layers (as was done in ConstNet and LeakyNet with '1' Init), all features at each layer are interchangeable. Therefore, when considering possible way to mask \textit{neurons}, all sub-networks that mask the same amount of neurons at all layers are equivalent. For example, a randomly initialized neural network with $L$ layers of width $d$ can have up to $2^{dL}$ possible unique sub-networks, while a symmetrical features, all-to-all initialization ensures no more than $d^L=2^{\log_2(d)L}$. 

If the features are initialized to have a 1:1 connection with the following layer, as was done in the case LeakyNet with identity initialization, given a network of $L$ layers if width $d$, we have $2^L$ ways to mask each feature, when only features with the same masks at all layers are interchangeable. We can roughly approximate the number of sub-networks in this case by:
\begin{equation*}
\left(\begin{array}{c}
2^{L}+d-1\\
d
\end{array}\right)\propto \frac{\left(2^{L}-d\right)^{2^{L}}}{d^{d}\left(2^{L}\right)^{2^{L}}}\approx\frac{\left(2^{L}+d\right)^{2^{L}+d}}{d^{d}\left(2^{L}\right)^{2^{L}}}\stackrel{d\ll L}{\to}\frac{2^{dL}}{d^{d}}.
\end{equation*}{}

%% file: sections/fullSymBreak.tex
\renewcommand{\figscale}{2.2in}
\renewcommand{\figspacescale}{-0.3in}

\begin{figure*}[ht!]
    
    \centering

    \begin{subfigure}{}
        \centering
    	\includegraphics[width=\figscale]{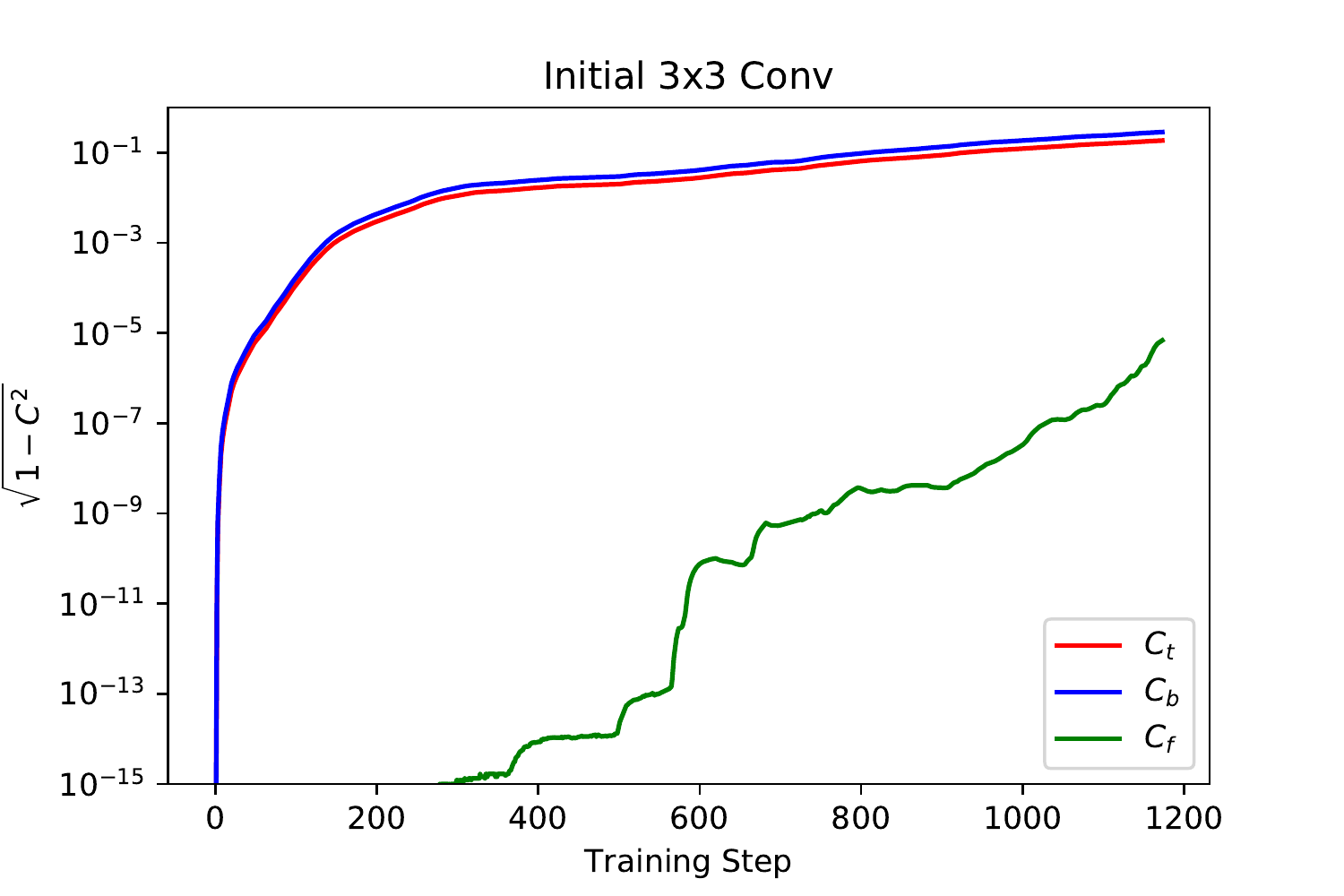}
    \end{subfigure}
    \hspace{0.5in}
    \begin{subfigure}{}
        \centering
    	\includegraphics[width=\figscale]{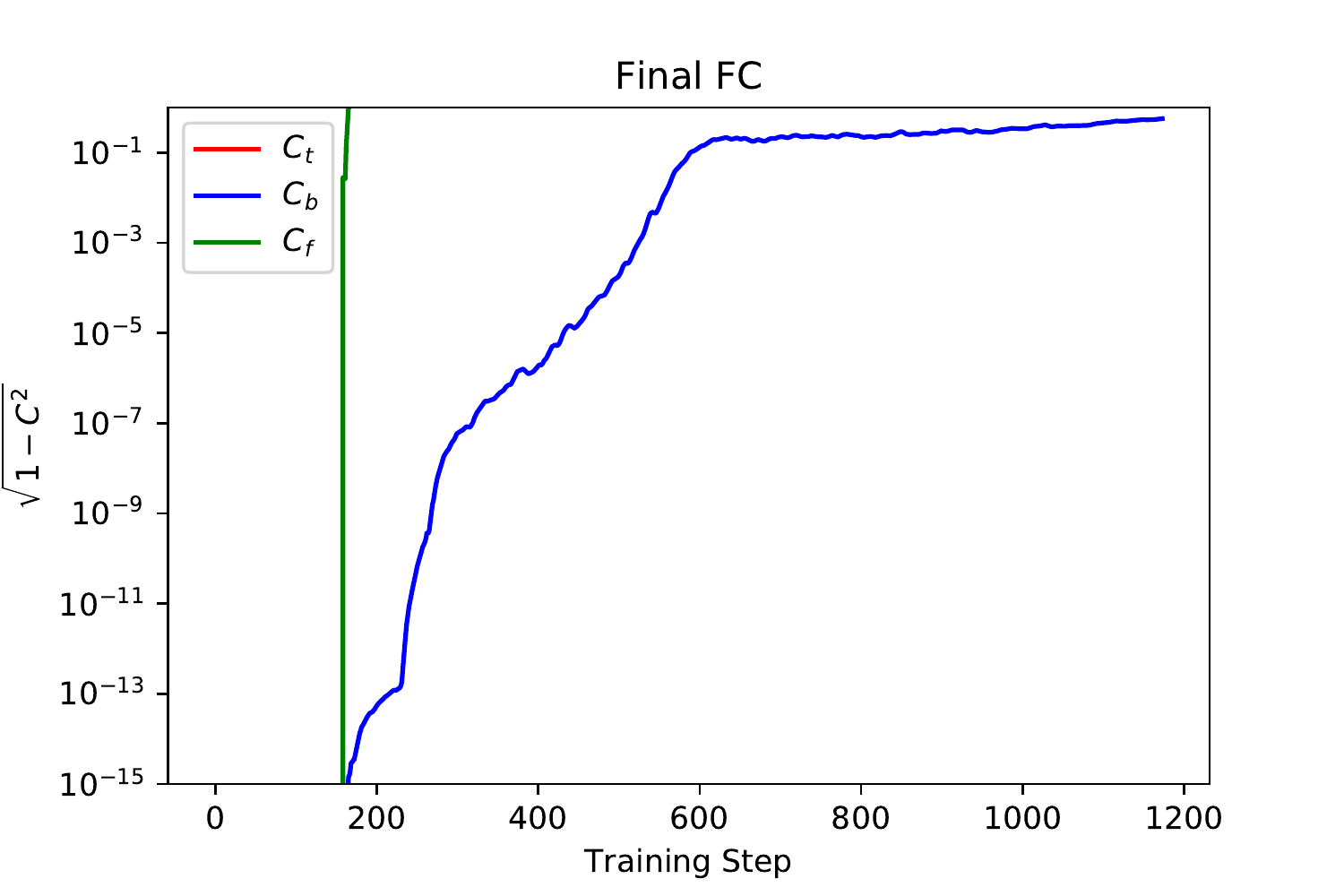}
    \end{subfigure}%
    \\
    \begin{subfigure}{}
        \centering
    	\includegraphics[width=\figscale]{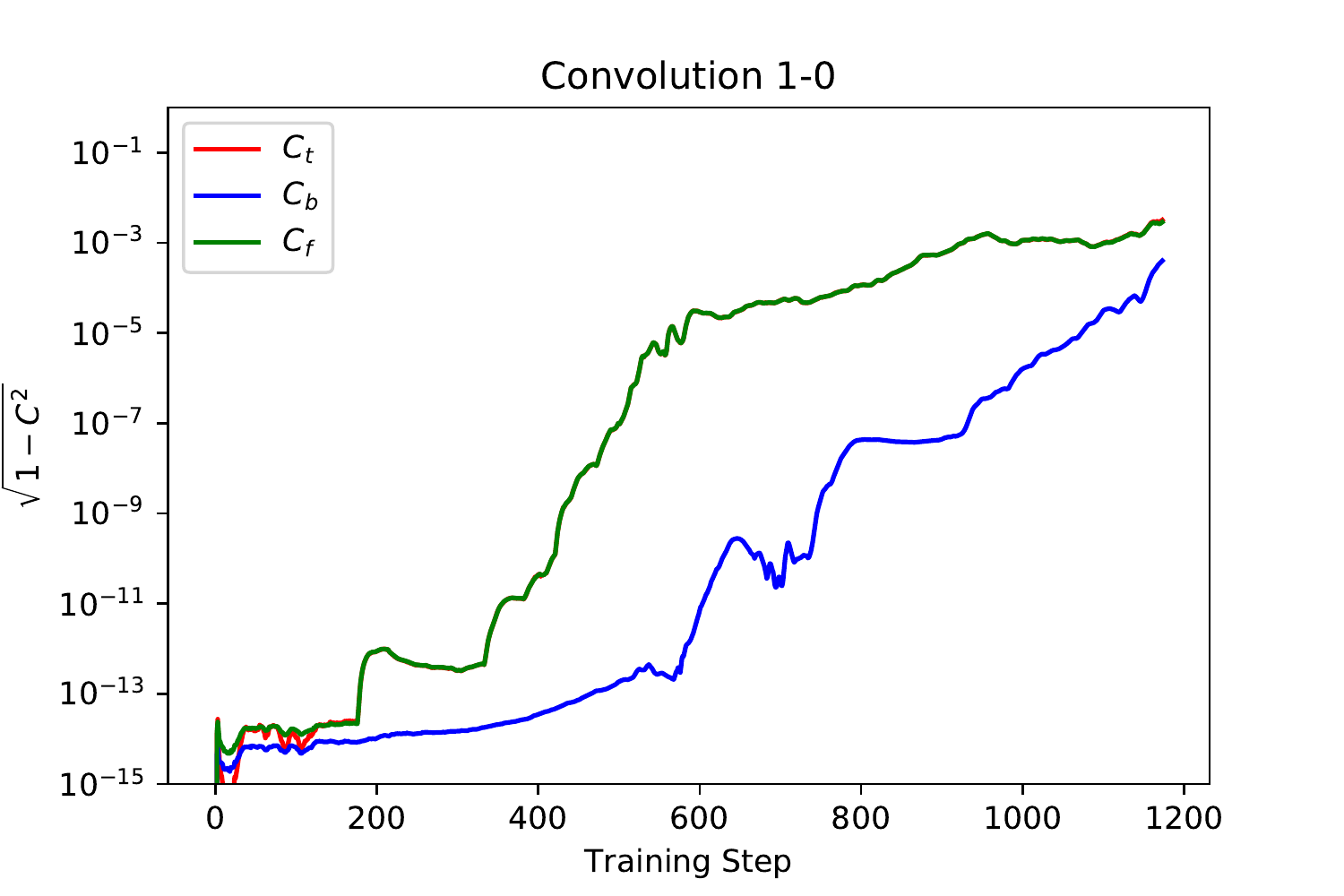}
    \end{subfigure}
    \hspace{\figspacescale}
    \begin{subfigure}{}
        \centering
    	\includegraphics[width=\figscale]{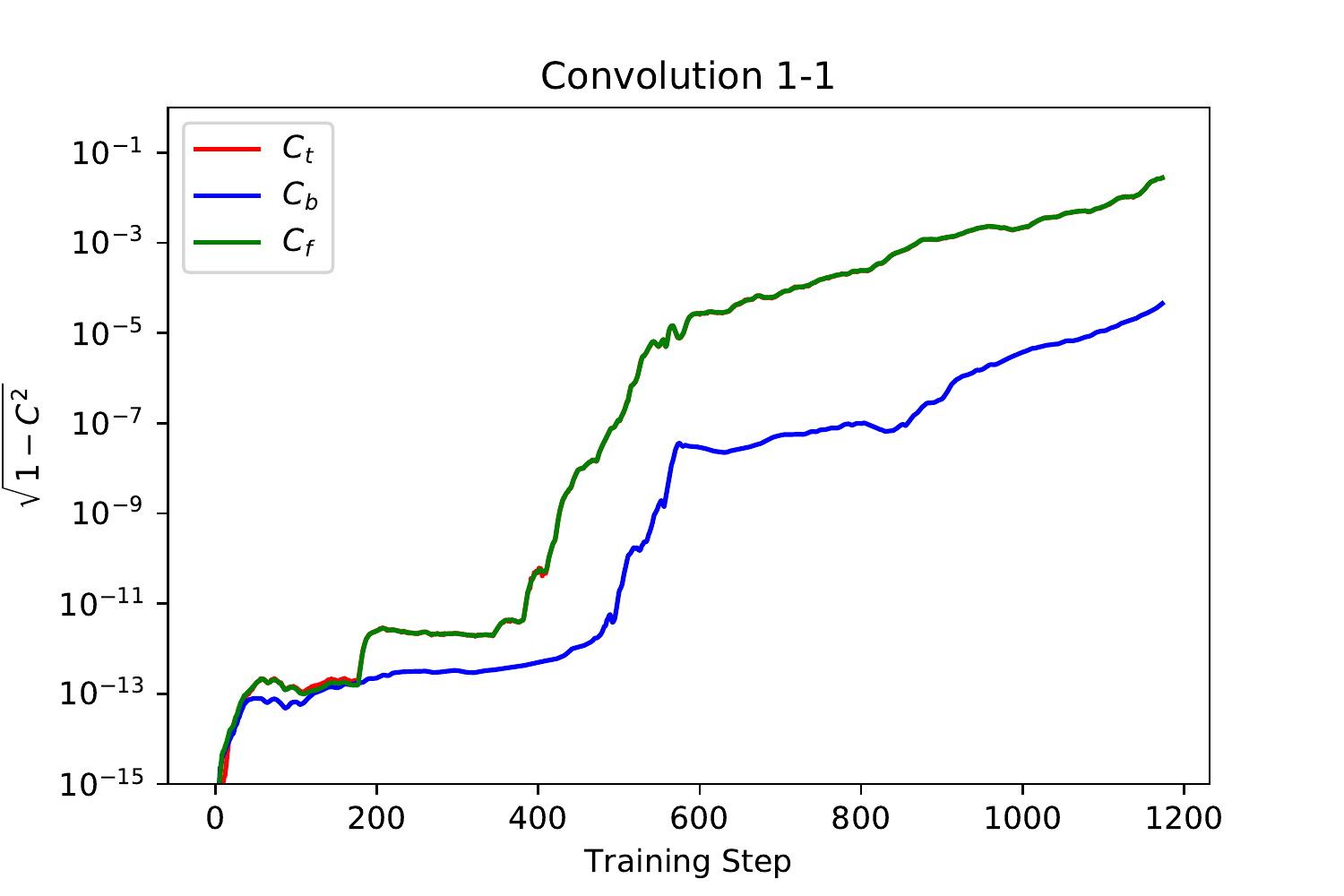}
    \end{subfigure}
    \hspace{\figspacescale}
    \begin{subfigure}{}
        \centering
    	\includegraphics[width=\figscale]{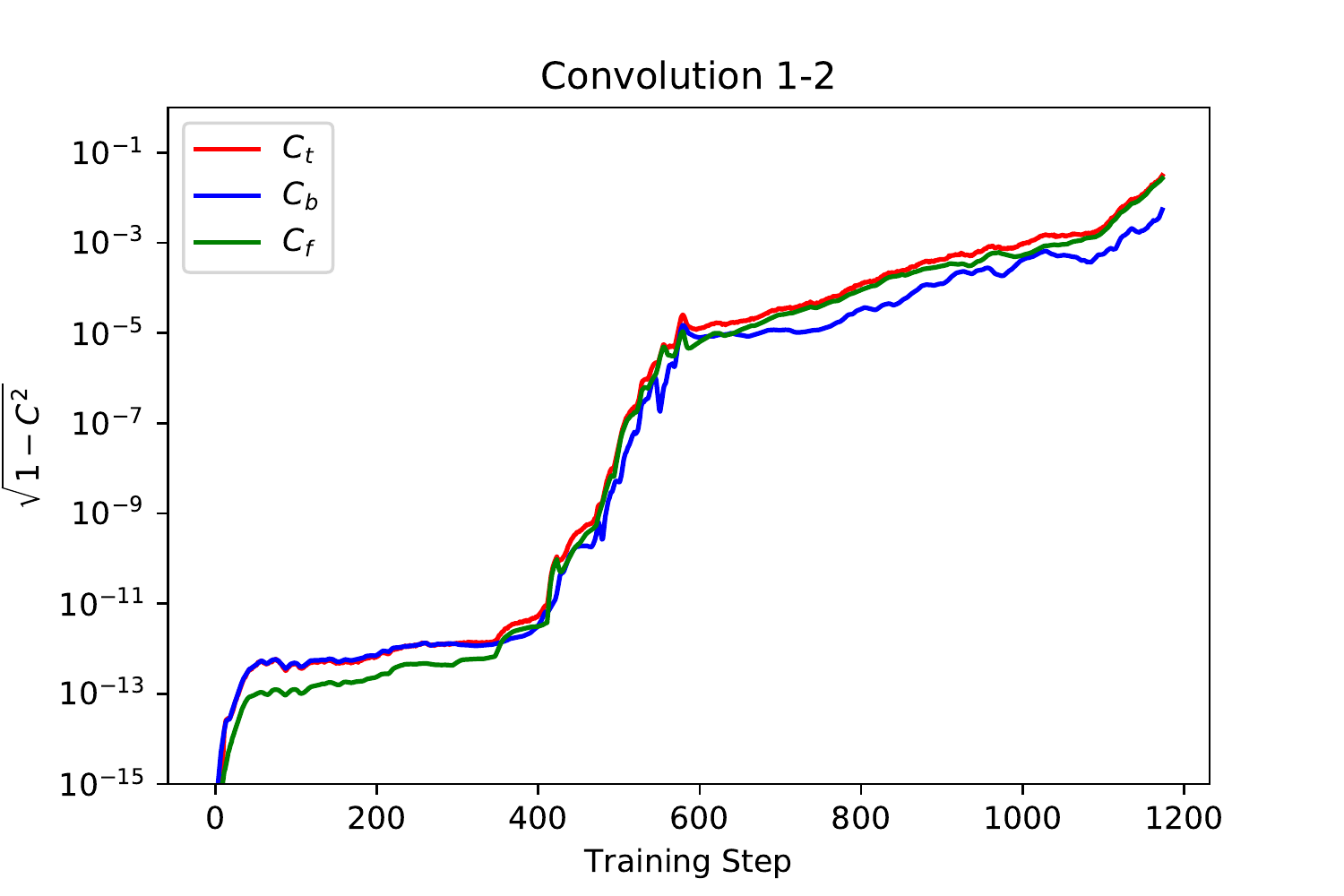}
    \end{subfigure}%
    \\
    \begin{subfigure}{}
        \centering
    	\includegraphics[width=\figscale]{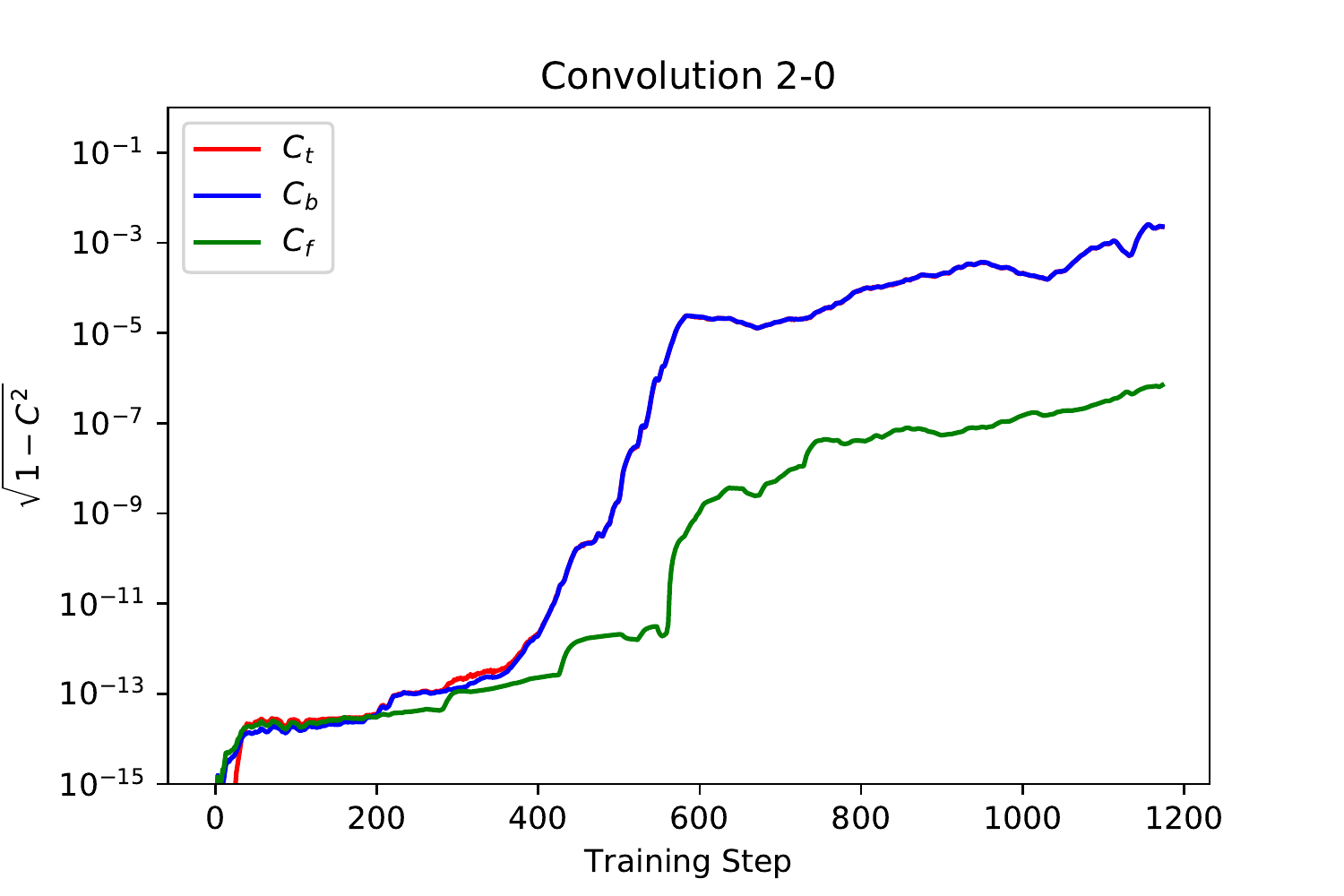}
    \end{subfigure}%
        \hspace{\figspacescale}
    \begin{subfigure}{}
        \centering
    	\includegraphics[width=\figscale]{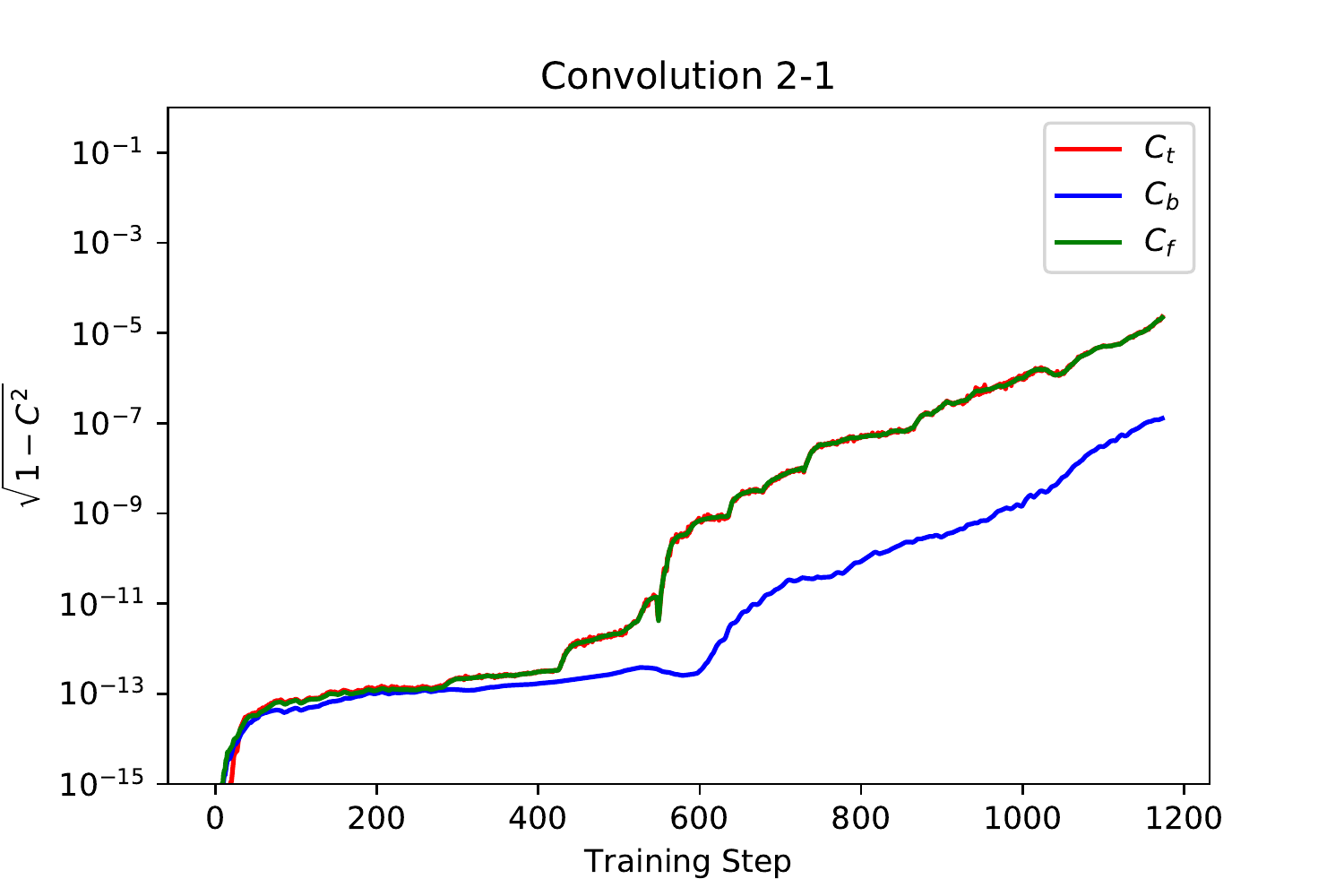}
    \end{subfigure}%
        \hspace{\figspacescale}
    \begin{subfigure}{}
        \centering
    	\includegraphics[width=\figscale]{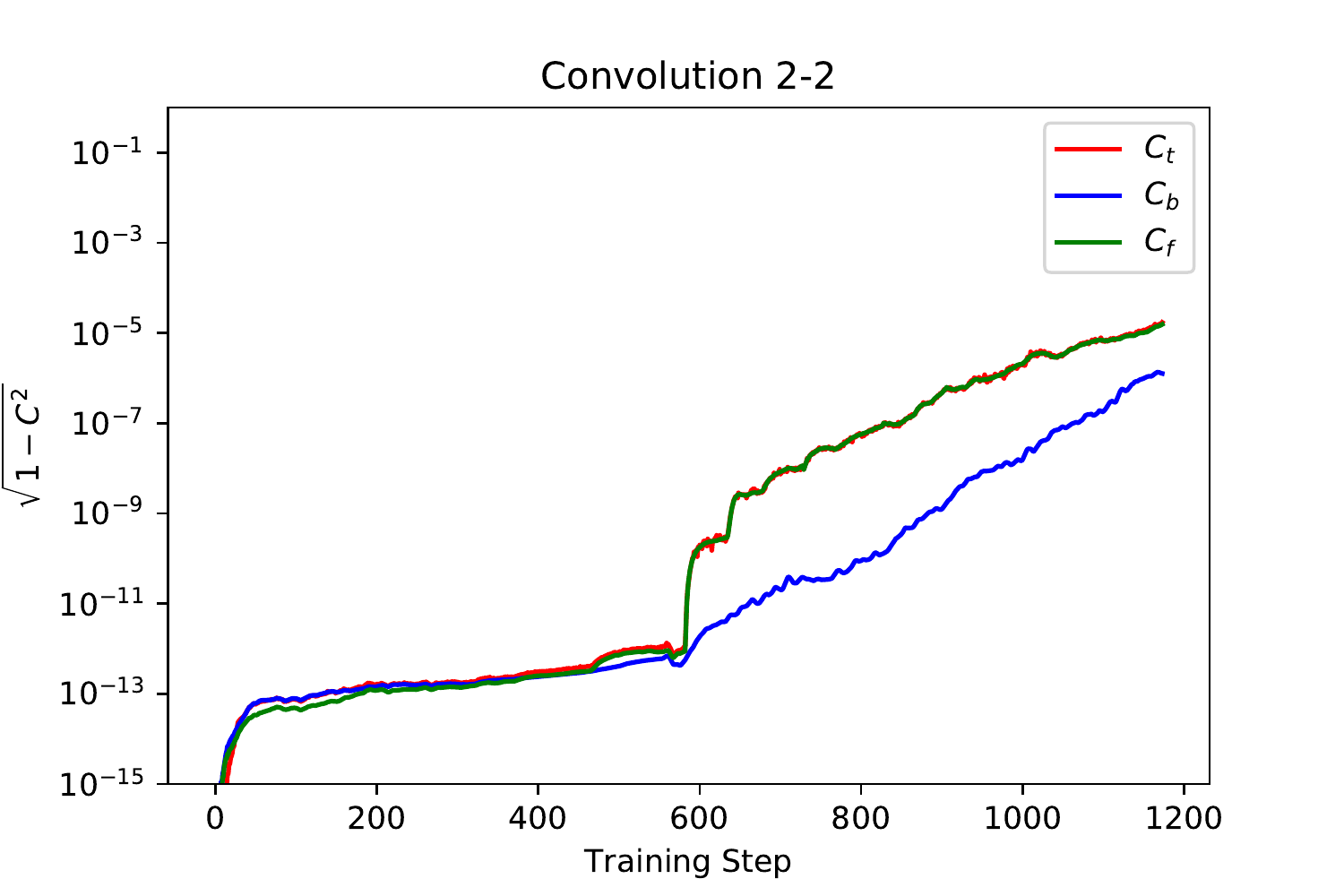}
    \end{subfigure}%
    \\
    \begin{subfigure}{}
        \centering
    	\includegraphics[width=\figscale]{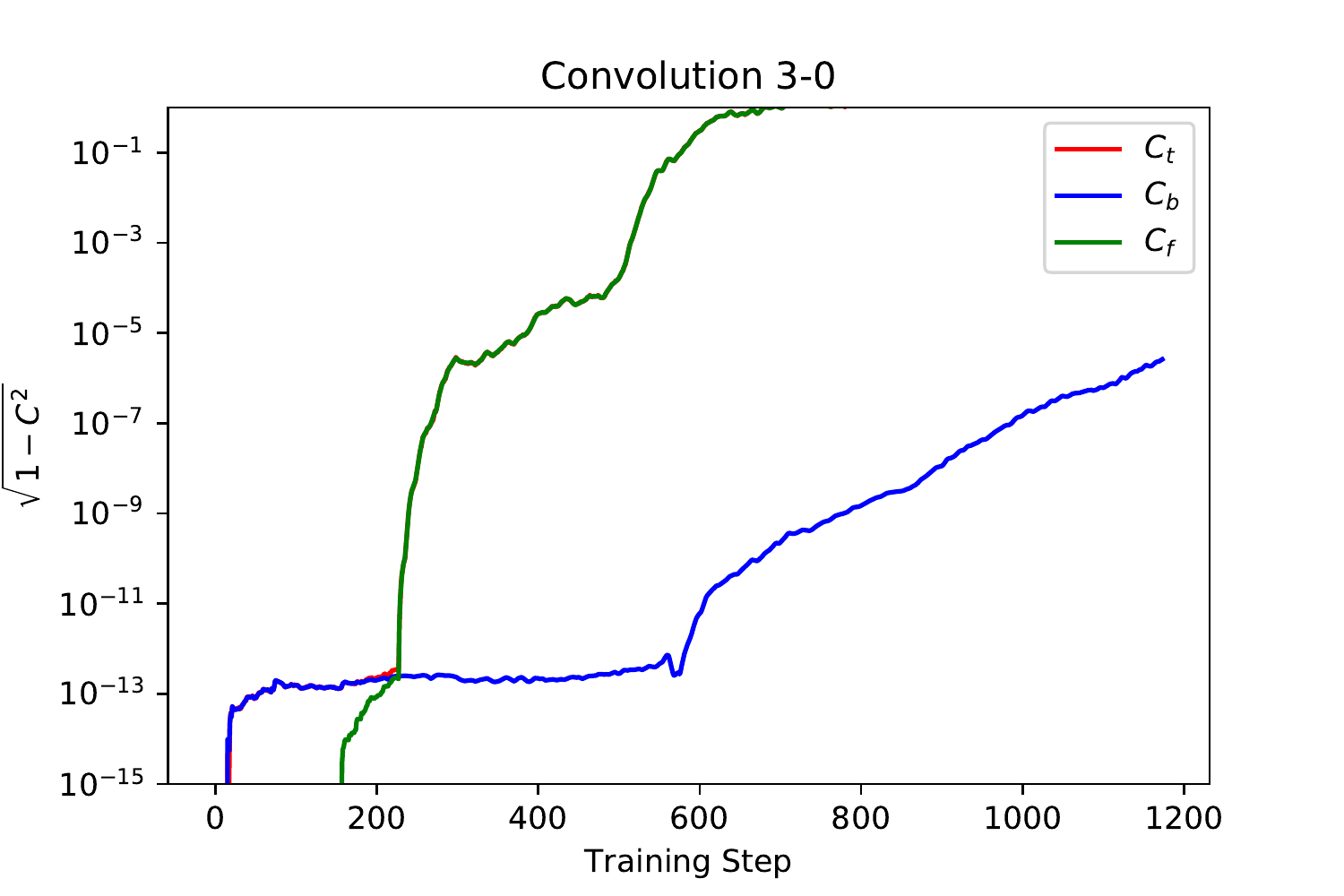}
    \end{subfigure}
    \hspace{\figspacescale}
    \begin{subfigure}{}
        \centering
    	\includegraphics[width=\figscale]{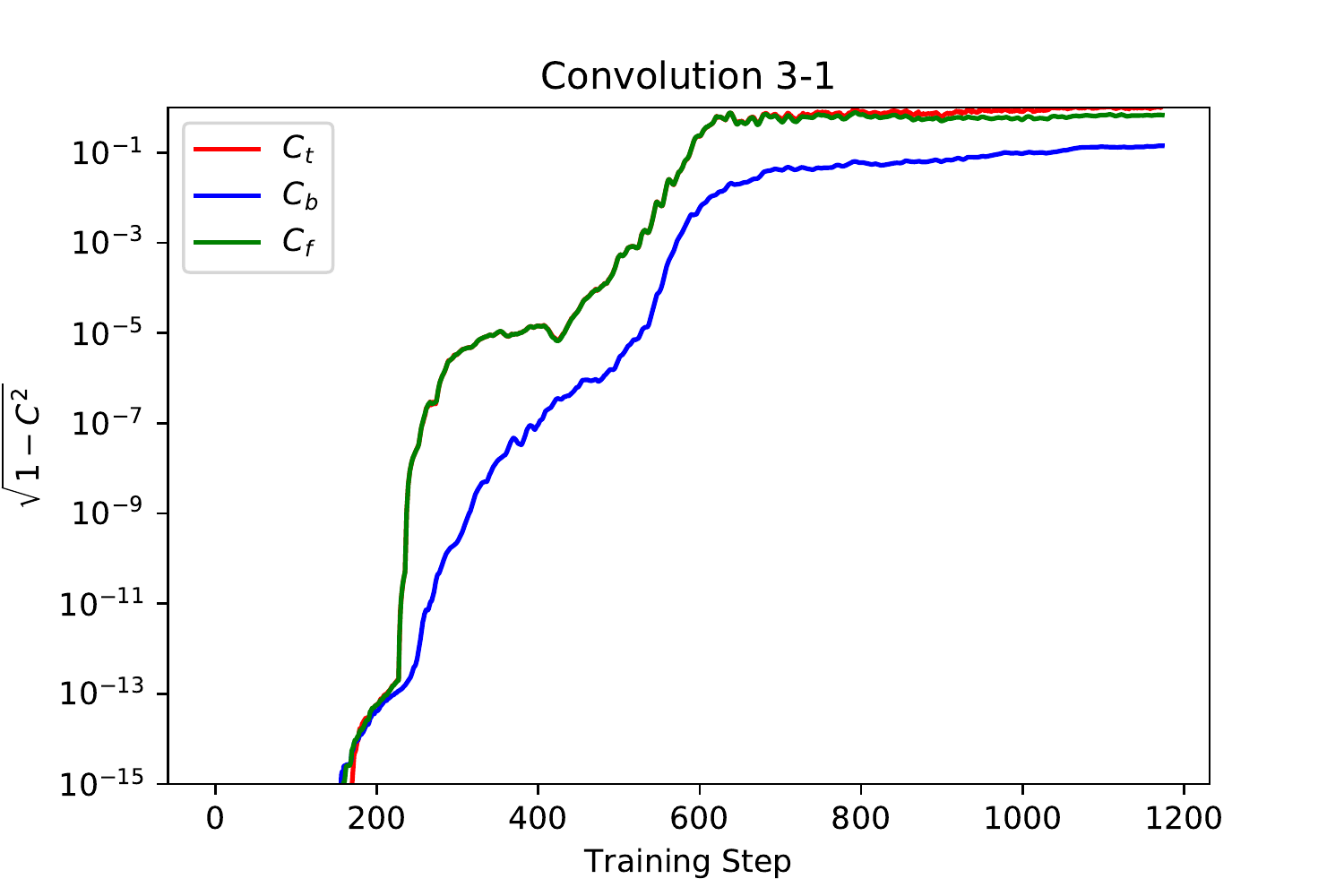}
    \end{subfigure}
    \hspace{\figspacescale}
    \begin{subfigure}{}
        \centering
    	\includegraphics[width=\figscale]{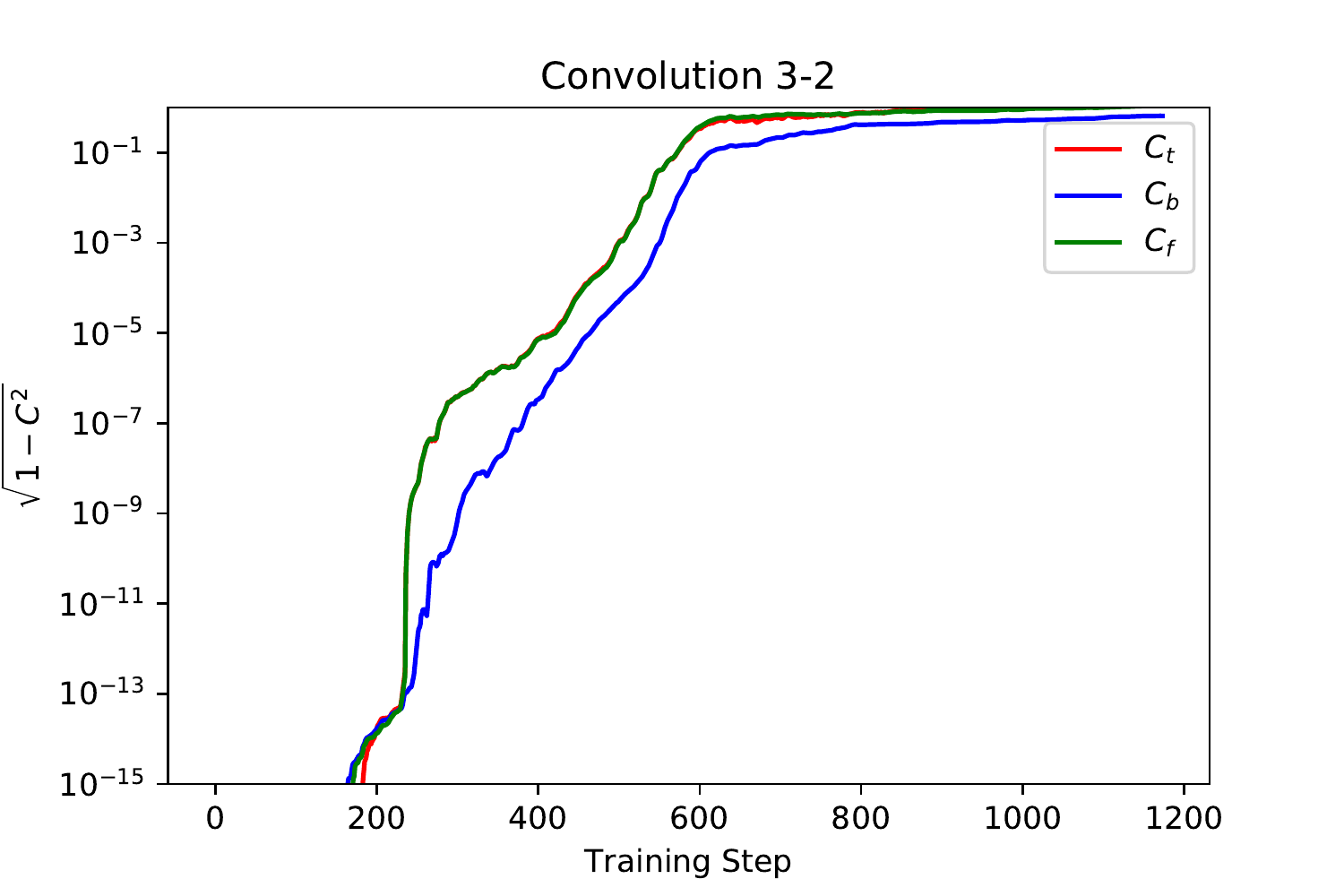}
    \end{subfigure}%
    \\
    \begin{subfigure}{}
        \centering
    	\includegraphics[width=\figscale]{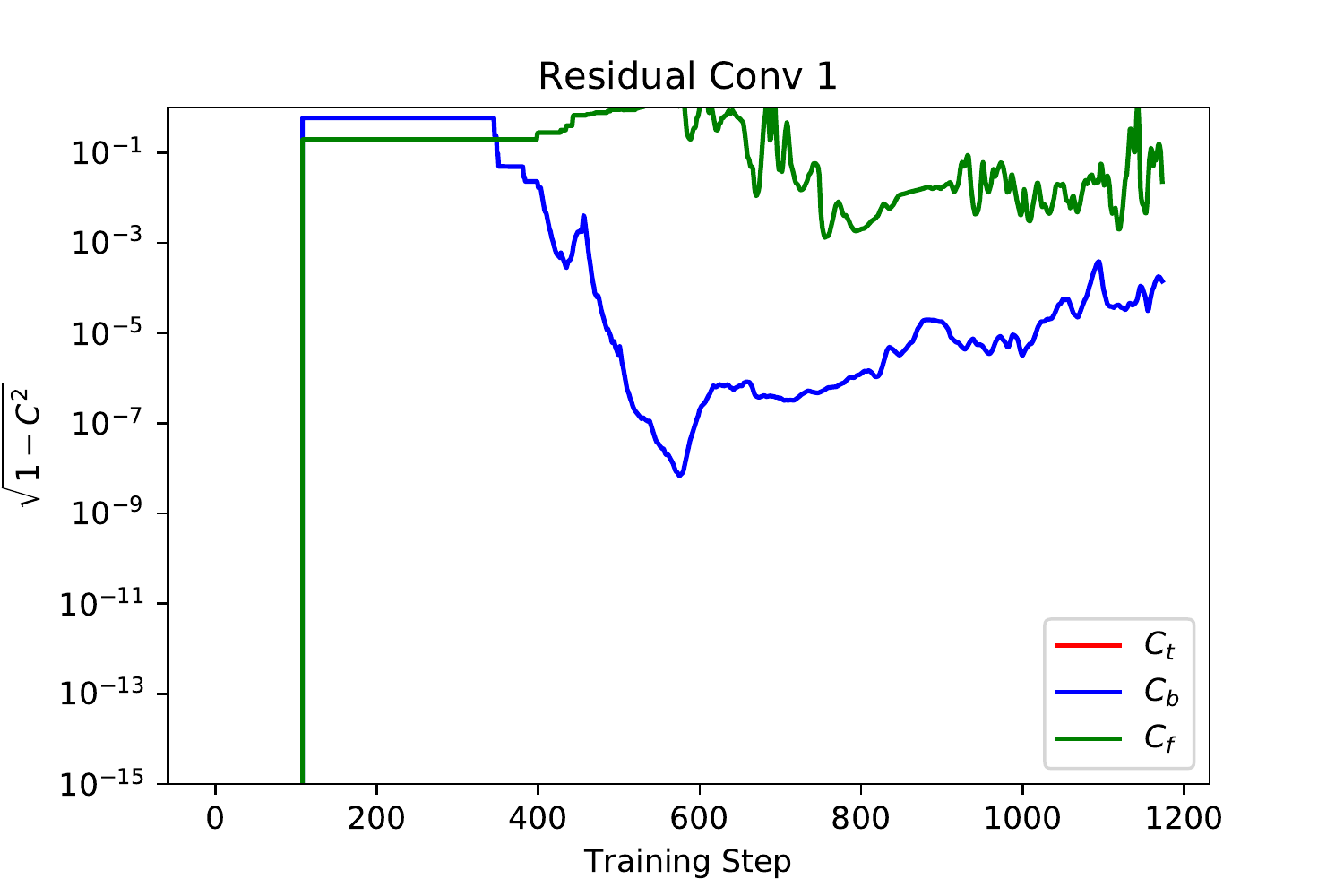}
    \end{subfigure}
    \hspace{\figspacescale}
    \begin{subfigure}{}
        \centering
    	\includegraphics[width=\figscale]{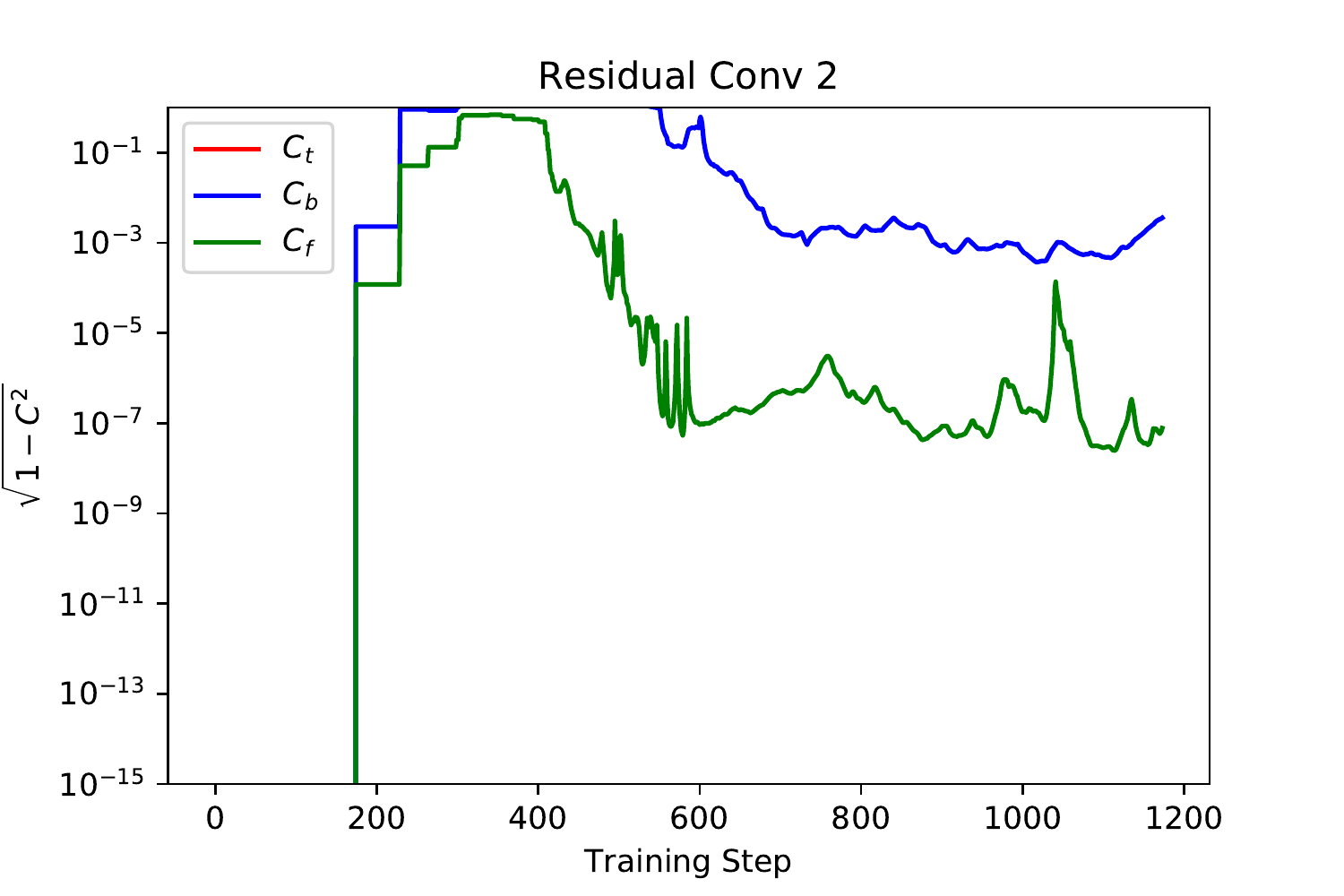}
    \end{subfigure}
    \hspace{\figspacescale}
    \begin{subfigure}{}
        \centering
    	\includegraphics[width=\figscale]{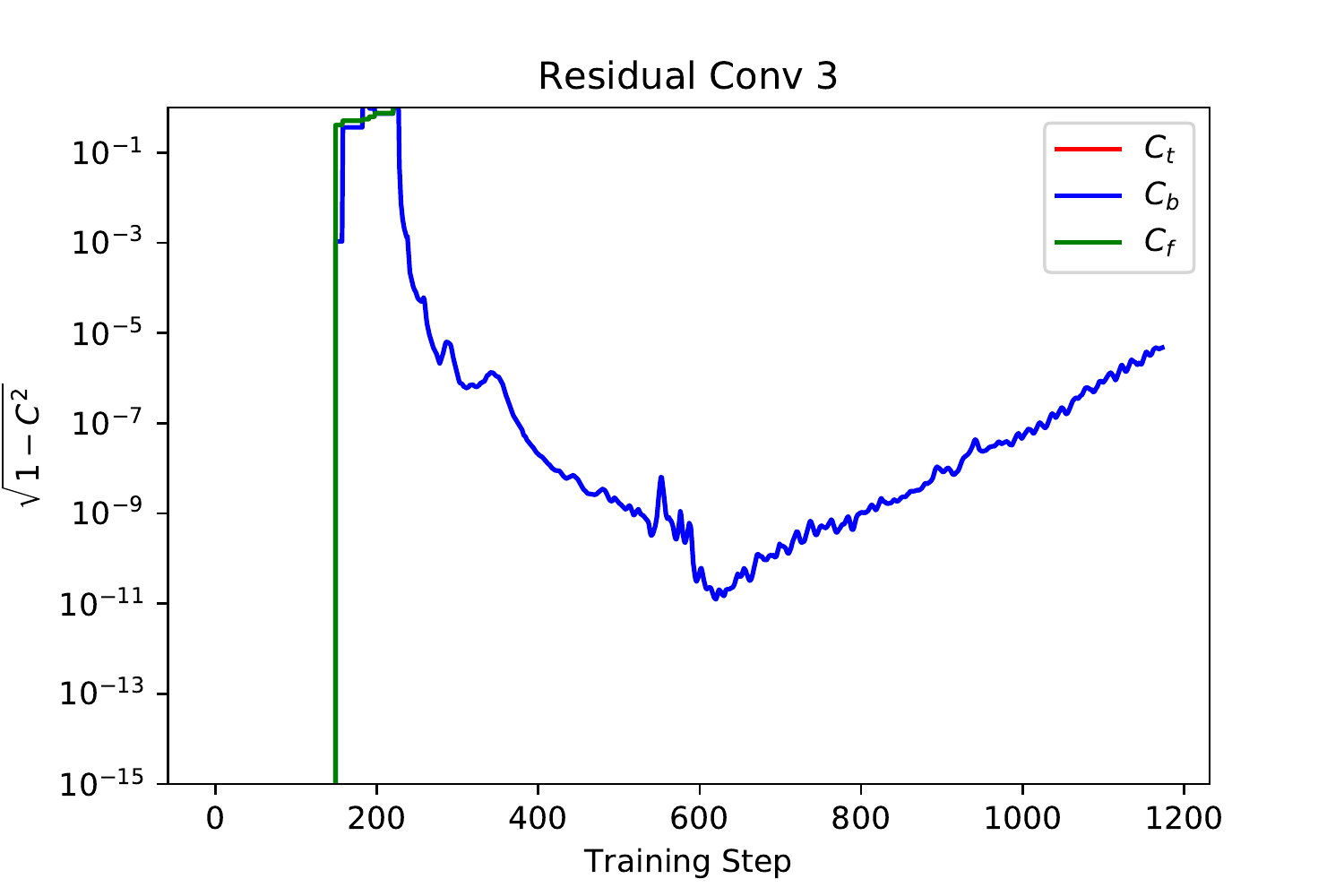}
    \end{subfigure}%
    \\
    \caption{Feature symmetry in each of the layers ConstNet ('$0$'-init, $9$ skip-able layers), on the first 3 epochs of training. The symmetry is measured using forward correlation $C_f$, backward correlation $C_b$, and total correlation $C_t$ (Correlation between all filters), in all of the weight tensors in the network. Skip-able operations are numbered based on their groups (shown in figure \ref{fig:groups}), and position within the group: Tensors belong to the same groups, if the output of their respective operations is summed together (with or without passing through an activation operator). It is apparent that the forward correlations behave in a similar manner across operations for the $3\times3$ convolutions, while the residual convolutions tend to fluctuate and are less predictable. The backward correlations also behave in accordance to their group (though its group can be different from the group of the forward correlation in the same layer). e.g, the forward correlation of 'Convolution 3-0' is on group 3, while it's backward correlation, matches the behaviour of correlations group 2, to which it's input is connected.}
    \label{fig:fullSymBreak_AtInit}
\end{figure*}

%% file: sections/constNetRuns.tex
\renewcommand{\figscale}{2.0in}
\renewcommand{\figspacescale}{-0.0in}

\begin{figure*}[ht!]

    \centering
    \begin{subfigure}{}
        \centering
    	\includegraphics[width=3.00in, height=1.45in]{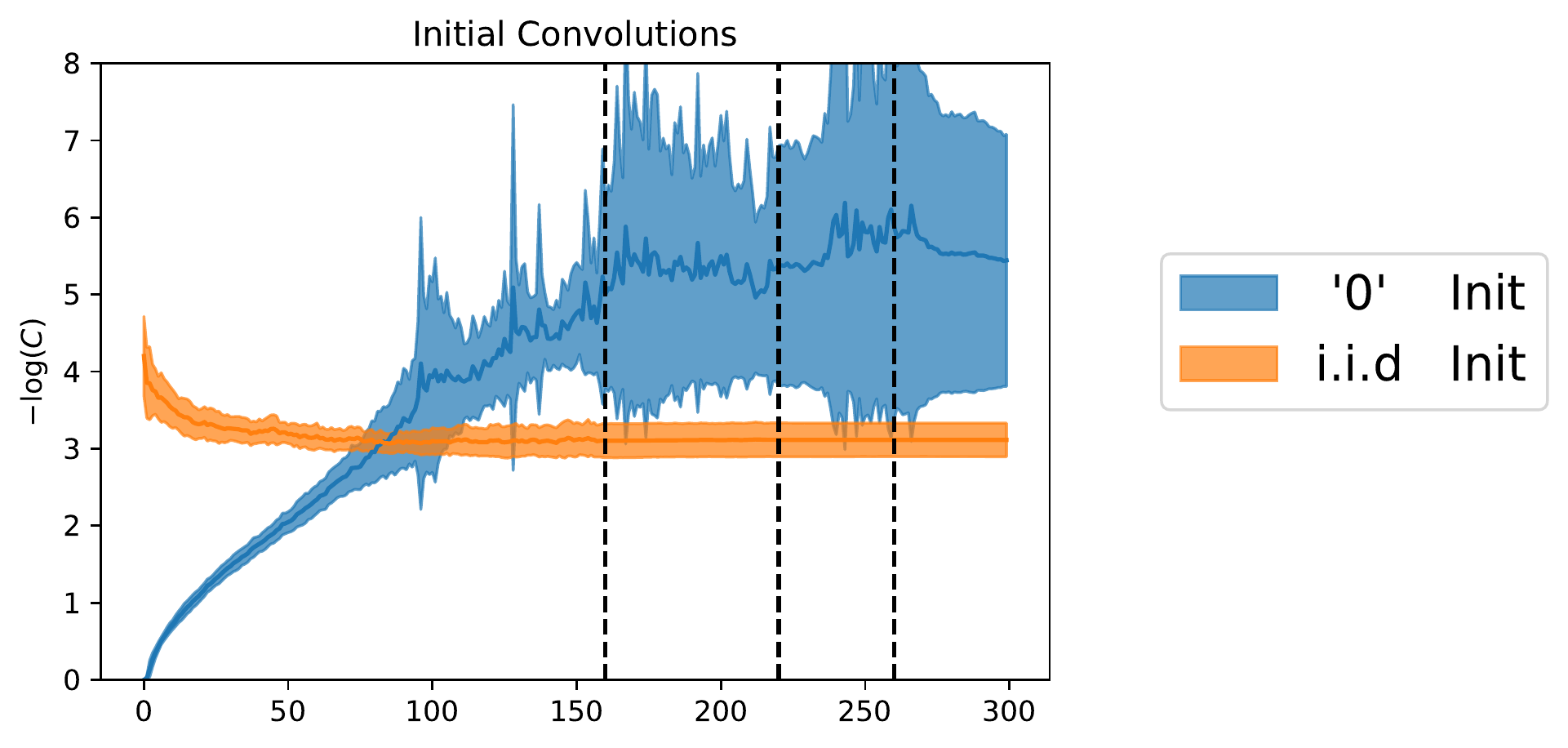}
    \end{subfigure}
    \begin{subfigure}{}
        \centering
    	\includegraphics[width=\figscale, height=1.45in]{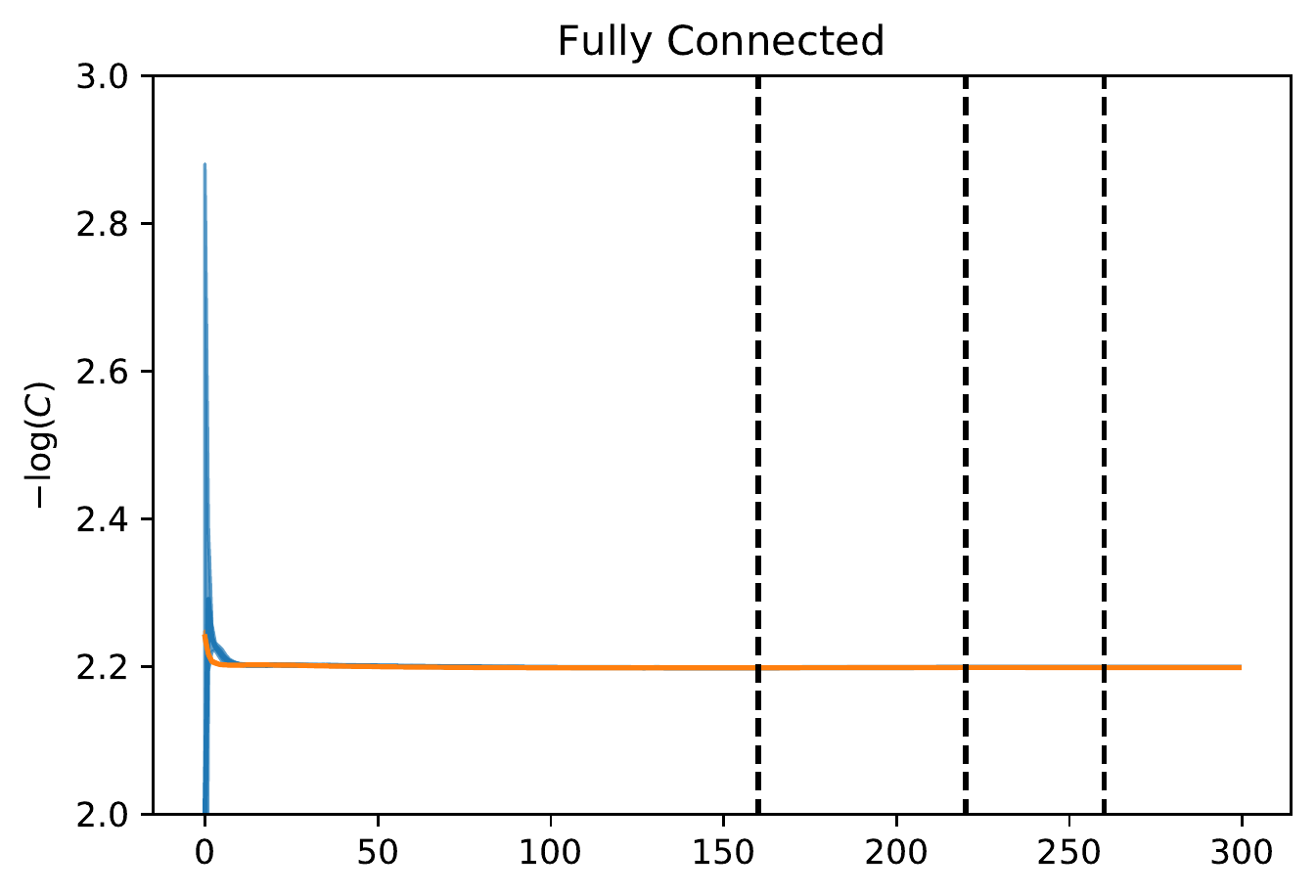}
    \end{subfigure}
    \\
    \begin{subfigure}{}
        \centering
    	\includegraphics[width=\figscale, height=1.3in]{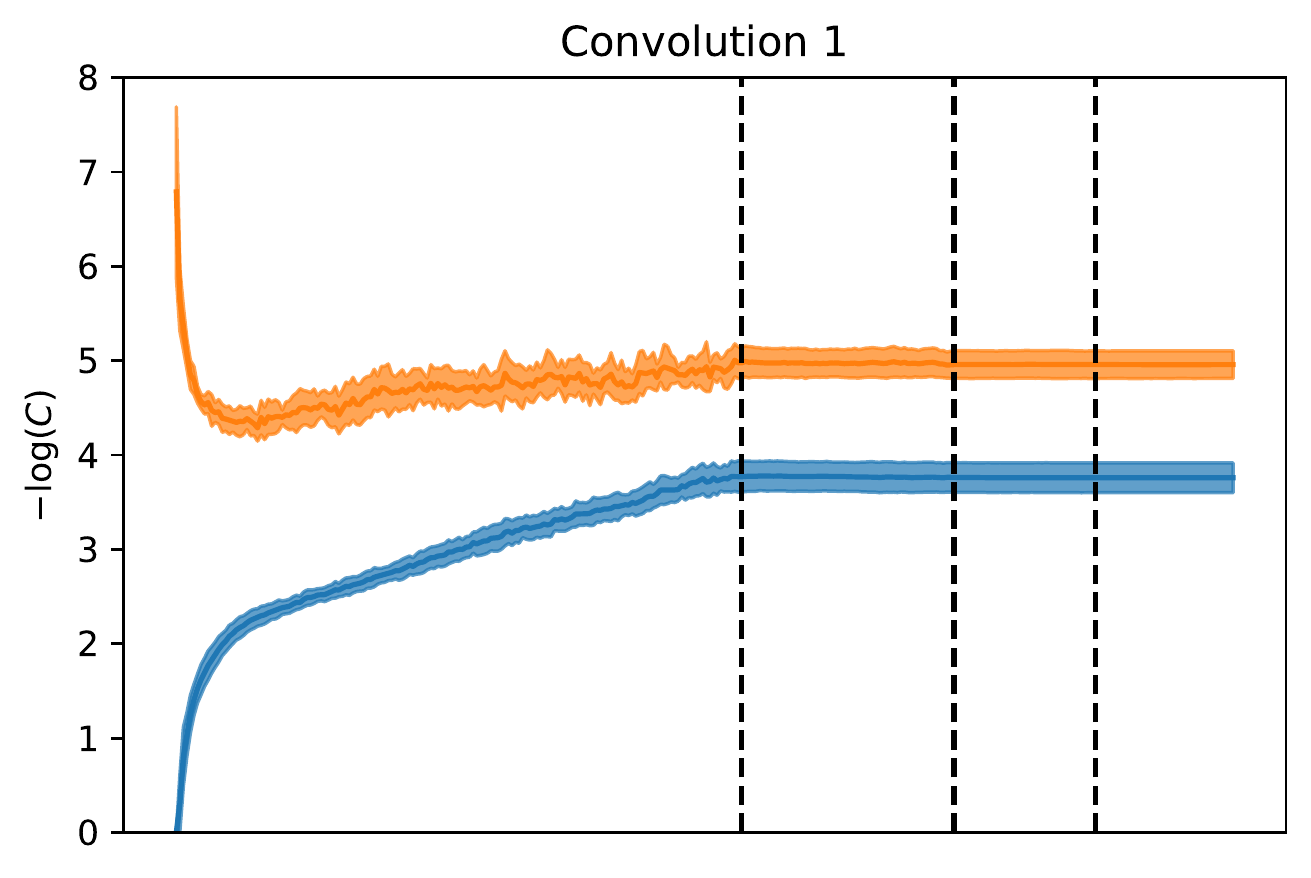}
    \end{subfigure}
    \hspace{\figspacescale}
    \begin{subfigure}{}
        \centering
    	\includegraphics[width=\figscale, height=1.3in]{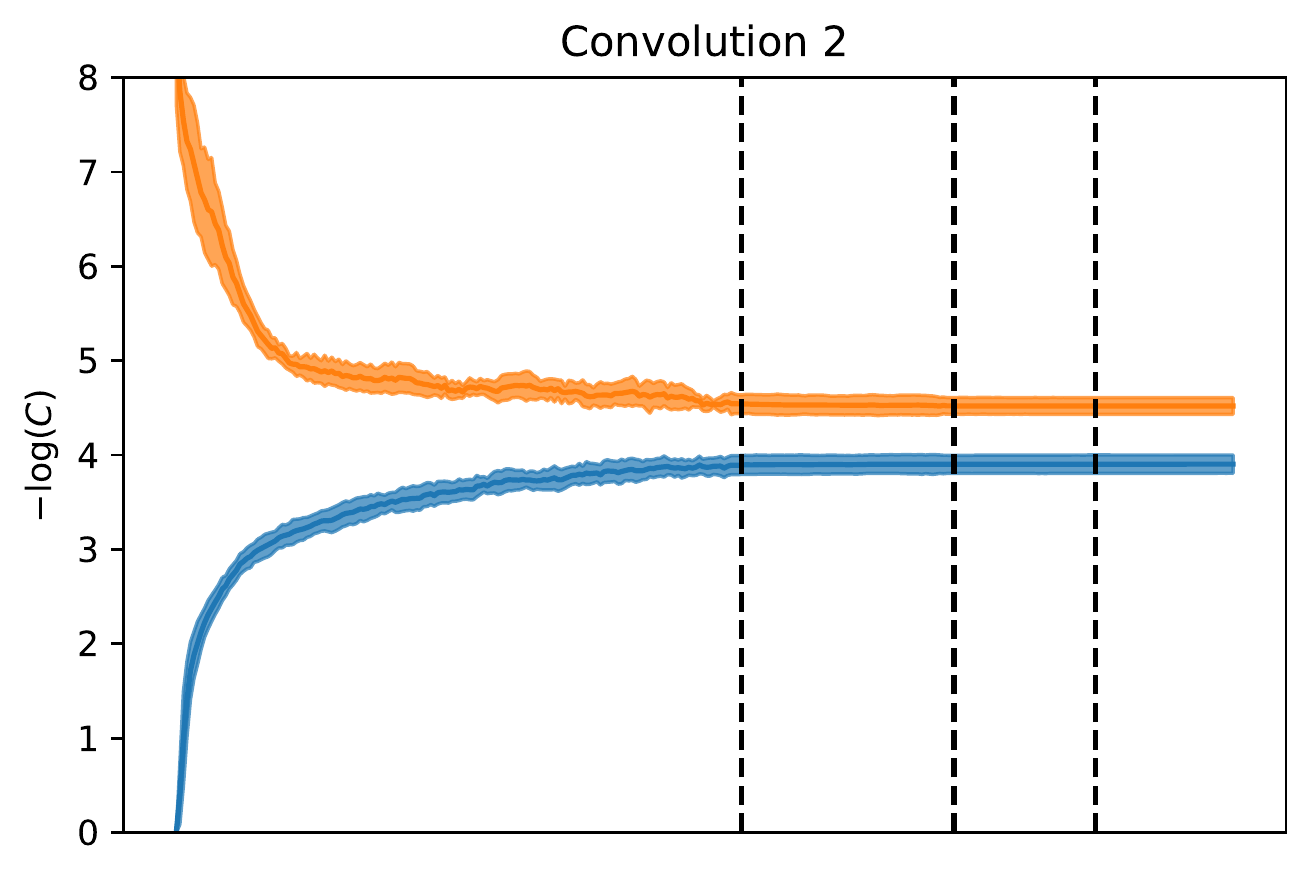}
    \end{subfigure}
    \hspace{\figspacescale}
    \begin{subfigure}{}
        \centering
    	\includegraphics[width=\figscale, height=1.3in]{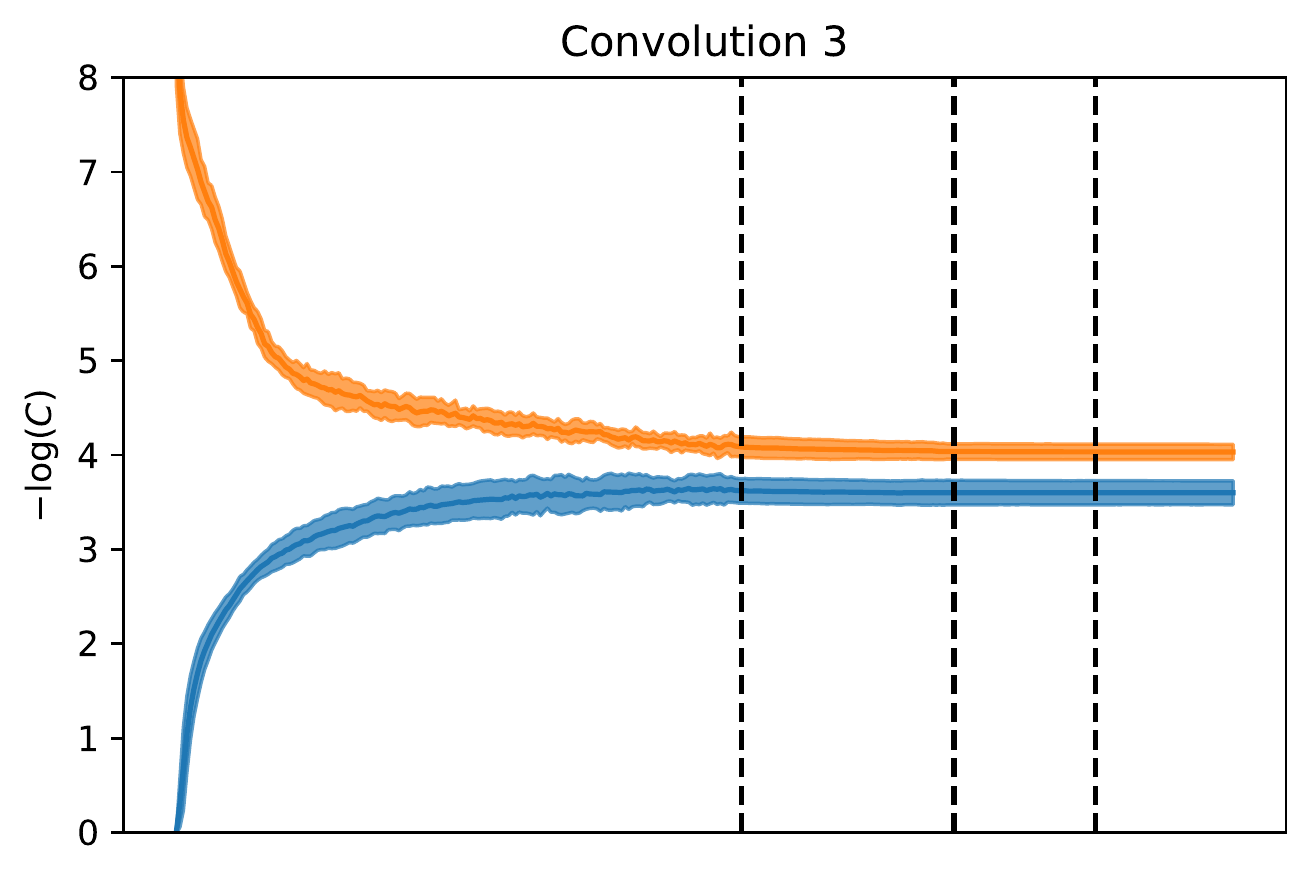}
    \end{subfigure}%
    \\
    \vspace{-0.1in}
    \begin{subfigure}{}
        \centering
    	\includegraphics[width=\figscale, height=1.3in]{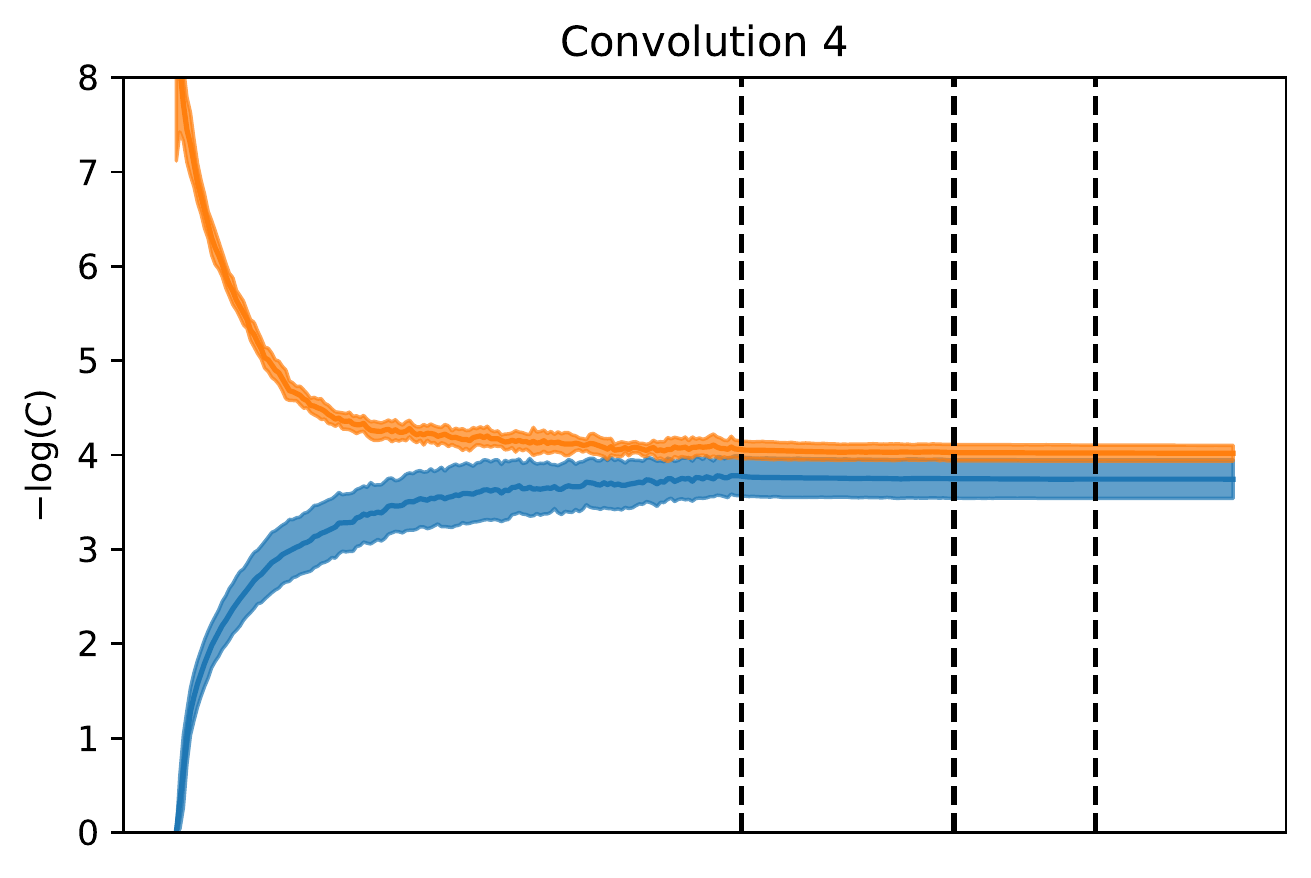}
    \end{subfigure}%
        \hspace{\figspacescale}
    \begin{subfigure}{}
        \centering
    	\includegraphics[width=\figscale, height=1.3in]{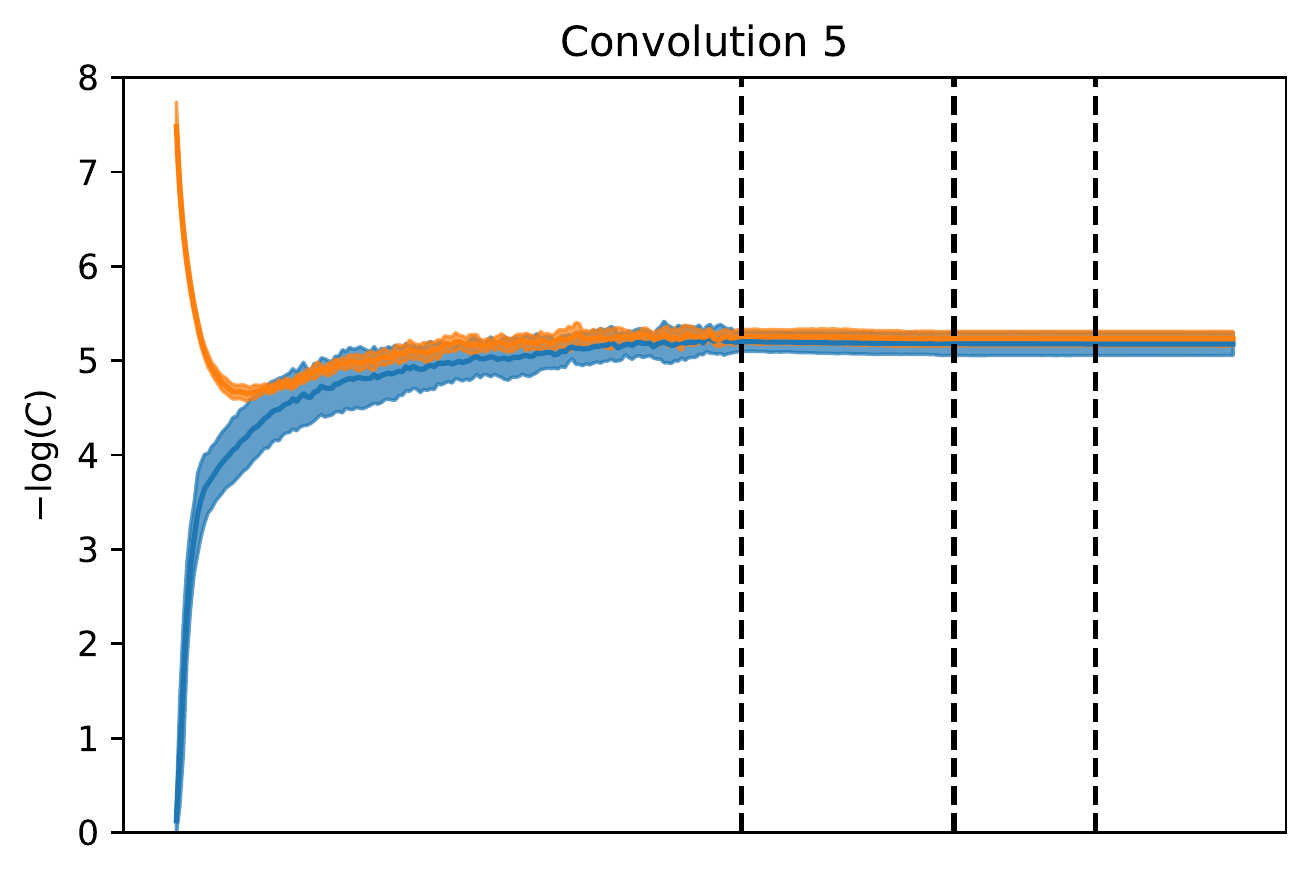}
    \end{subfigure}%
        \hspace{\figspacescale}
    \begin{subfigure}{}
        \centering
    	\includegraphics[width=\figscale, height=1.3in]{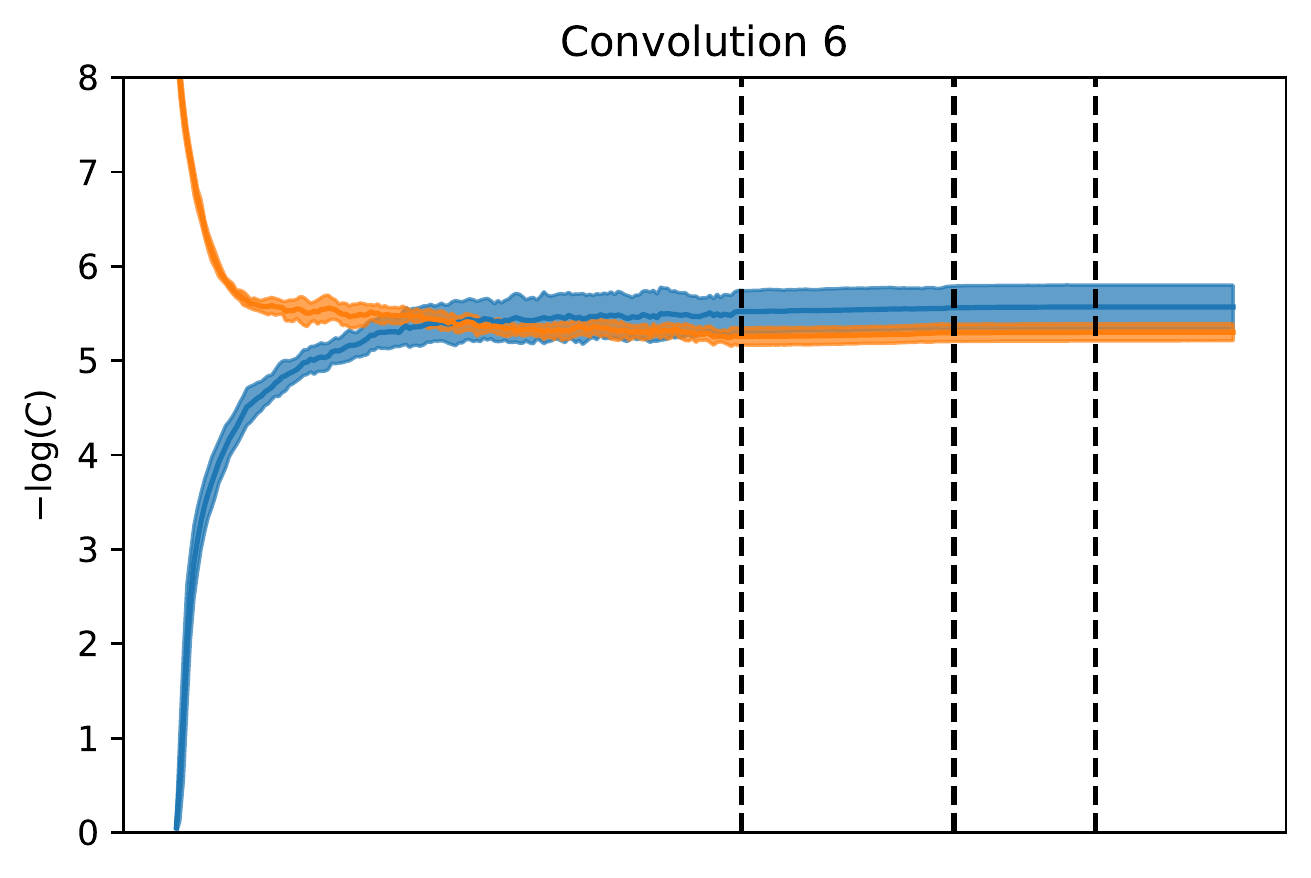}
    \end{subfigure}%
    \vspace{-0.1in}
    \\
    \begin{subfigure}{}
        \centering
    	\includegraphics[width=\figscale, height=1.3in]{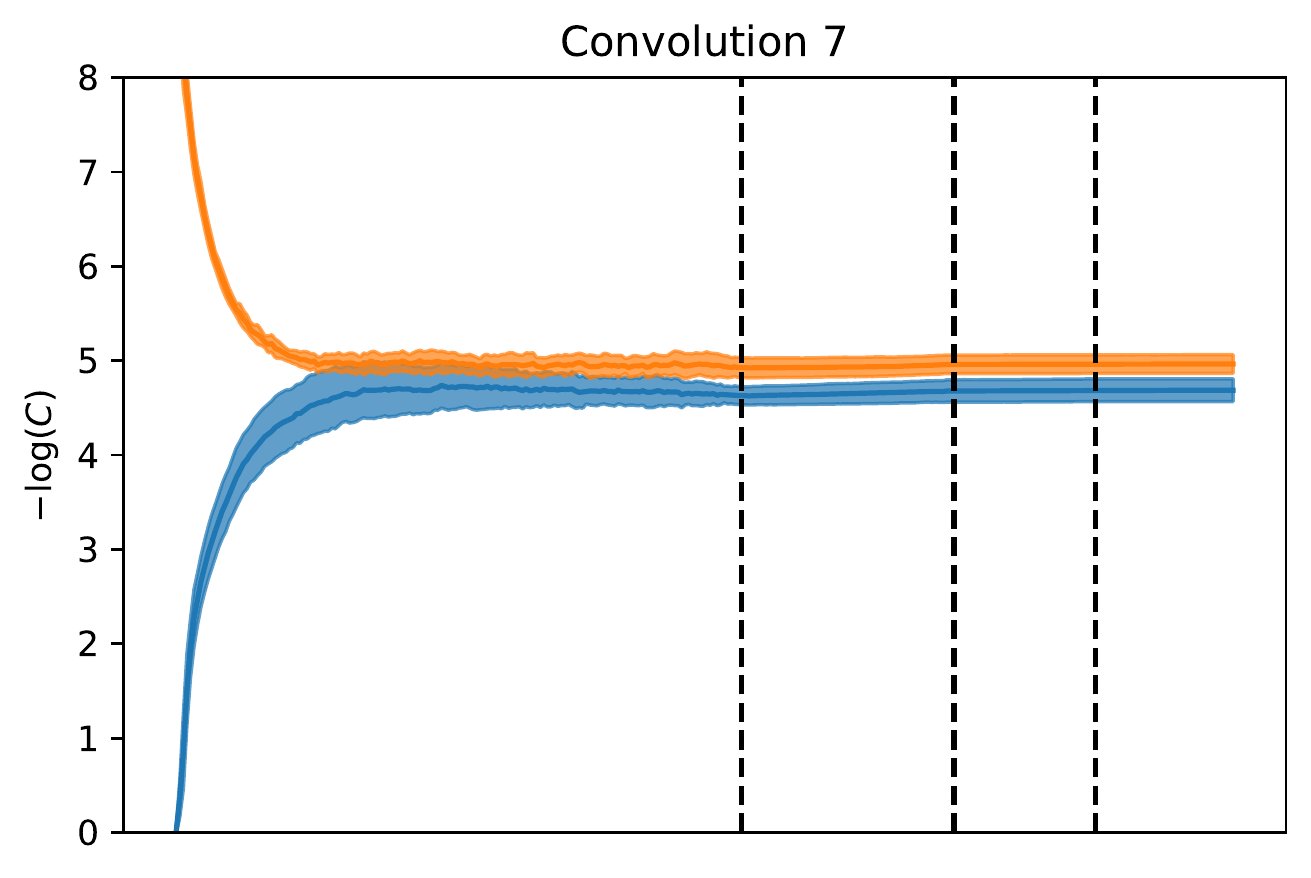}
    \end{subfigure}
    \hspace{\figspacescale}
    \begin{subfigure}{}
        \centering
    	\includegraphics[width=\figscale, height=1.3in]{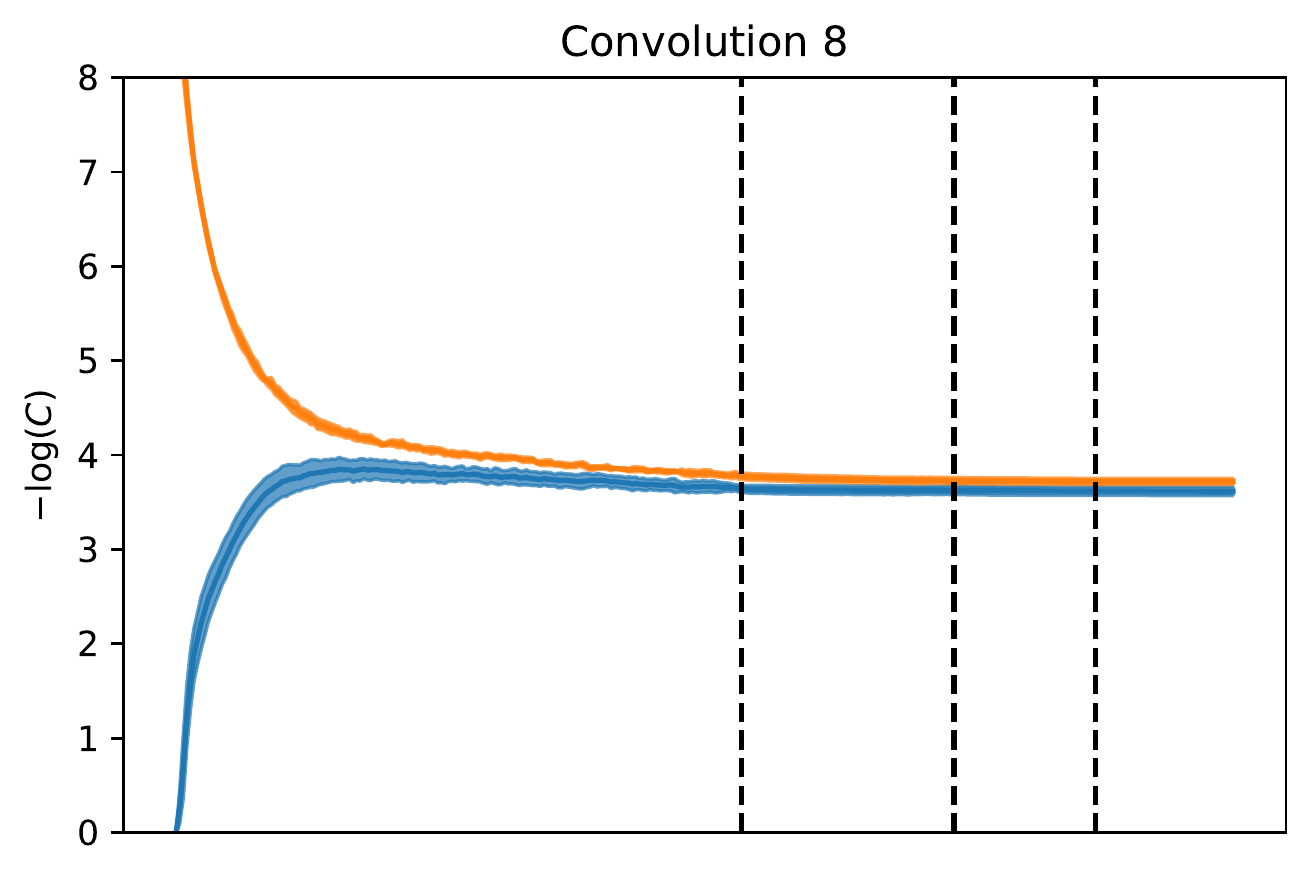}
    \end{subfigure}
    \hspace{\figspacescale}
    \begin{subfigure}{}
        \centering
    	\includegraphics[width=\figscale, height=1.3in]{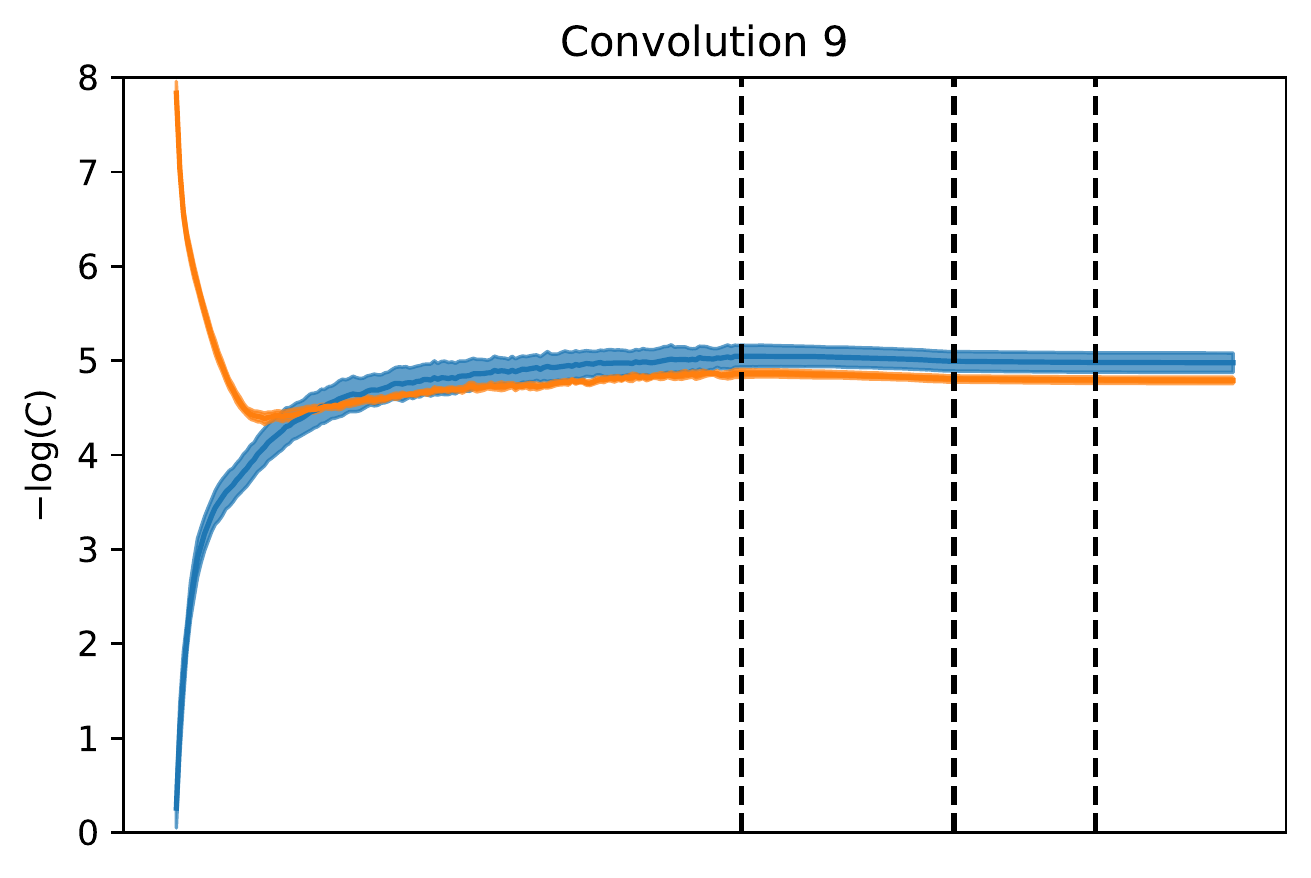}
    \end{subfigure}%
    \vspace{-0.1in}
    \\
    \begin{subfigure}{}
        \centering
    	\includegraphics[width=\figscale, height=1.3in]{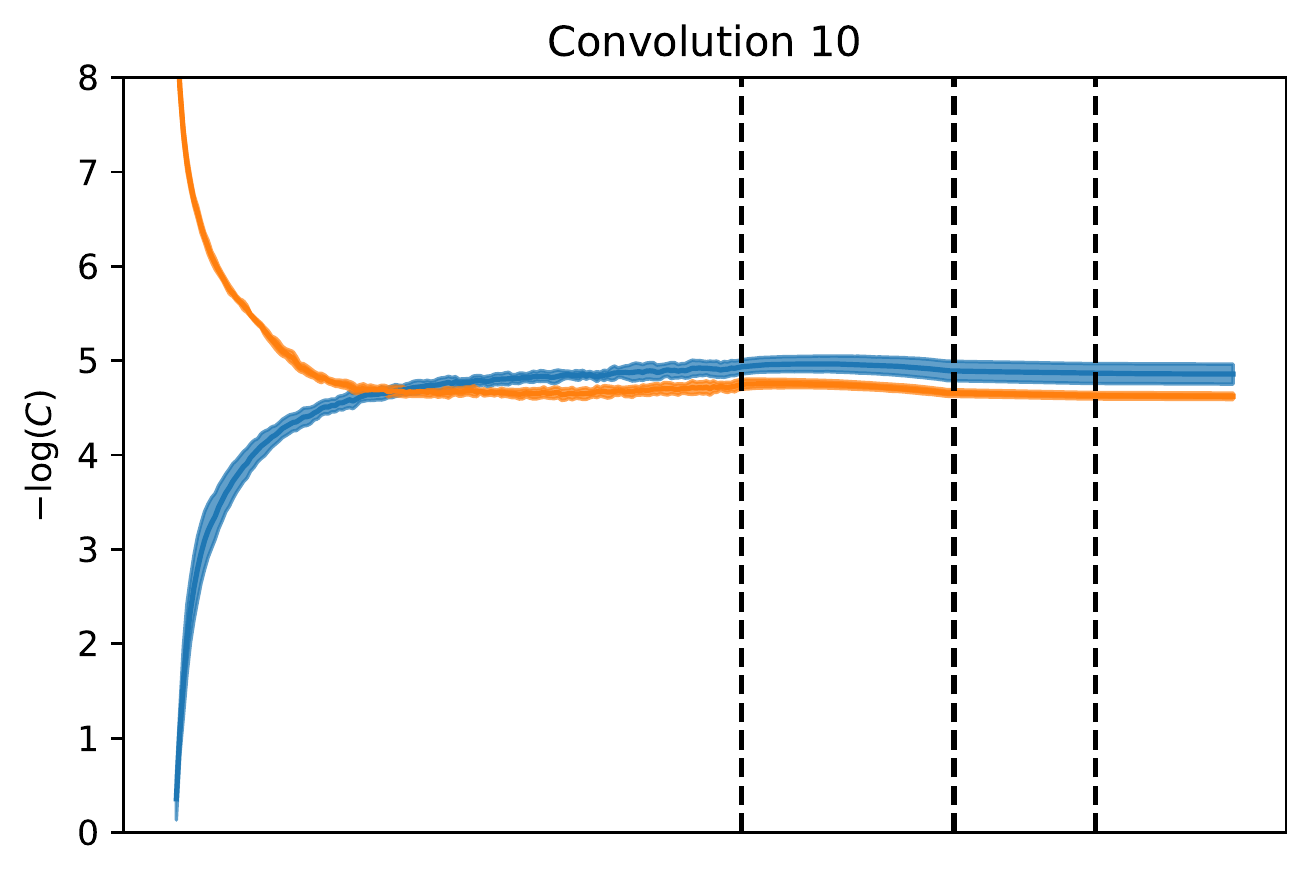}
    \end{subfigure}
    \hspace{\figspacescale}
    \begin{subfigure}{}
        \centering
    	\includegraphics[width=\figscale, height=1.3in]{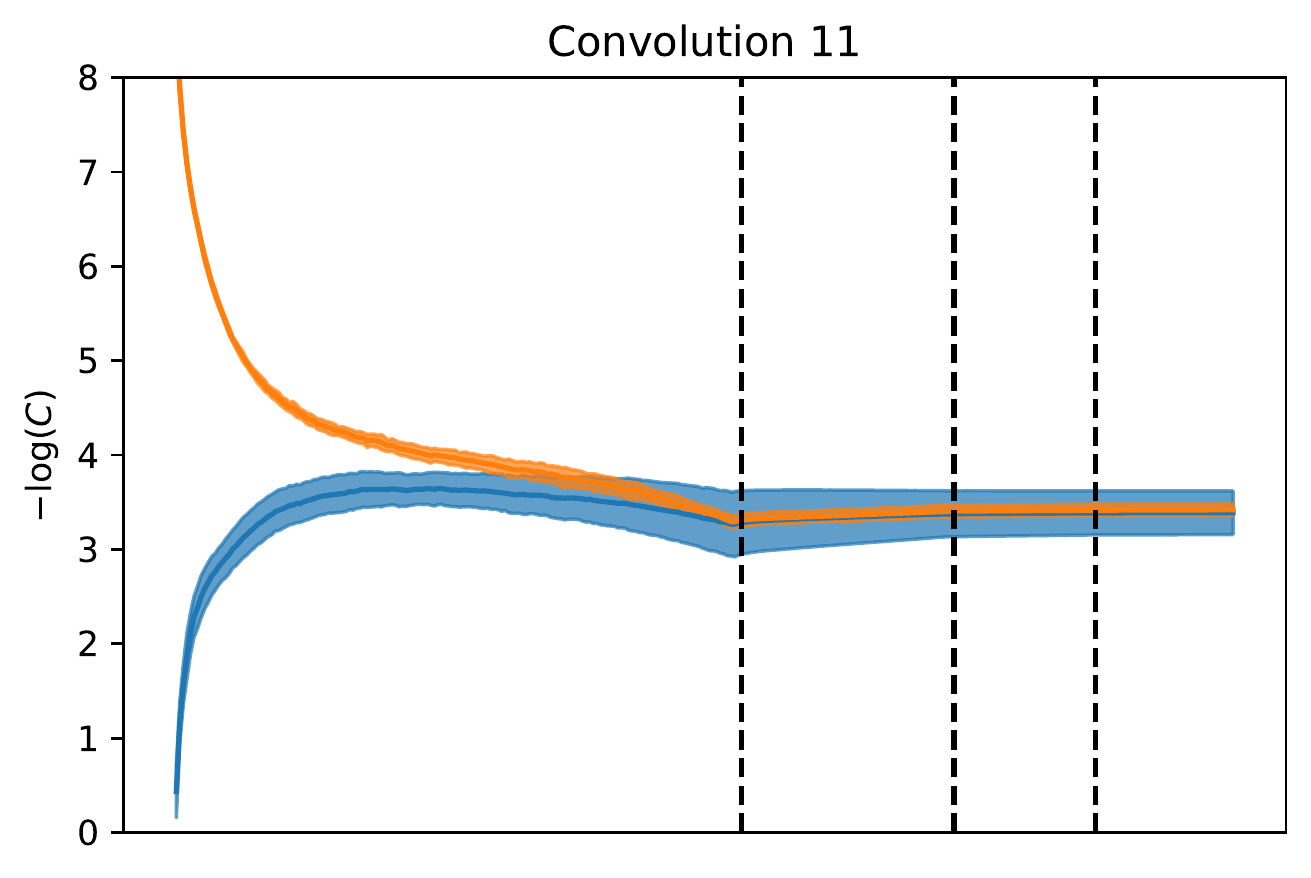}
    \end{subfigure}
    \hspace{\figspacescale}
    \begin{subfigure}{}
        \centering
    	\includegraphics[width=\figscale, height=1.3in]{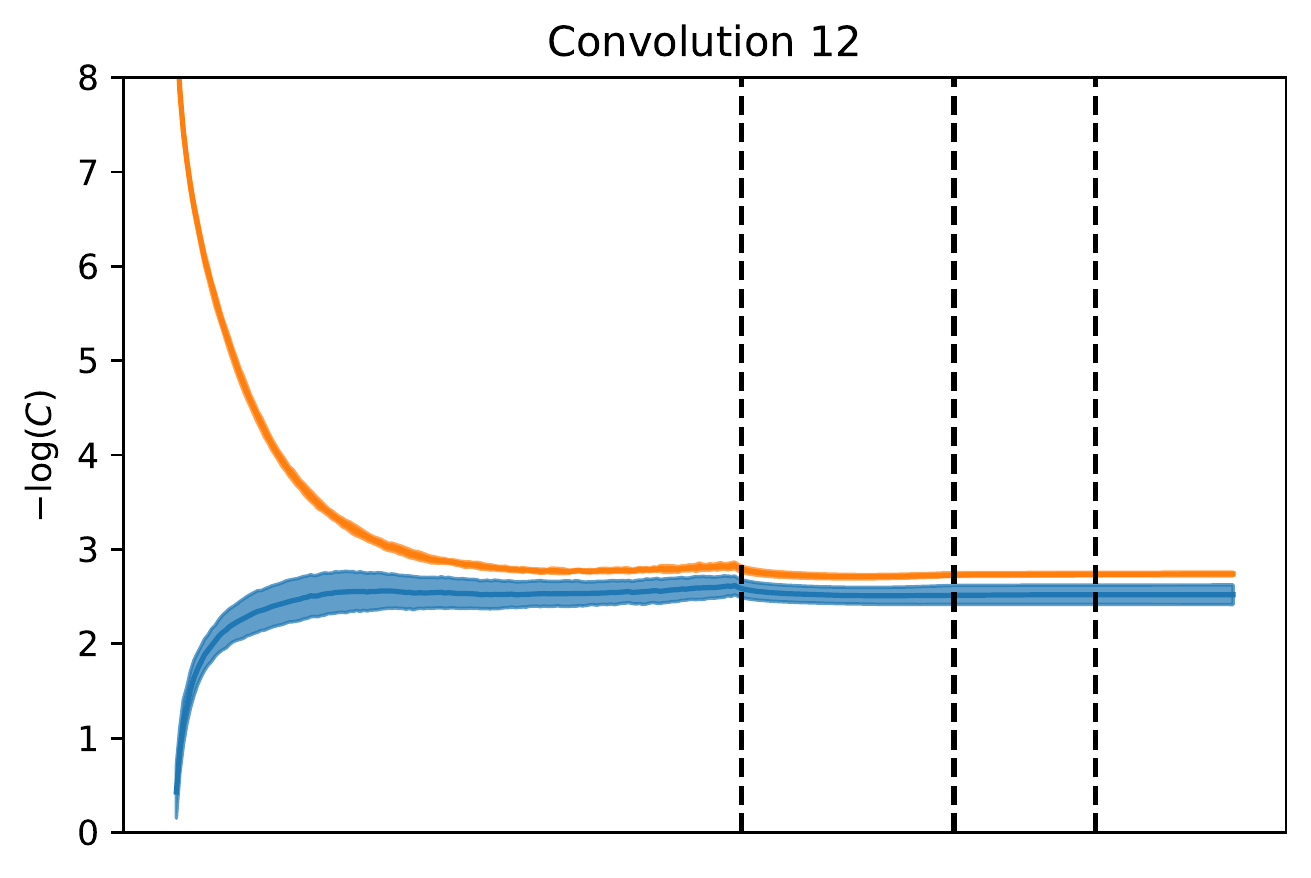}
    \end{subfigure}%
    \vspace{-0.1in}
    \\
    \begin{subfigure}{}
        \centering
    	\includegraphics[width=\figscale, height=1.45in]{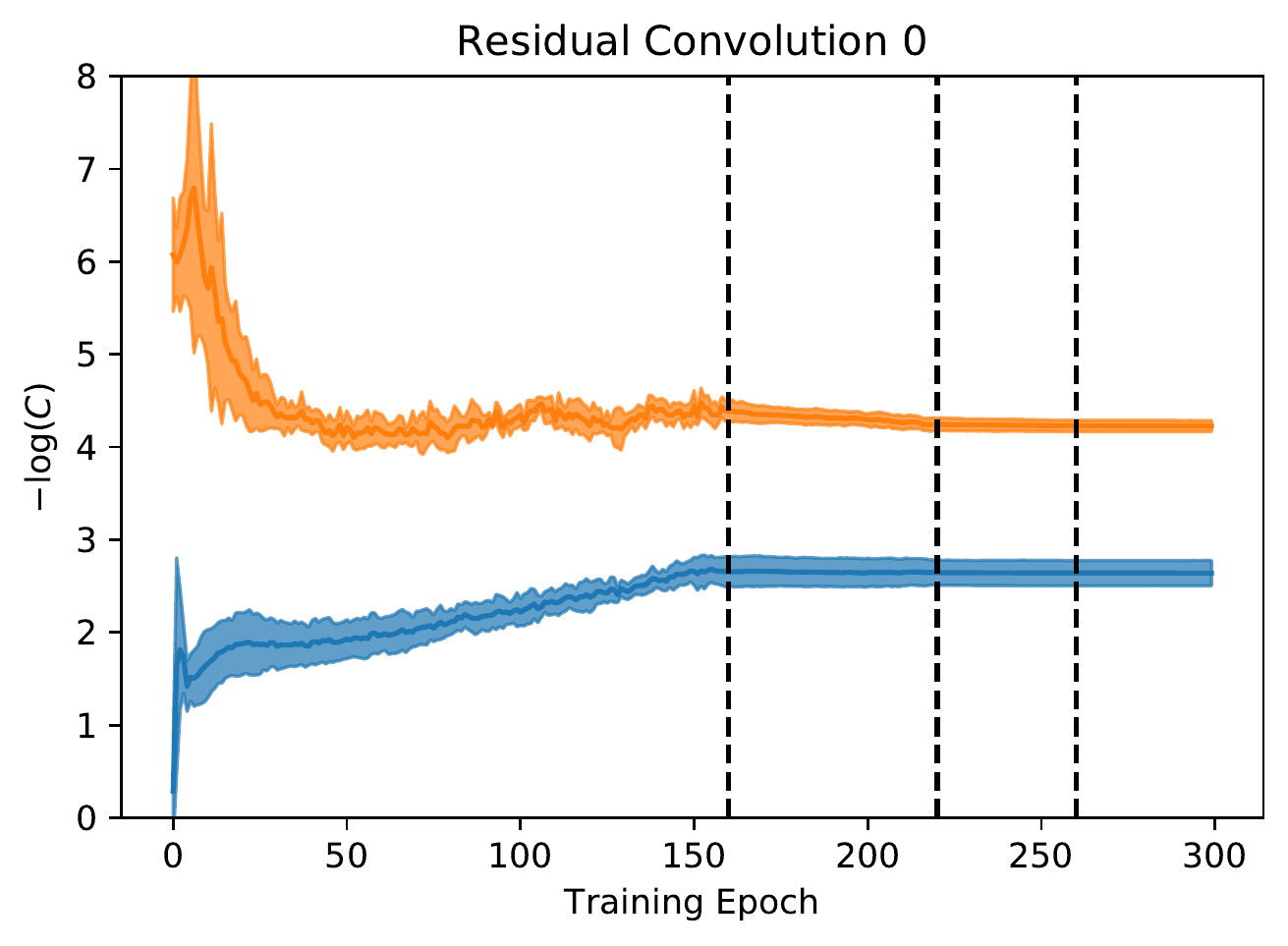}
    \end{subfigure}
    \hspace{\figspacescale}
    \begin{subfigure}{}
        \centering
    	\includegraphics[width=\figscale, height=1.45in]{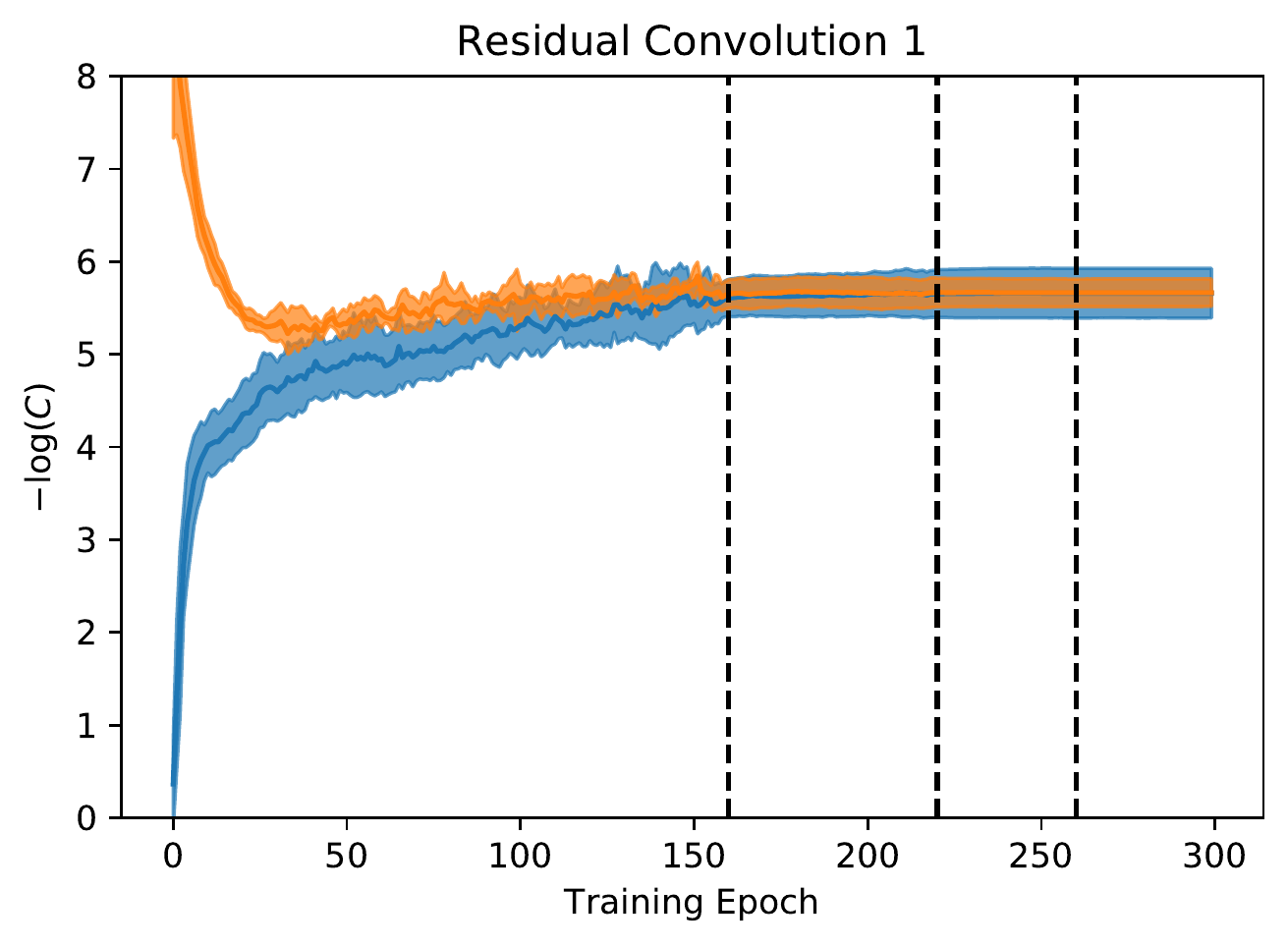}
    \end{subfigure}
    \hspace{\figspacescale}
    \begin{subfigure}{}
        \centering
    	\includegraphics[width=\figscale, height=1.45in]{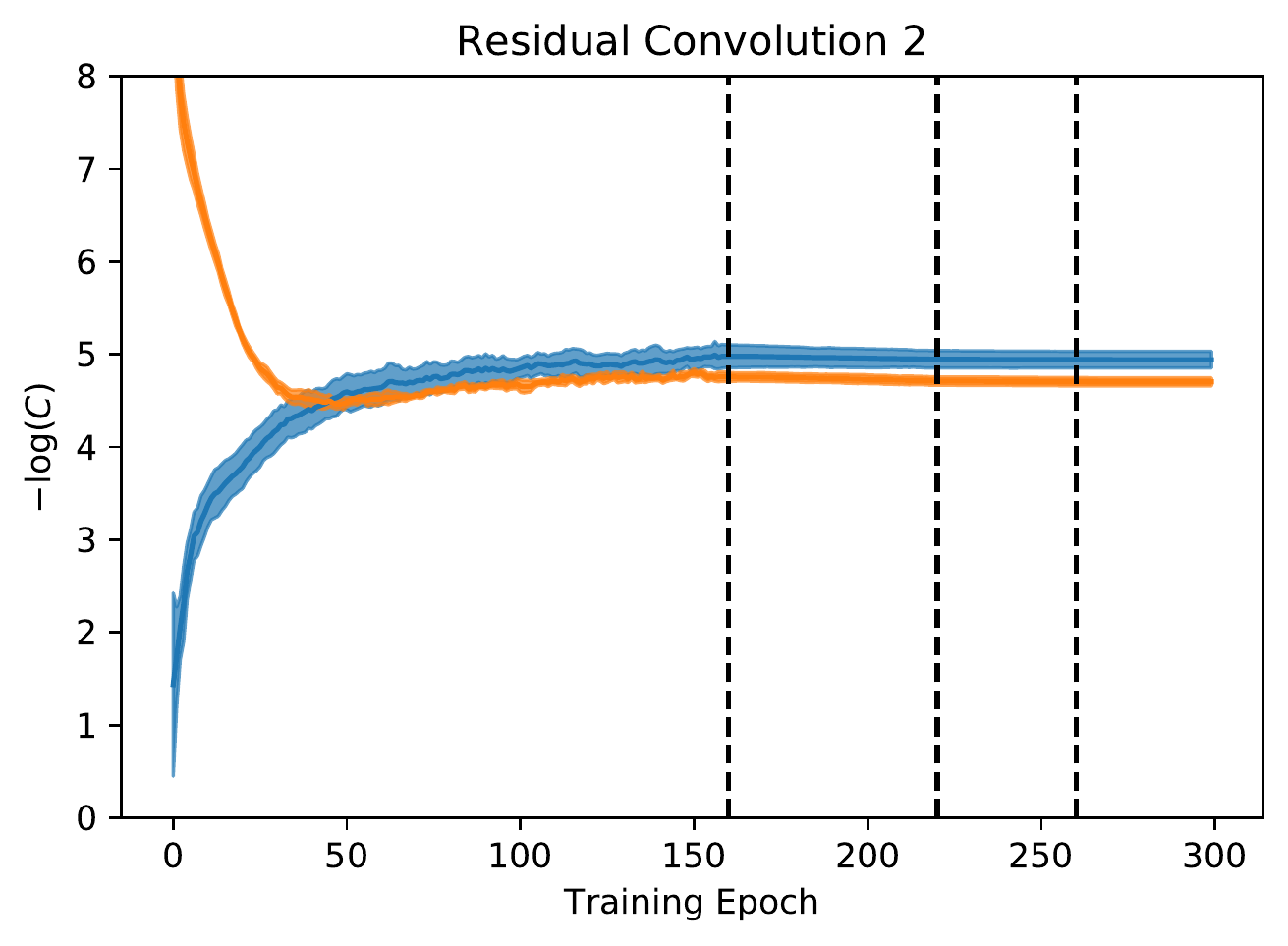}
    \end{subfigure}%
    \vspace{-0.1in}
    \\
    \caption{Natural logarithm of forward correlations for ConstNet during the entire training period. Average and error margin of $6$ seeds per curve. The forward correlations of ConstNet with $0$-init tend to converge to the same values measured for the same layer when ConstNet was initialized with independent, random initialization. Dashed lines mark the epochs in which the learning rate was lowered.}
    \label{fig:full_run_constnet}
\end{figure*}

%% file: sections/LeakyNetsRunsFinal.tex
\renewcommand{\figscale}{3.0in}
\renewcommand{\figspacescale}{-0.0in}

\begin{figure*}[ht!]
    
    \centering
    \begin{subfigure}{}
        \centering
    	\includegraphics[width=\figscale, height=2.0in]{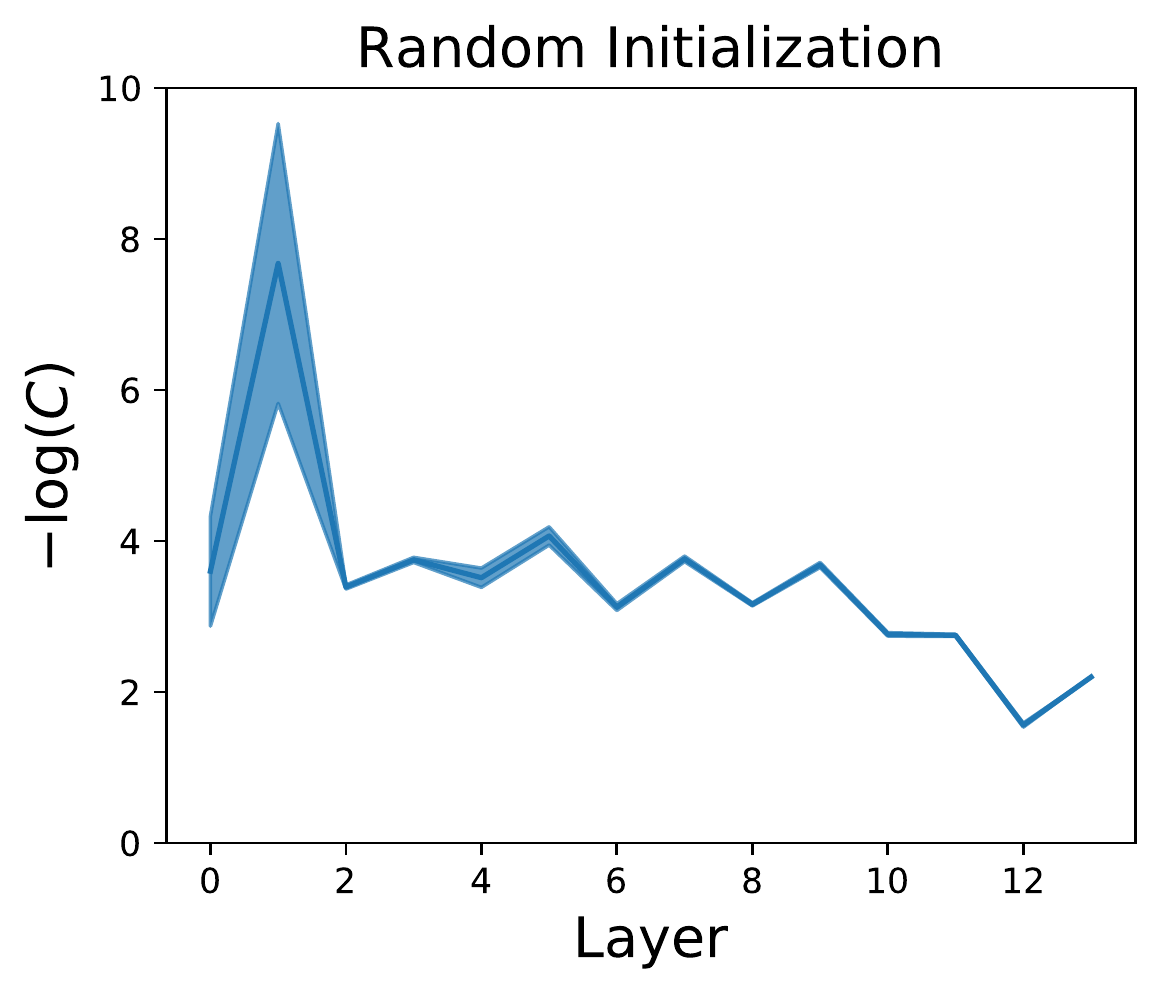}
    \end{subfigure}
    \\
    \centering
    \begin{subfigure}{}
        \centering
    	\includegraphics[width=\figscale, height=2.0in]{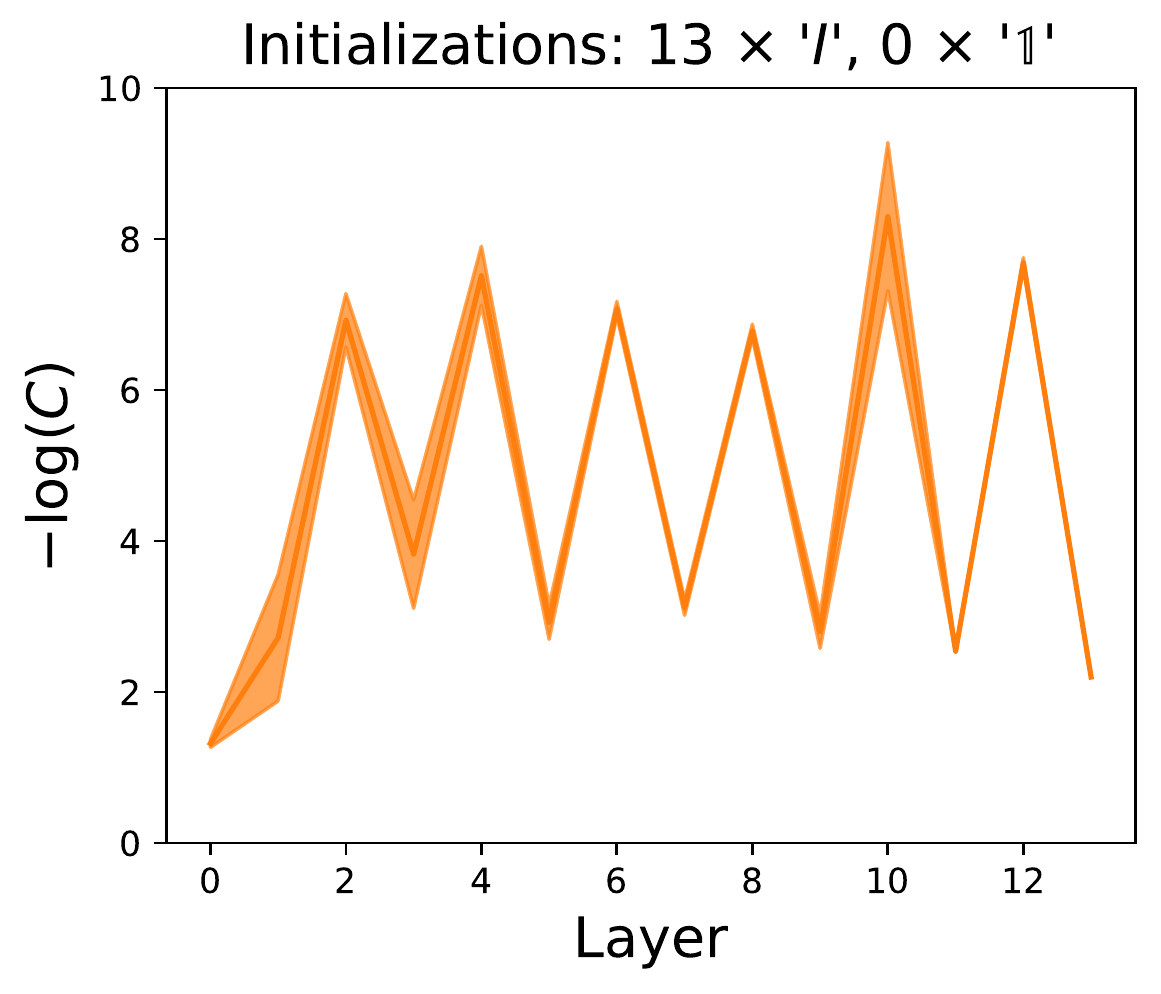}
    \end{subfigure}
    \begin{subfigure}{}
        \centering
    	\includegraphics[width=\figscale, height=2.0in]{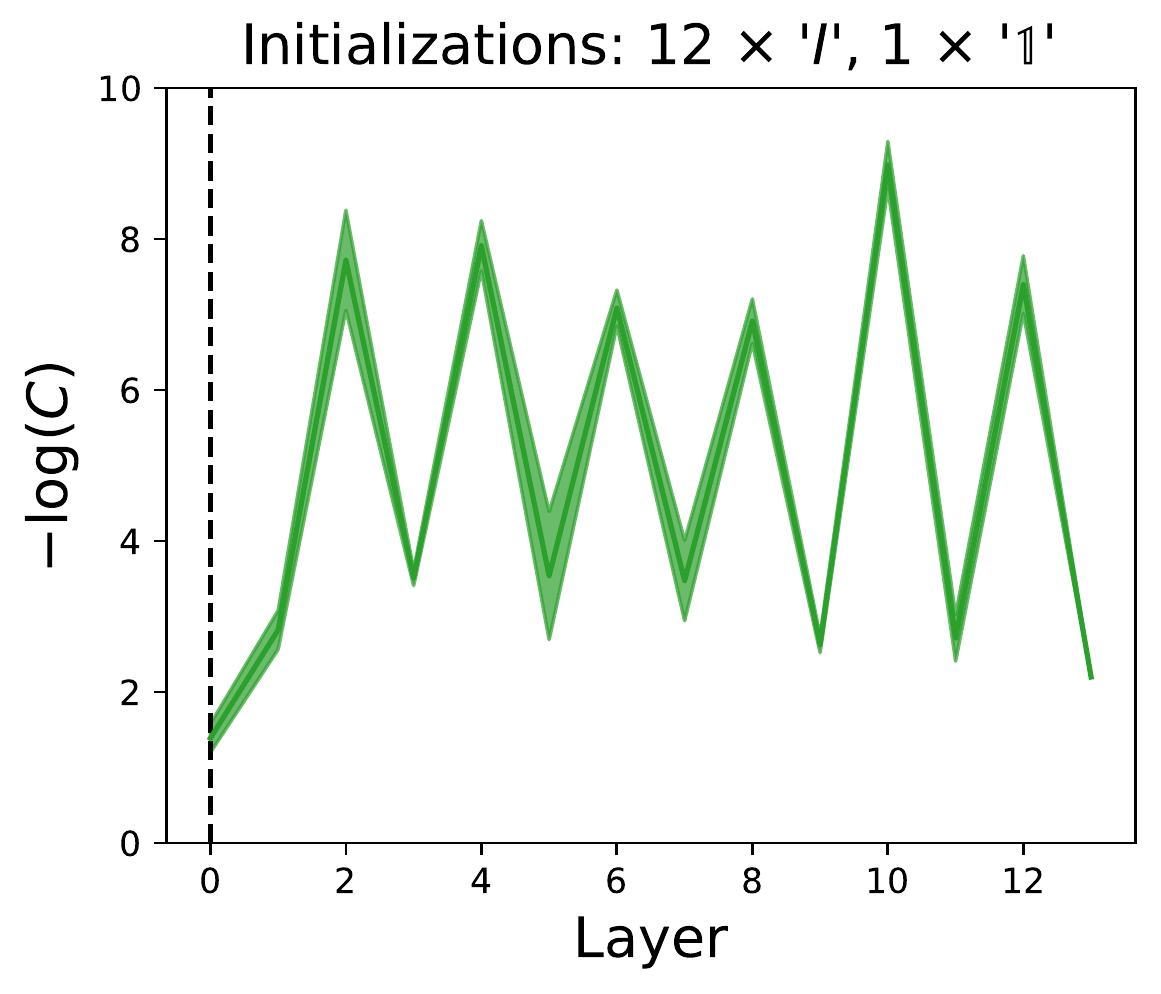}
    \end{subfigure}
    \\
    \centering
    \begin{subfigure}{}
        \centering
    	\includegraphics[width=\figscale, height=2.0in]{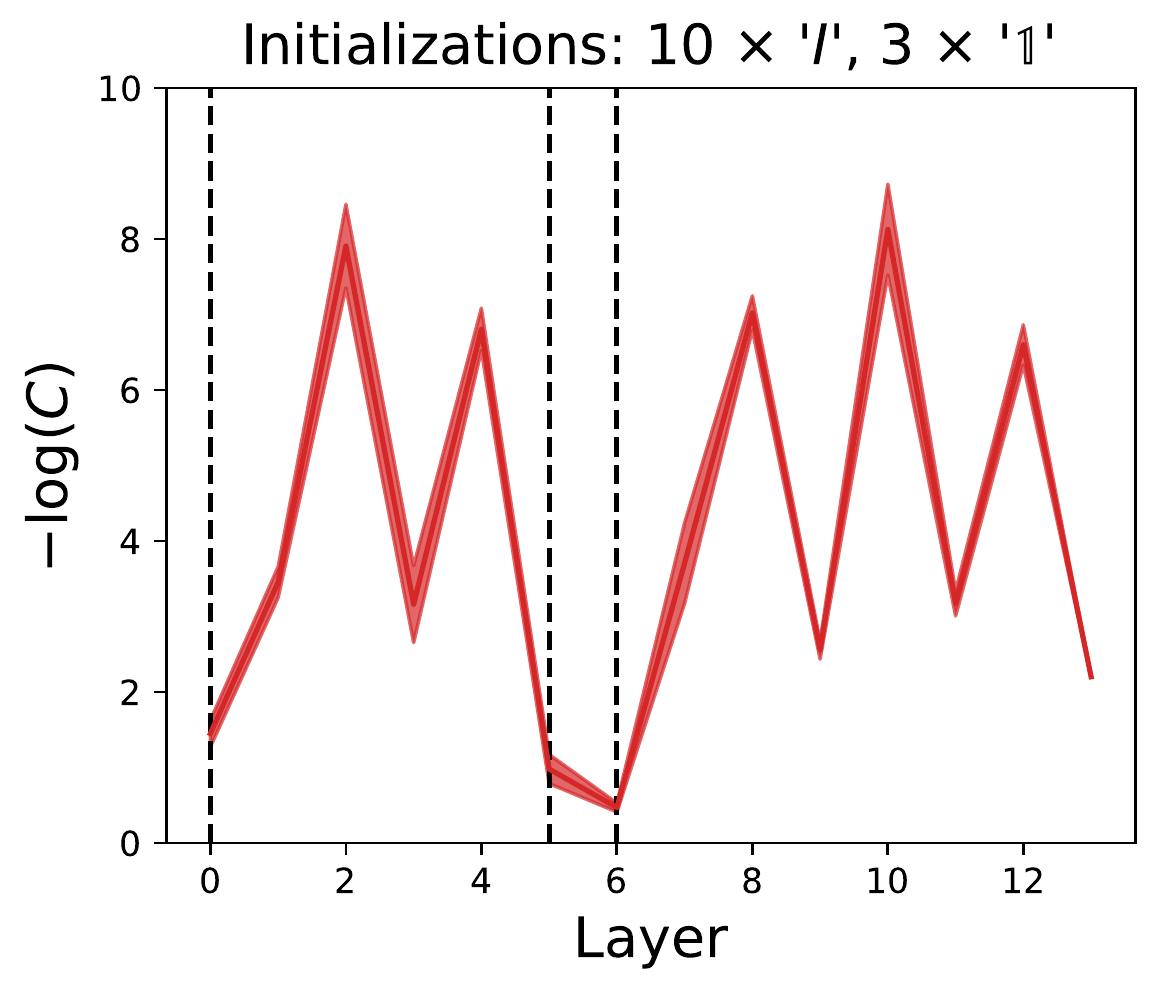}
    \end{subfigure}
    \begin{subfigure}{}
        \centering
    	\includegraphics[width=\figscale, height=2.0in]{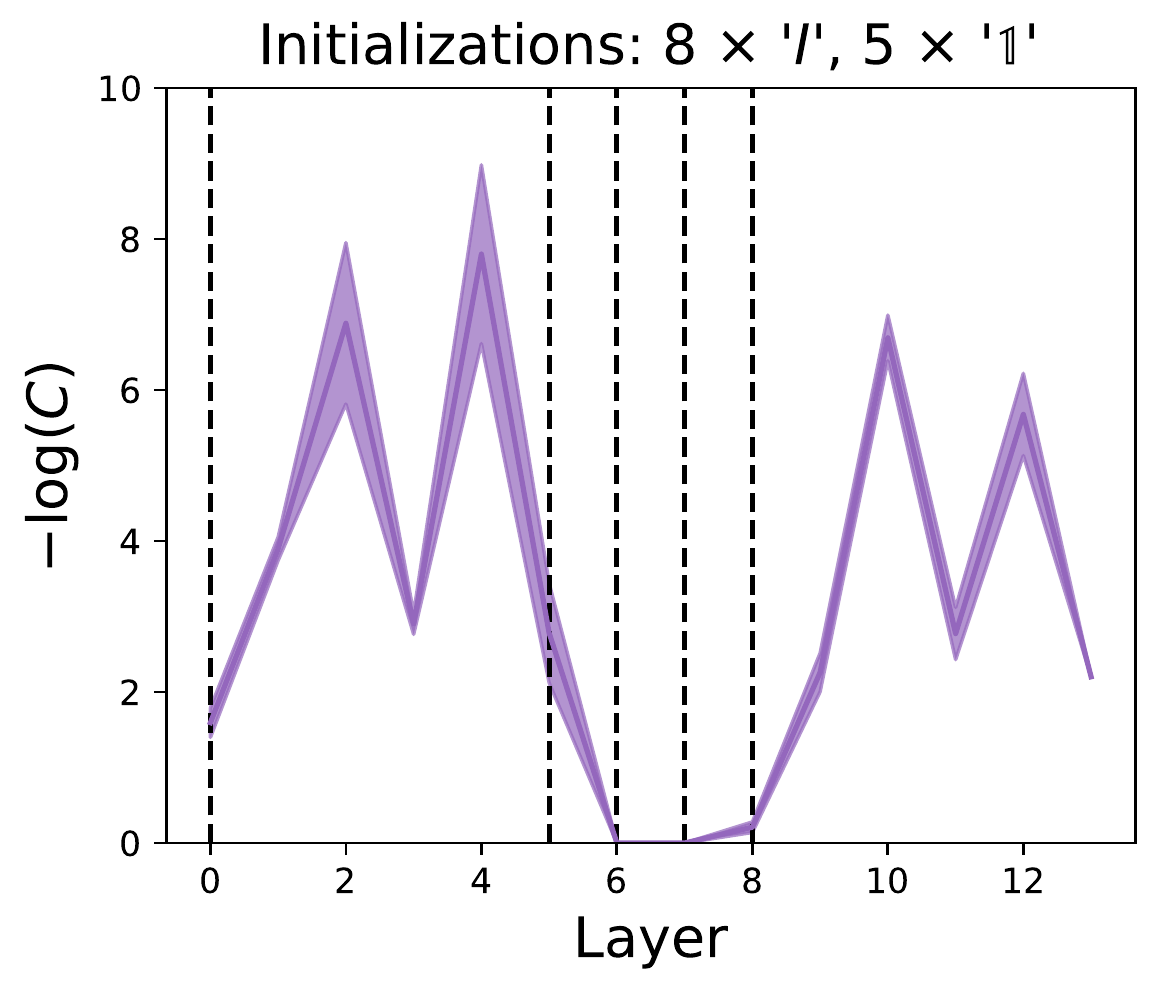}
    \end{subfigure}
    \\
    \centering
    \begin{subfigure}{}
        \centering
    	\includegraphics[width=\figscale, height=2.0in]{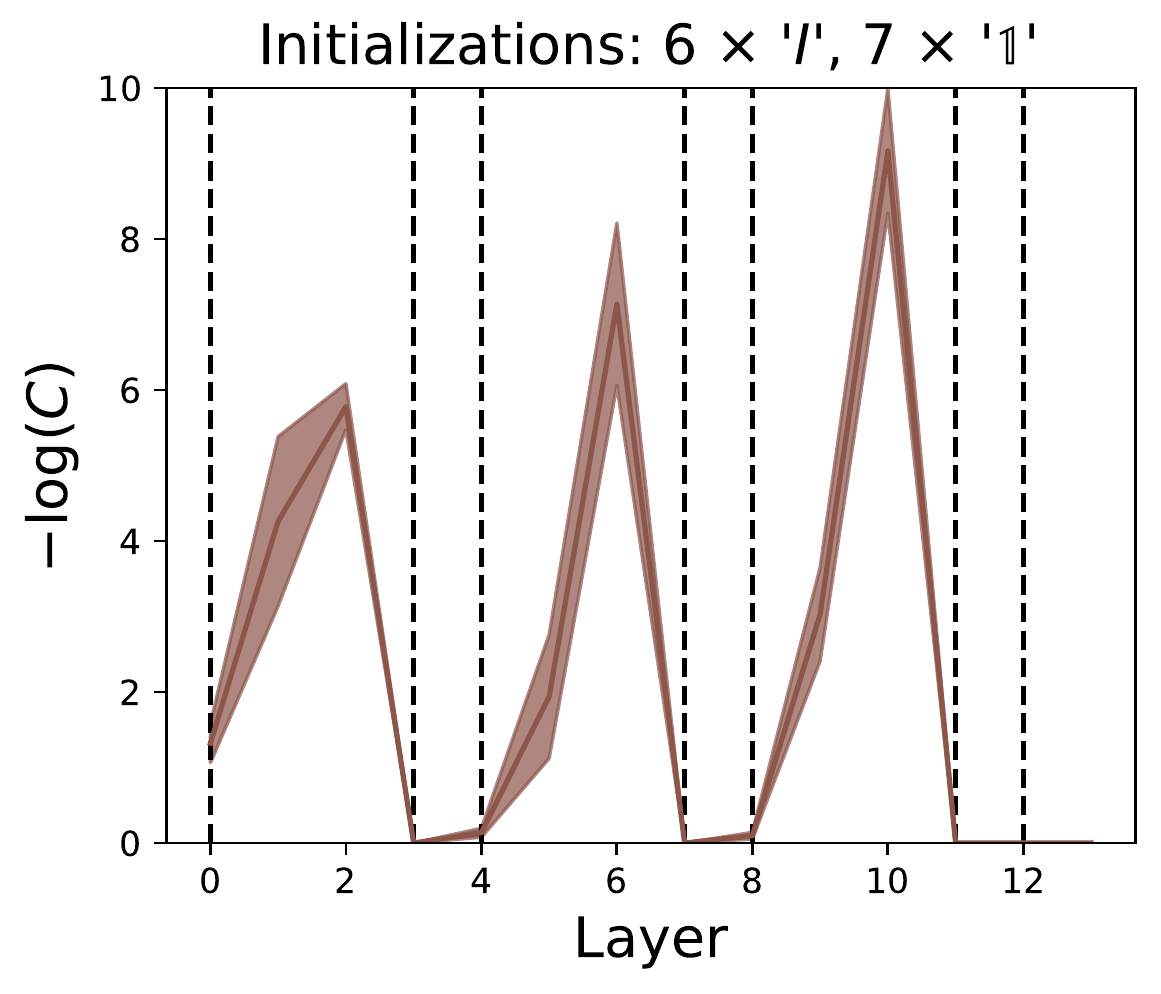}
    \end{subfigure}
    \begin{subfigure}{}
        \centering
    	\includegraphics[width=\figscale, height=2.0in]{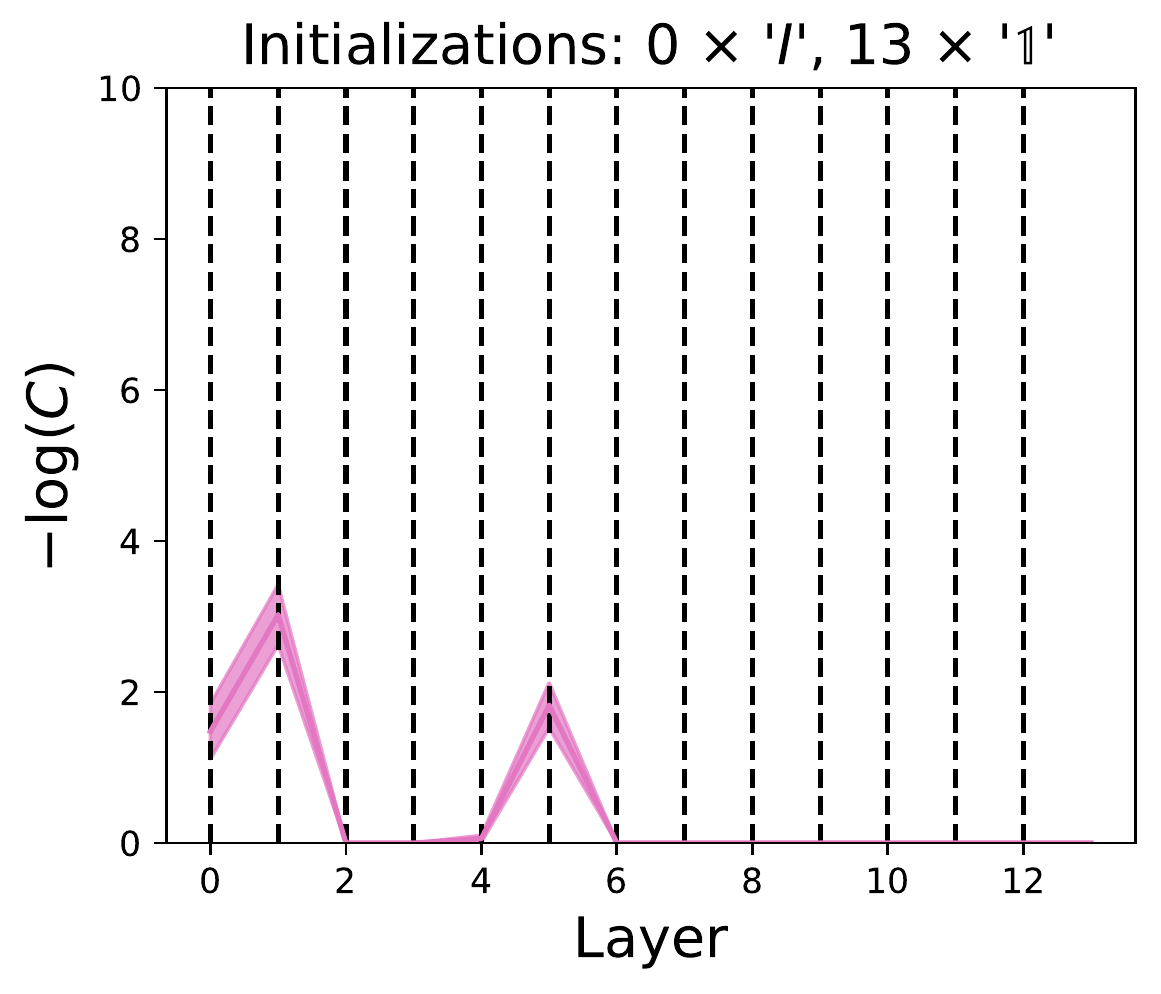}
    \end{subfigure}
    \\
    \caption{Natural logarithm of forward correlations in differently initialized leakyNets, at the end of training. Average and error margin of $4$ seeds per curve. Unlike what we saw in ConstNet, the correlations in leakyNets initialized with mixed '$I/\mathbb{1}$' do not converge to the same values seen in the case of i.i.d. random He initialization (which eventually achieved higher test accuracy). Their inability to efficiently break symmetry can be explained by tracking the propagation of the perturbations from the initial features, as discussed in appendix \ref{app:leaky_pert} and can be seen in figures \ref{fig:pertubation_prop},\ref{fig:pertubation_grad_prop}. Dashed lines mark the layers initialized with '$\mathbb{1}$'-init. Layer number 13 is the fully-connected layer, initialized to $0$ on all instances.} 
    \label{fig:leakyNet_Correlations}
\end{figure*}

%% file: sections/LeakyNetPretubationPropagation.tex
\renewcommand{\figscale}{3.0in}
\renewcommand{\figspacescale}{-0.0in}

\begin{figure*}[ht]
    \centering
    \textbf{Perturbation norm / Total norm}\par\medskip
    \begin{subfigure}{}
        \centering
    	\includegraphics[width=\figscale, height=2.0in]{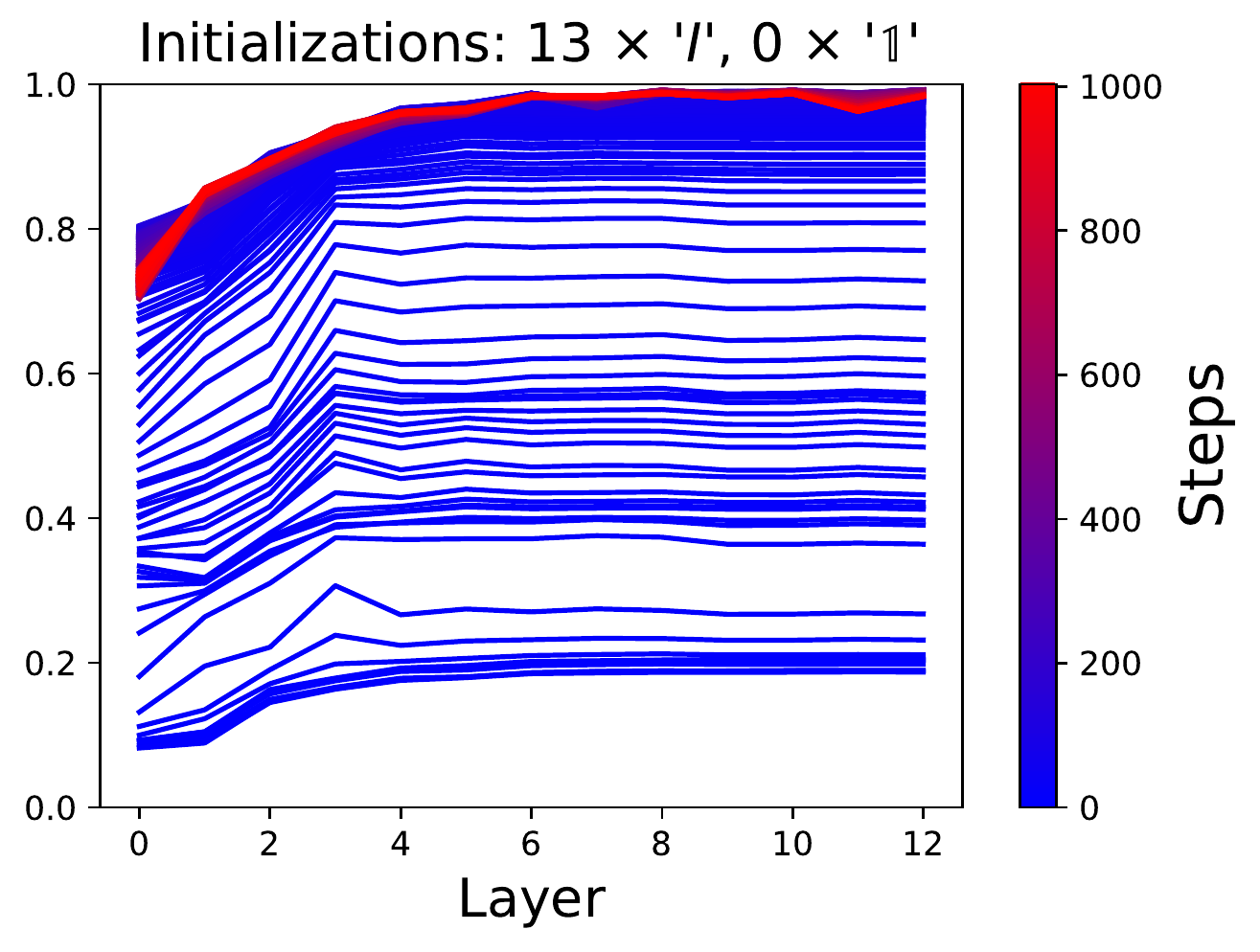}
    \end{subfigure}
    \begin{subfigure}{}
        \centering
    	\includegraphics[width=\figscale, height=2.0in]{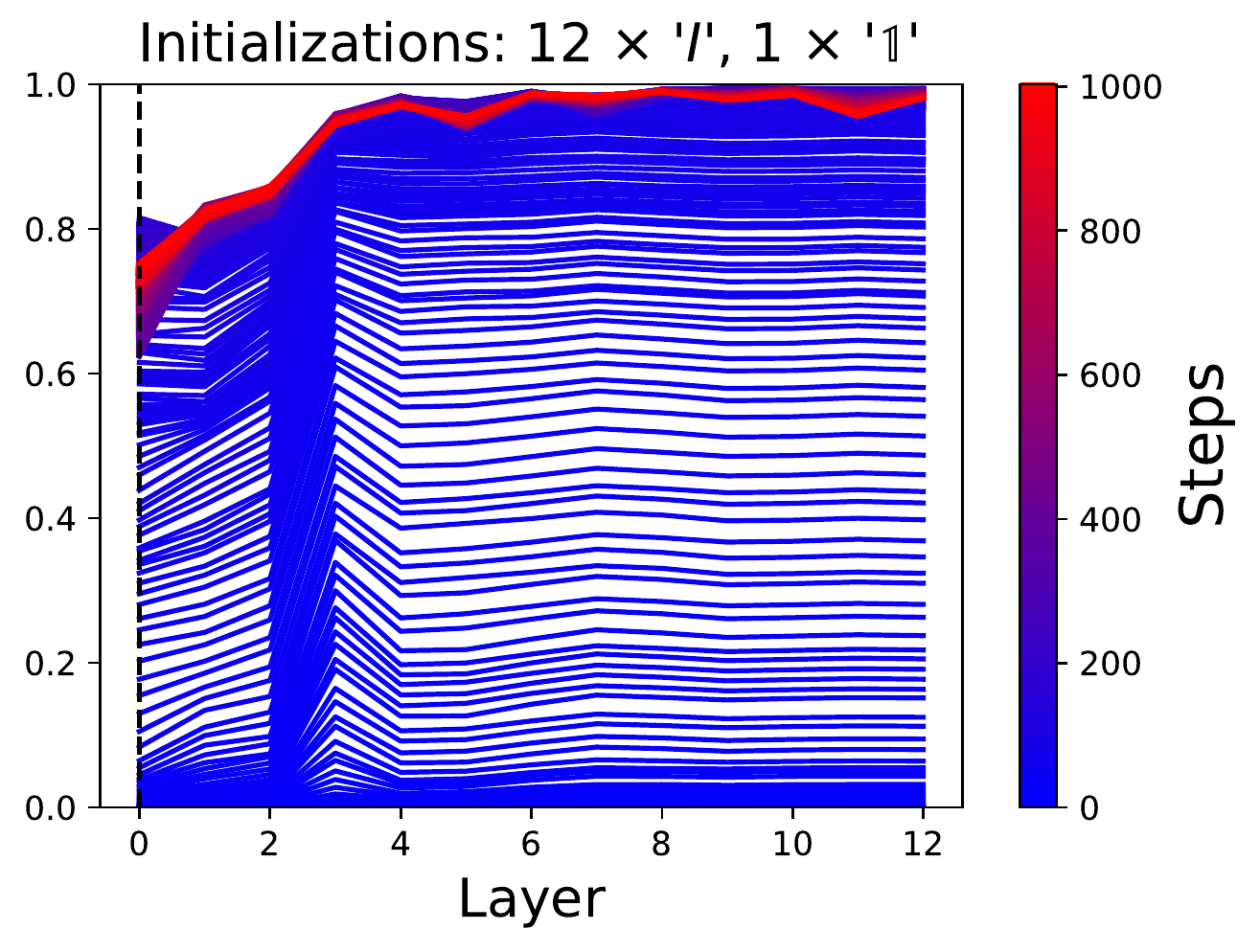}
    \end{subfigure}
    \\
    \centering
    \begin{subfigure}{}
        \centering
    	\includegraphics[width=\figscale, height=2.0in]{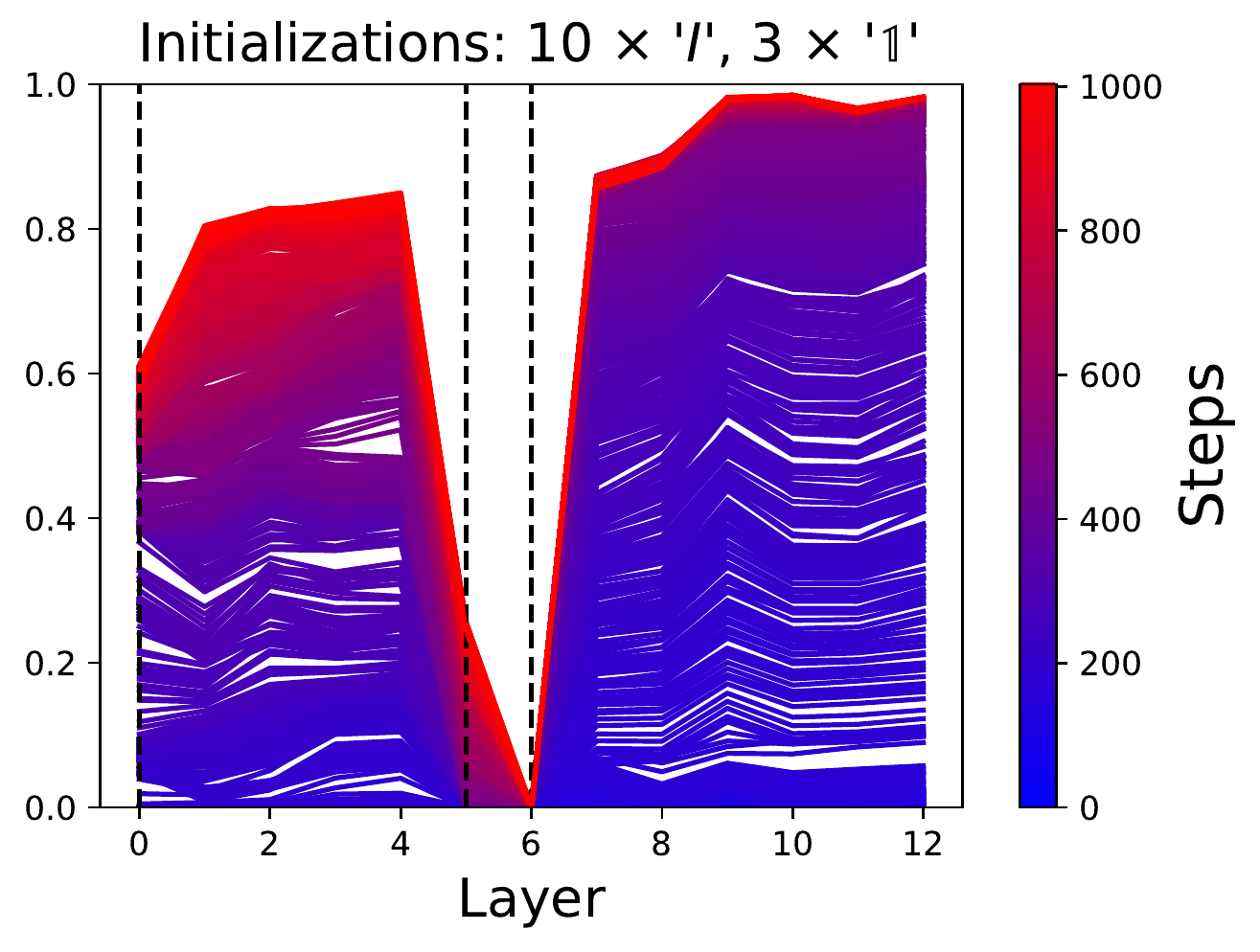}
    \end{subfigure}
    \begin{subfigure}{}
        \centering
    	\includegraphics[width=\figscale, height=2.0in]{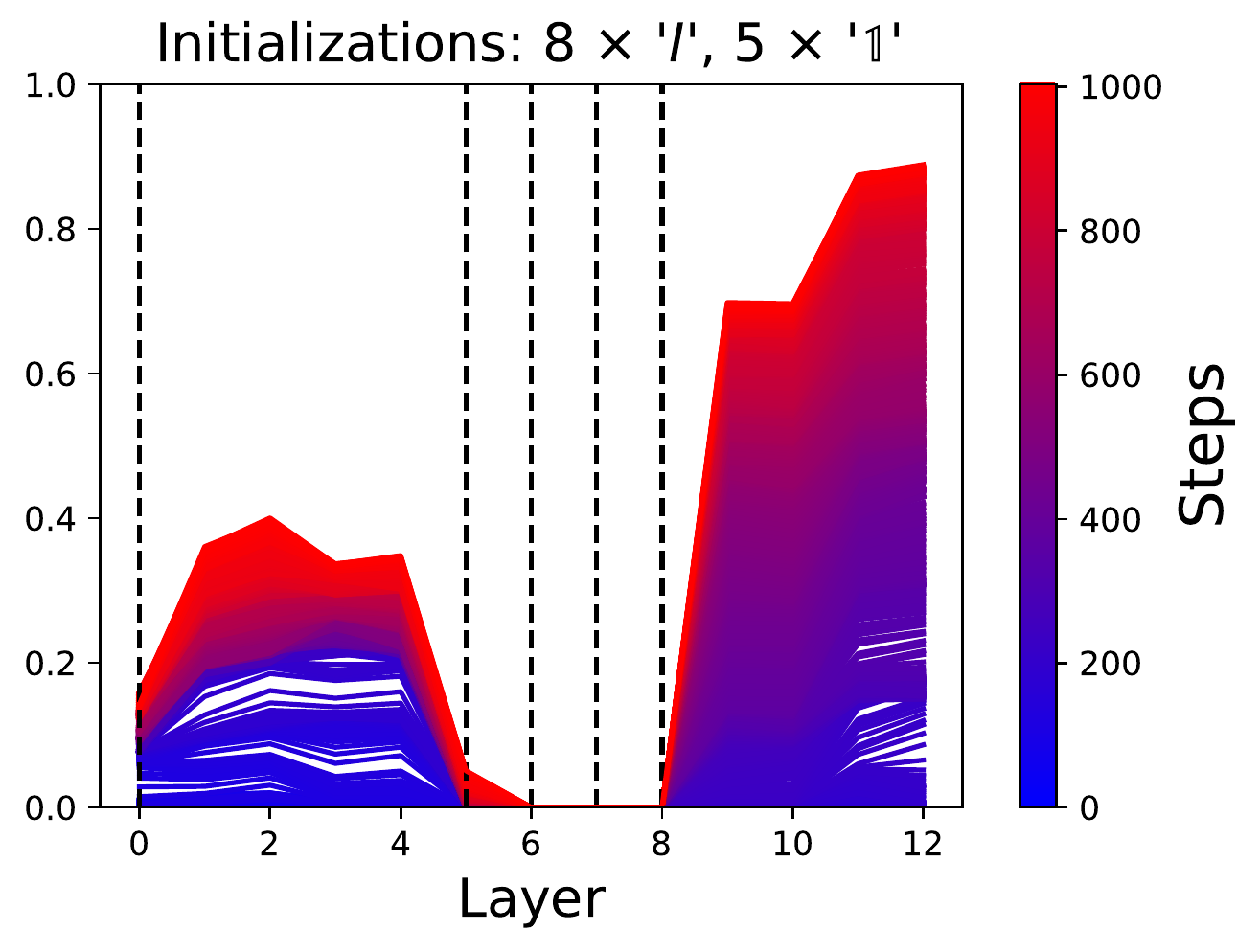}
    \end{subfigure}
    \\
    \centering
    \begin{subfigure}{}
        \centering
    	\includegraphics[width=\figscale, height=2.0in]{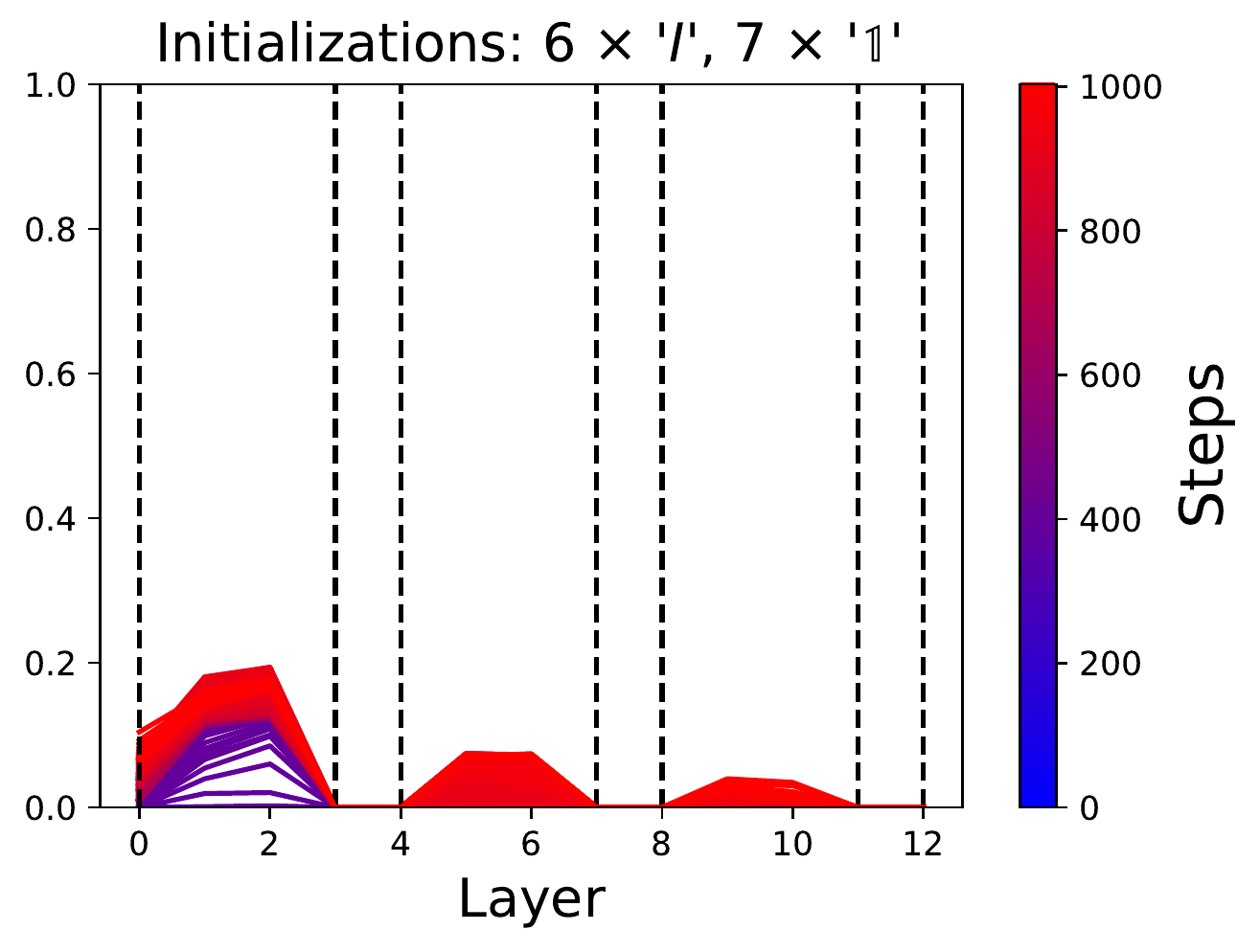}
    \end{subfigure}
    \begin{subfigure}{}
        \centering
    	\includegraphics[width=\figscale, height=2.0in]{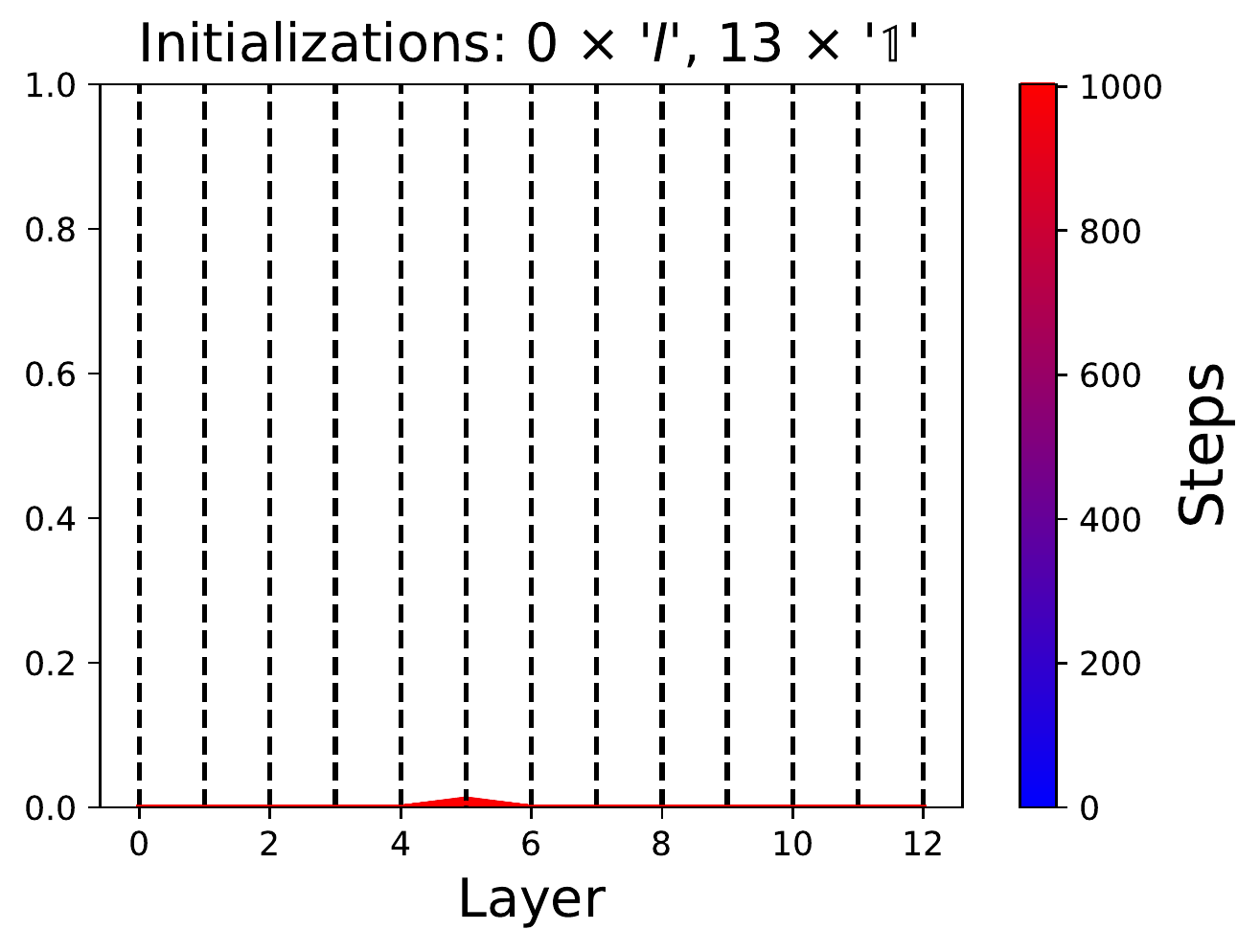}
    \end{subfigure}
    \\
    \caption{Symmetry breaking activation perturbation decays as it propagates forward through the network, if the $\mathbb{1}$ initialization was extensively used. The panels show the ratio between the perturbation norm $\left\Vert \Delta\otimes\left(I-\frac{1}{n}\bm{1}\bm{1}^{T}\right)\widehat{*}\alpha^{(\ell)}(x)\right\Vert _{F}$  to the total norm $\left\Vert \alpha^{(\ell)}(x)\right\Vert _{F}$, for every layer $\ell$ and different initializations of leakyNet. Each panel also shows the evolution of the ratio at the first $1000$ steps of training, as indicated by the color-map. Dashed lines mark the layers that were initialized with '$\mathbb{1}$'-init (while the other layers were initialized with $I$).} 
    \label{fig:pertubation_prop}
\end{figure*}

\begin{figure*}[ht]
    \centering
    \textbf{Gradient Perturbation norm /  Gradient Total norm}\par\medskip
    \begin{subfigure}{}
        \centering
    	\includegraphics[width=\figscale, height=2.0in]{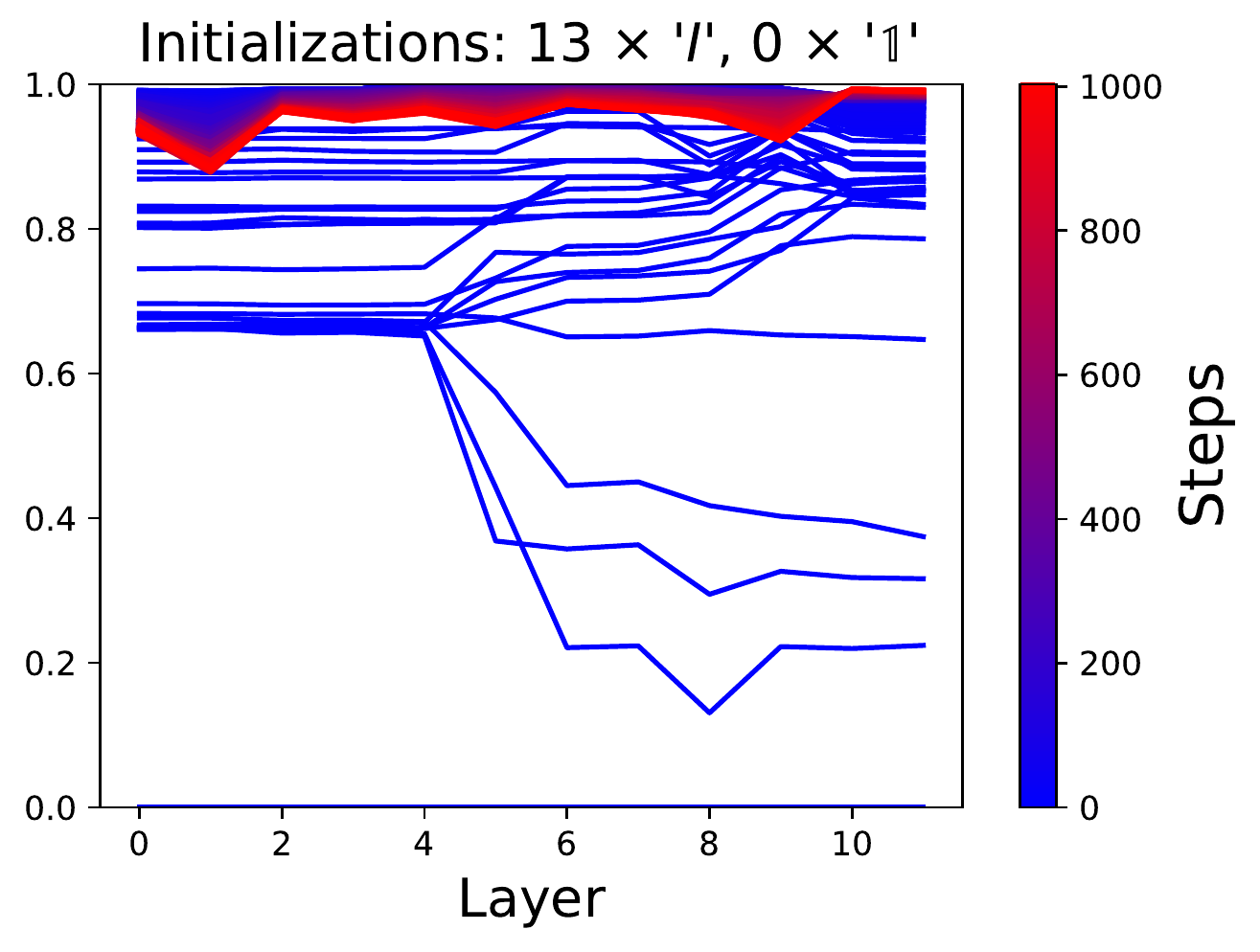}
    \end{subfigure}
    \begin{subfigure}{}
        \centering
    	\includegraphics[width=\figscale, height=2.0in]{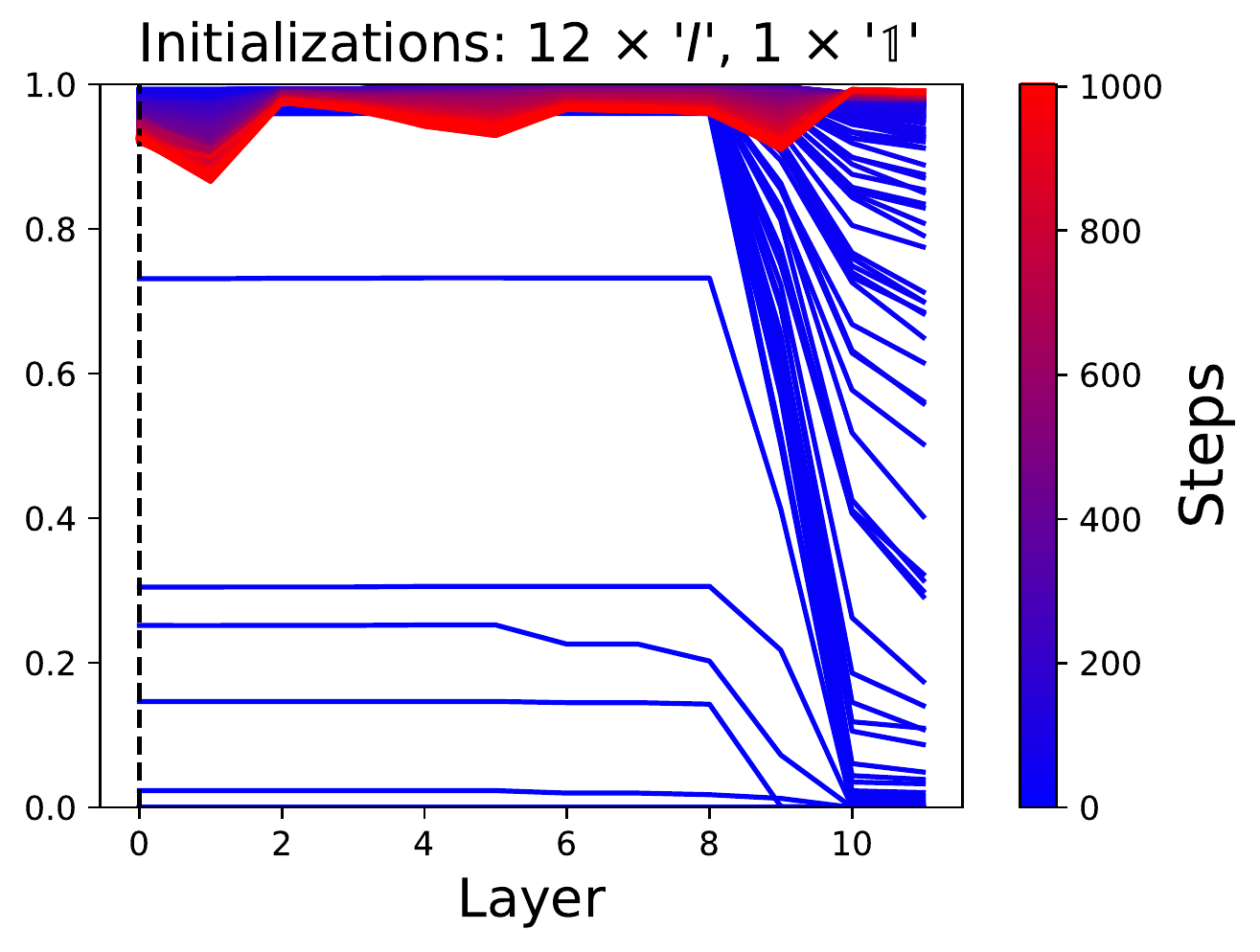}
    \end{subfigure}
    \\
    \centering
    \begin{subfigure}{}
        \centering
    	\includegraphics[width=\figscale, height=2.0in]{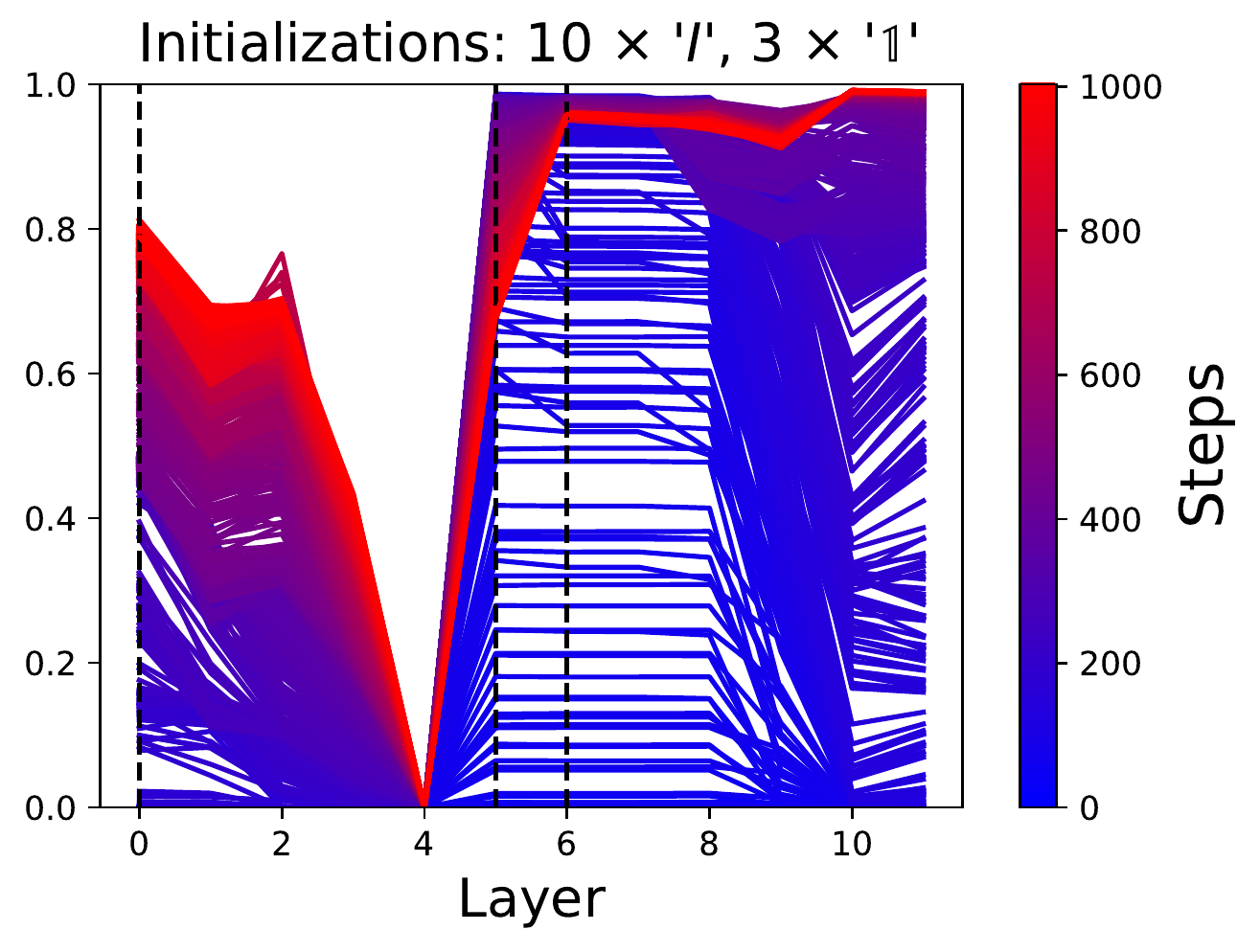}
    \end{subfigure}
    \begin{subfigure}{}
        \centering
    	\includegraphics[width=\figscale, height=2.0in]{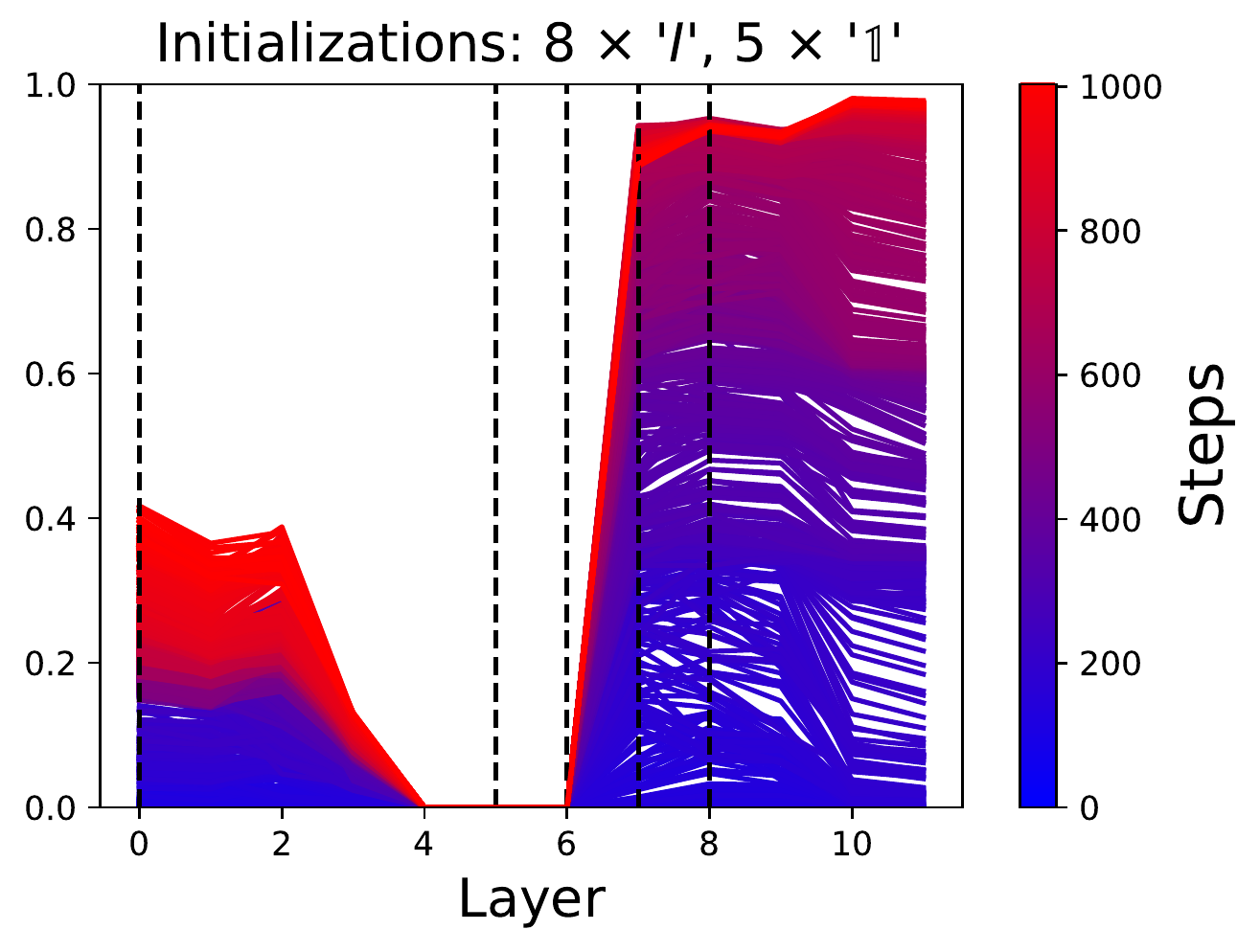}
    \end{subfigure}
    \\
    \centering
    \begin{subfigure}{}
        \centering
    	\includegraphics[width=\figscale, height=2.0in]{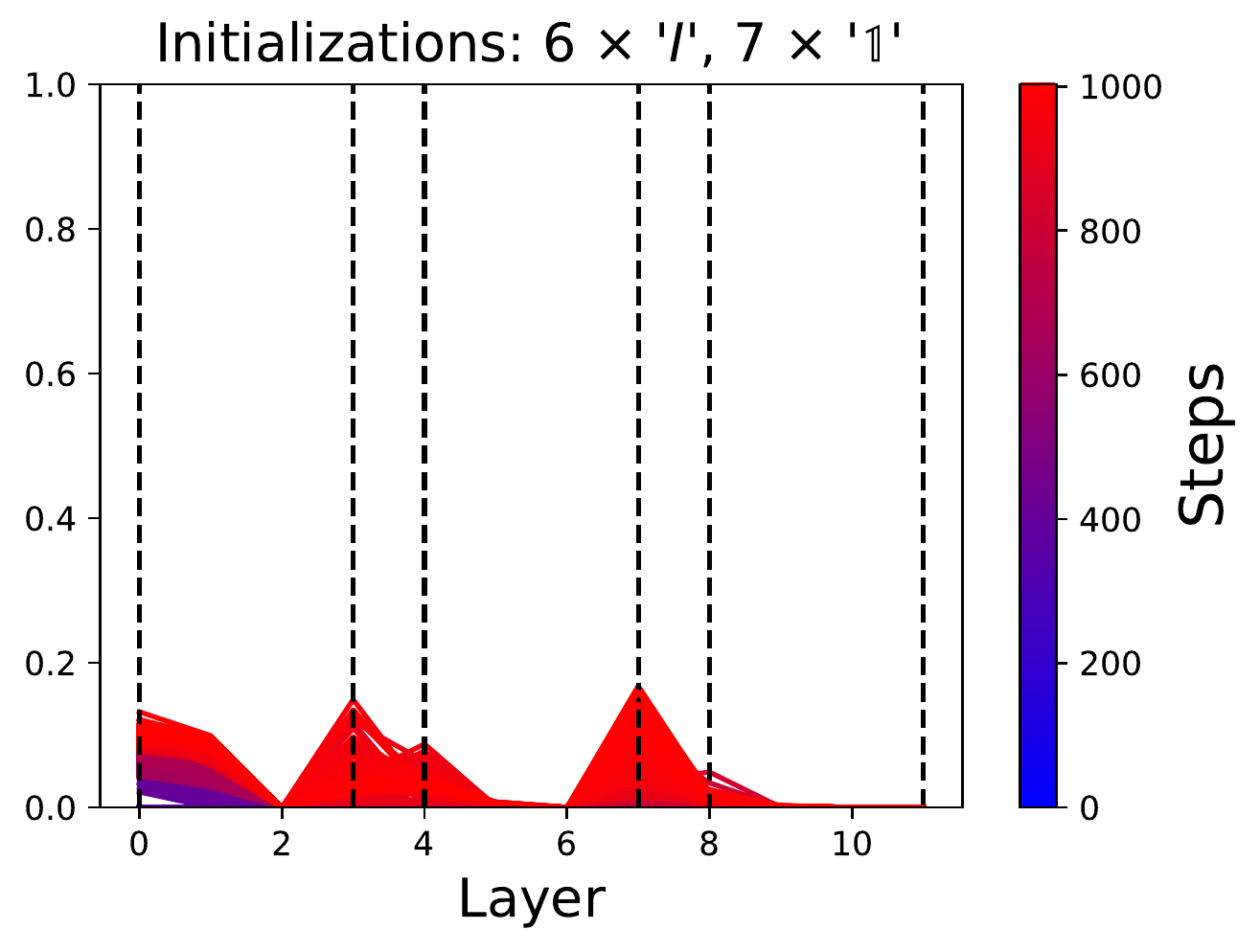}
    \end{subfigure}
    \begin{subfigure}{}
        \centering
    	\includegraphics[width=\figscale, height=2.0in]{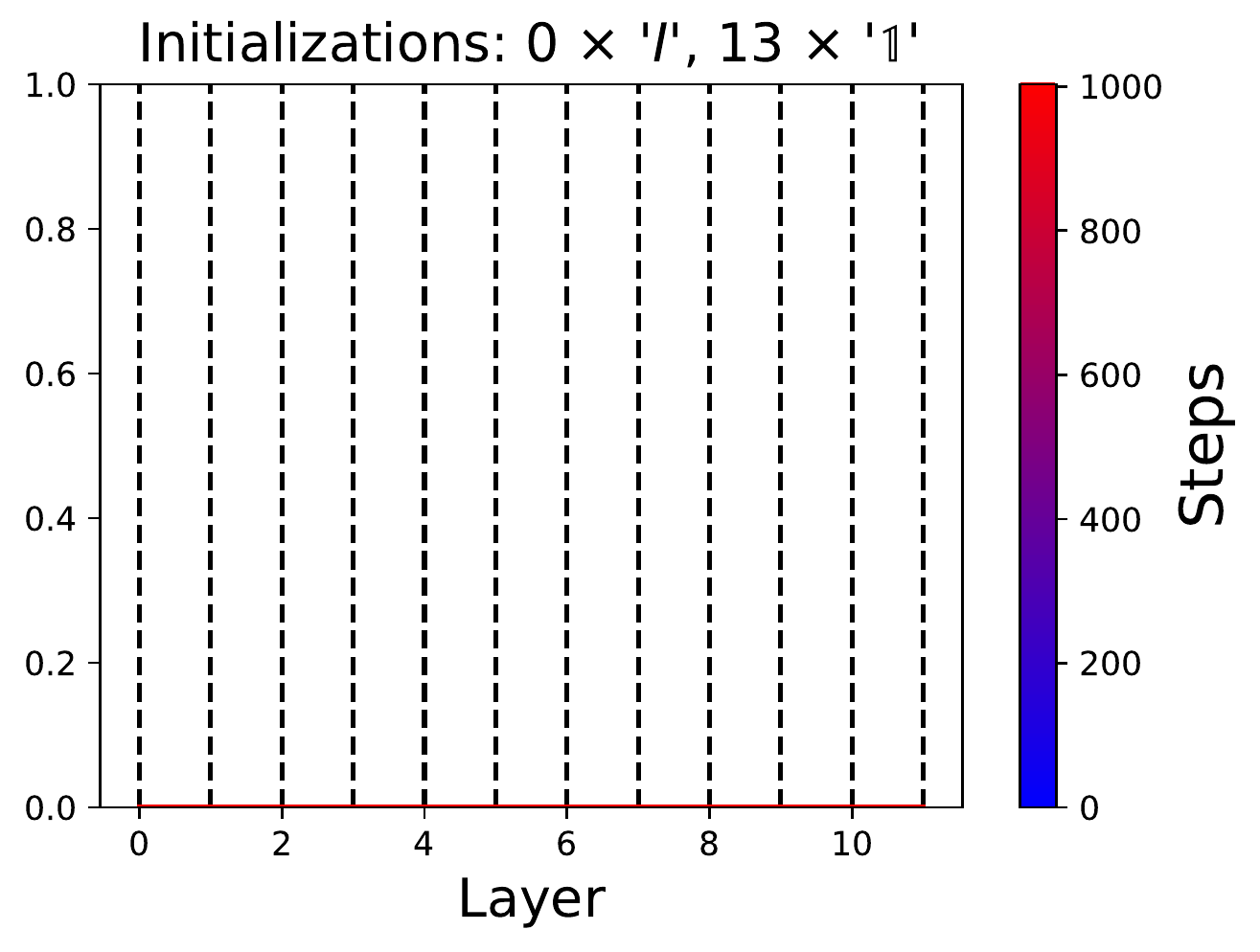}
    \end{subfigure}
    \\
    \caption{Symmetry breaking gradient perturbation decays as it propagates backwards through the network, due to the effect of the $\mathbb{1}$ initialization. As in figure \ref{fig:pertubation_prop}, we measure the ratio of the gradient magnitude that is 'perturbed' from the mean-gradient over all channels. We thus measure the ratio between $\left\Vert \Delta\otimes\left(I-\frac{1}{n}\bm{1}\bm{1}^{T}\right)\widehat{*} \frac{\partial\mathcal{L}(x,y)}{\partial\alpha^{(\ell)}(x)}\right\Vert _{F}$  to  $\left\Vert \frac{\partial\mathcal{L}(x,y)}{\partial\alpha^{(\ell)}(x)}\right\Vert _{F}$. The perturbation in the gradients will directly influence the updates of the weight tensors, causing symmetry break (decrease forward correlation). In networks with symmetrical initialization ($\mathbb{1}$-init), it takes longer for the perturbation to grow-- in both of the cases where the entire network was initialized with averaging initialization, the network remained symmetrical during the $1000$ steps presented in this figure. Like before, dashed lines mark the layers that were initialized with '$\mathbb{1}$'-init (while the other layers were initialized with $I$).} 
    \label{fig:pertubation_grad_prop}
\end{figure*}